\documentclass{ecai} 
\usepackage{latexsym}
\usepackage{amssymb}
\usepackage{amsmath}
\usepackage{amsthm}
\usepackage{booktabs}
\usepackage{graphicx}
\usepackage{color}

\usepackage{algorithm,algorithmic}

\usepackage{subfiles}
\usepackage{svg}
\usepackage{custom}

\newcommand{\BibTeX}{B\kern-.05em{\sc i\kern-.025em b}\kern-.08em\TeX}

\begin{document}

%%%%%%%%%%%%%%%%%%%%%%%%%%%%%%%%%%%%%%%%%%%%%%%%%%%%%%%%%%%%%%%%%%%%%%%%

\begin{frontmatter}

%%% Use this command to specify your submission number.
%%% In doubleblind mode, it will be printed on the first page.

%\paperid{2242} 

%%% Use this command to specify the title of your paper.

\title{{\huge A Single Online Agent Can Efficiently Learn Mean Field Games}}

%%% Use this combinations of commands to specify all authors of your 
%%% paper. Use \fnms{} and \snm{} to indicate everyone's first names 
%%% and surname. This will help the publisher with indexing the 
%%% proceedings. Please use a reasonable approximation in case your 
%%% name does not neatly split into "first names" and "surname".
%%% Specifying your ORCID digital identifier is optional. 
%%% Use the \thanks{} command to indicate one or more corresponding 
%%% authors and their email address(es). If so desired, you can specify
%%% author contributions using the \footnote{} command.

\author[A]{\fnms{Chenyu}~\snm{Zhang}}
\author[B]{\fnms{Xu}~\snm{Chen}\footnotemark}
\author[B]
{\fnms{Xuan}~\snm{Di}\thanks{Email: sharon.di@columbia.edu}} 

\address[A]{Data Science Institute, Columbia University, New York, NY, USA}
\address[B]{Department of Civil Engineering and Engineering Mechanics, Columbia University, New York, NY, USA}
%\address[C]{Short Alternate Affiliation of Third Author}

%%% Use this environment to include an abstract of your paper.

\begin{abstract}
	% background & intro
	Mean field games (MFGs) are a promising framework for modeling the behavior of large-population systems.
	% challenges & gaps
	However, solving MFGs can be challenging due to the coupling of forward population evolution and backward agent dynamics. Typically, obtaining mean field Nash equilibria (MFNE) involves an iterative approach where the forward and backward processes are solved alternately, known as fixed-point iteration (FPI). This method requires fully observed population propagation and agent dynamics over the entire spatial domain, which could be impractical in some real-world scenarios.
	% one-line description
	To overcome this limitation, this paper introduces a novel online single-agent model-free learning scheme, which enables a single agent to learn MFNE using online samples, without prior knowledge of the state-action space, reward function, or transition dynamics.
	% novelty & improvement
	Specifically, the agent updates its policy through the value function (Q), while simultaneously evaluating the mean field state (M), using the same batch of observations.
	% main results
	We develop two variants of this learning scheme: off-policy and on-policy QM iteration. We prove that they efficiently approximate FPI, and a sample complexity guarantee is provided. The efficacy of our methods is confirmed by numerical experiments.
\end{abstract}

\end{frontmatter}

%%%%%%%%%%%%%%%%%%%%%%%%%%%%%%%%%%%%%%%%%%%%%%%%%%%%%%%%%%%%%%%%%%%%%%%%

\section{Introduction} % 2 pages

%%%%%%%%%%%%%%%%%%%%%%%%%%%%%%%%%%%%%%%%
% Brief introduction to MFGs
%%%%%%%%%%%%%%%%%%%%%%%%%%%%%%%%%%%%%%%%

Mean field games (MFGs) \citep{huang2006Largepopulation,lasry2007Meanfield} offer a tractable model to describe the population impact on individual agents in multi-agent systems with a large population.
This work delves into the increasingly prominent field of applying reinforcement learning (RL) \citep{sutton2018Reinforcementlearninga} to learn MFGs.
% On the other hand, mean-field techniques present an effective strategy to approximate multi-agent reinforcement learning (MARL) systems, addressing MARL's scalability issue \citep{cui2022SurveyLargePopulation}.

%%%%%%%%%%%%%%%%%%%%%%%%%%%%%%%%%%%%%%%%
% Brief introduction to FPI
%%%%%%%%%%%%%%%%%%%%%%%%%%%%%%%%%%%%%%%%

In an MFG, the influence of other agents is encapsulated by a \emph{population mass} which provides a reliable approximation of real interactions between agents when the number of agents is large.
A widely used method for learning MFGs is fixed-point iteration (FPI), which iteratively calculates the \emph{best response} (BR) w.r.t. the current population, and the \emph{induced population distribution} (IP) w.r.t. the current policy \citep{guo2019Learningmeanfield}.
The FPI algorithm can be formally expressed as:
\[
	(\pi_k,\mu_k) = (\Gamma_{\mathrm{IP}} \circ \Gamma_{\mathrm{BR}})^k(\pi_0,\mu_0)
	,\]
where operators
\(\Gamma_{\mathrm{BR}}\) calculates the best response and \(\Gamma _{\mathrm{IP}}\)
calculates the induced population distribution.
We defer the full definitions of these operators to \cref{sec:pre}.

%%%%%%%%%%%%%%%%%%%%%%%%%%%%%%%%%%%%%%%%
% Current limitations
%%%%%%%%%%%%%%%%%%%%%%%%%%%%%%%%%%%%%%%%

Although it is a prominent scheme for learning MFGs, current implementations of FPI and its variants face several limitations, especially in the IP calculation:
\begin{enumerate*}[label=\upshape\arabic*\upshape)]
	\item \(\Gamma_{\mathrm{BR}}\) and $\Gamma_{\mathrm{IP}}$ are implemented separately and executed alternately, impeding parallel computing and potentially increasing the \emph{computational complexity} of the entire algorithm.
	\item The implementation of $\Gamma_{\mathrm{IP}}$ typically requires the knowledge of the transition dynamics of the environment \citep{yang2018LearningDeep,perrin2021MeanField,chen2023Learningdual,chen2023hybrid}, limiting the use of \emph{model-free} methods.
	\item Despite some proposals of model-free strategies in existing literature, these methods demand direct observability of population dynamics \citep{carmona2021ModelFreeMeanField,lee2021reinforcement,anahtarci2023Qlearningregularized}. In reality, fulfilling this requirement generally needs a central server capable of communication across the entire state space, restricting the feasibility of implementing such methods with a single online agent, i.e., an agent that interacts with the environment and collects local observations to learn and act on-the-go.
\end{enumerate*}

%%%%%%%%%%%%%%%%%%%%%%%%%%%%%%%%%%%%%%%%
% Why question
%%%%%%%%%%%%%%%%%%%%%%%%%%%%%%%%%%%%%%%%

While these limitations paint part of the picture, we still need to answer the following question:
\begin{quoting}\itshape
	Why should we employ a single online agent to learn the equilibria of mean field games?
\end{quoting}

%%%%%%%%%%%%%%%%%%%%%%%%%%%%%%%%%%%%%%%%
% Answers to the why question
%%%%%%%%%%%%%%%%%%%%%%%%%%%%%%%%%%%%%%%%

The reasons are multifold:
% These challenges serve as the major motivation for this work.
%\vspace{-0.1in}
\begin{itemize}%[nosep, left=0pt]
	\item In many real-world scenarios, a single online agent is often the most accessible, and sometimes the only available resource \citep{shou2022MultiAgentReinforcement}.
	\item Online single-agent model-free methods are more straightforward to implement, since they do not require prior knowledge of the data or the model.
	\item Once a single-agent model-free method is devised, this fundamental scheme can accommodate extensions such as multi-agent collaborative learning and model learning.
\end{itemize}

%%%%%%%%%%%%%%%%%%%%%%%%%%%%%%%%%%%%%%%%
% Can question
%%%%%%%%%%%%%%%%%%%%%%%%%%%%%%%%%%%%%%%%

Motivated by answers to the ``why'' question, we ask:
\begin{quoting}\itshape
	Can a single online agent learn the equilibria of mean field games efficiently?
\end{quoting}

%%%%%%%%%%%%%%%%%%%%%%%%%%%%%%%%%%%%%%%%
% Answer to the can question
%%%%%%%%%%%%%%%%%%%%%%%%%%%%%%%%%%%%%%%%

% In addition to the empirical motivation discussed above, our work is also strongly driven by the following theoretical result. 
\begin{figure}[H]
	\centering
	% \includesvg[width=0.8\columnwidth]{toy-2.svg}
	\includegraphics[width=0.9\columnwidth]{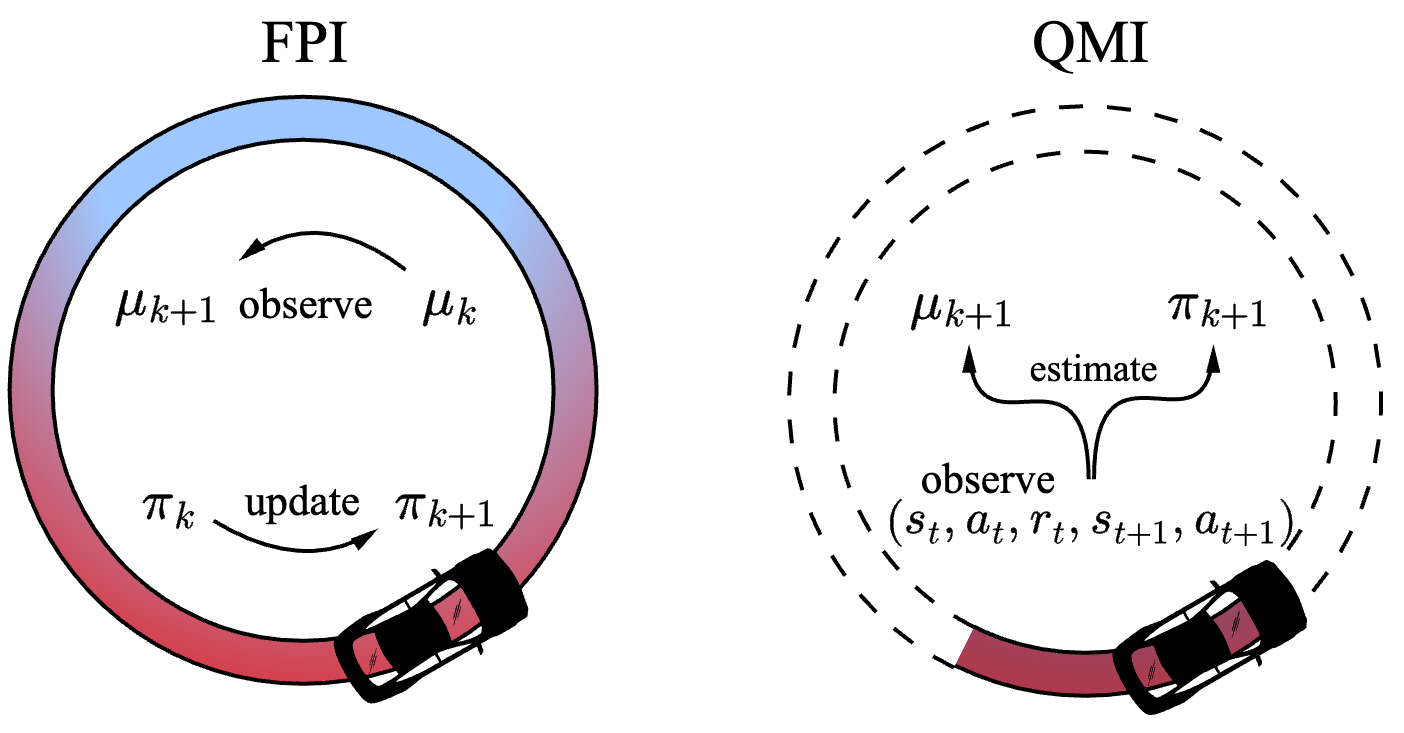}
	\caption{Illustration of learning processes of FPI and QMI for speed control on a ring road. The gradient color map signifies the varying population density on the ring road, with the dashed line indicating elements unobserved by the online agent. In FPI, the BR is calculated by a \emph{representative} agent and the IP is directly observed. In QMI, a single online agent observes only \emph{local} states and resultant rewards $(s_t,a_t,r_t,s_{t+1},a_{t+1})$, and uses these observations to estimate both the BR and IP.} \label{fig:toy}
\end{figure}

Our work affirmatively answers this question by presenting QM iteration (QMI), an efficient online single-agent model-free method for learning MFGs.
QMI is strongly backed by the following theoretical premise.
In an MFG, as all agents follow the same policy, we know that any agent's state is sampled from the population distribution.
This fact reveals that a single agent encapsulates information about the entire population, suggesting that the induced population distribution can be learned through a single agent's state observations.
More importantly, these observations are already collected during the phase where the agent updates its policy using an online RL method, suggesting that a single agent can learn both the BR and IP simultaneously using the same batch of online observations.%

% \[
% s_t \sim \mu_t \overset{L_1}{\longrightarrow} \mu
% ,\] 
% where measure convergence is in the $L_1$ distance,\footnotemark and $\mu$ is the steady distribution of the Markov chain. 
% Given these challenges and motivations, we ask

% \footnotetext{We consider finite state spaces in this work. A distribution on a finite state space $\mathcal{S}$ is specified by a vector $\mu$ in the probability simplex $\Delta(\mathcal{S})$. For a finite state space with the trivial metric, the total variation distance equals the 1-Wasserstein distance \citep{gibbs2002ChoosingBounding}, with the $L_1$ distance being twice as large as them. Without loss of generality, we focus on $L_1$ and $L_2$ distances for probability measures in this work.}

% This question not only echoes the above challenges but is also strongly motivated by real-world applications. 
% Furthermore, the success of single-agent methods for learning MFGs will pave the way for exploring multi-agent and federated learning approaches in this context. Note that MFGs approximate systems with a large number of agents; and here we argue that a single agent can efficiently learn this system and anticipate that collaborative efforts among multiple agents can expedite this learning process.

%%%%%%%%%%%%%%%%%%%%%%%%%%%%%%%%%%%%%%%%
% Toy example
%%%%%%%%%%%%%%%%%%%%%%%%%%%%%%%%%%%%%%%%

We present the example of speed control on a ring road, as illustrated in \cref{fig:toy}, to concretize the above ideas and highlight the improvements of QMI over FPI.
In this game, vehicles aim to maintain some desired speed while avoiding collisions.
In FPI, a \emph{representative} agent interacts with the population mass to learn the BR. Then, a dedicated forward process is needed to calculate the IP, either by leveraging knowledge of the transition dynamics or directly observing population dynamics across the entire state space.
In contrast, QMI employs a single online agent with only local state and reward observations.
Unlike FPI's representative agent, the online agent in QMI has no population information and thus no interaction with the population mass.
Consequently, it maintains an \emph{estimate} of the IP, and derives rewards according to this estimate%
%(analogous to self-play)
. Equipped with this estimate, the online agent in QMI, similar to FPI's representative agent, can update its policy using local observations by online RL methods. As a distinctive feature, this agent also uses these local observations to update its population distribution estimate. Hence, QMI consolidates the two separate backward and forward processes in FPI into one and eliminates the need for prior environmental knowledge and global communication.

\paragraph{Contributions.}
Our primary contributions include:

\begin{itemize}%[nosep, left=0pt]
	\item We propose an online single-agent model-free scheme for learning MFGs, termed as QM iteration (QMI).
	      At each step of QMI, the agent updates its BR and IP estimates \emph{simultaneously} using an online observation.
	      More practical than FPI, QMI is applicable when no prior knowledge of the transition dynamics or the state space is available.
	      We develop two variants of QMI,
	      contingent on whether the agent selects actions following a fixed \emph{behavior} policy, or adaptively updates its behavior policy within an outer iteration (\cref{alg}).
	      An overview of the distinct features exhibited by the two variants is provided in \cref{tab:comp}.
	      % we present two practical variants of QMI: off-policy and on-policy.
	      % Off-policy QMI learns two parallel MFG instances exactly using one batch of samples.
	      % On-policy QMI learns the BR and IP simultaneously using one batch of samples.
	      % Additionally, on-policy can accommodate population-dependent transition kernels and learn more robust policies.
	      % We briefly summarize the distinct features exhibited by the two variants in \cref{tab:comp}; detailed discussion is deferred to \cref{sec:off,sec:on}.
	\item We prove that QMI efficiently approximates FPI and, therefore enjoys a similar convergence guarantee.
	      The resemblance between the learning dynamics of QMI and FPI is illustrated in \cref{fig:qmi}.
	      % Despite the intuitive efficiency boost achieved by using one sample to update both the BR and IP estimates, 
	      We provide sample complexity guarantees for our methods (\cref{thm}).
	      Our methods are the first provably efficient online single-agent model-free methods for learning MFGs.
	      We validate our findings through numerical experiments on various MFGs (\cref{sec:exp}).
\end{itemize}

\paragraph{Related work.}

\cz{\citet{huang2006Largepopulation} introduced mean field games and suggested a forward-backward FPI scheme to solve them.
% \citet{guo2019Learningmeanfield} proposed using Q-learning for the backward best response calculation in fixed-point iteration.
% The theoretical assumptions needed for FPI are often difficult to verify, and empirically, FPI tends to exhibit instability in discrete-time \citep{cui2022ApproximatelySolving}.
To address the instability of FPI in discrete-time \citep{cui2022ApproximatelySolving}, researchers have proposed various stabilization techniques, including fictitious play \citep{cardaliaguet2017Learningmean,perrin2020Fictitiousplay}, online mirror descent \citep{perolat2021ScalingMean,lauriere2022ScalableDeep}, and entropy regularization \citep{cui2022ApproximatelySolving,guo2022Entropyregularization,anahtarci2023Qlearningregularized}.
% \citet{lauriere2022ScalableDeep} extended these techniques to incorporate scalable deep neural networks.
% In addition to FPI, other formulations have been proposed to avoid solving a forward-backward process.
\citet{yang2018LearningDeep,chen2023Learningdual} formulated MFGs as a population MDP, avoiding solving the forward-backward process in FPI, while requiring the knowledge of the entire state space and transition dynamics to update the state of the population.}
% \citet{carmona2021ModelFreeMeanField} proposed a comparable formulation called mean-field MDP, which similarly presumes a \emph{system function} that can directly return the population evolution.

A comprehensive survey on the application of RL in learning MFGs is presented by \citet{lauriere2022LearningMean,cui2022SurveyLargePopulation}.
Existing work exclusively focuses on obtaining BRs in FPI using RL methods, including
Q-learning \citep{guo2019Learningmeanfield,perrin2021MeanField,cui2022ApproximatelySolving},
policy gradient \citep{elie2020convergence},
and actor critic \citep{mguni2018decentralised,Subramanian2022mean,chen2023hybrid}.
\cz{For the population evolution (IP), most existing methods require either knowledge of the transition dynamics or direct observability.
}\cz{
Recently, \citet{angiuli2022unified,angiuli2023convergence} proposed an asynchronous Q-learning method for MFGs, removing the IP observability assumption, and proved its asymptotic convergence when population estimate updates occurs much slower than Q-value function updates ($\beta_{t} \ll \alpha _{t}$).
\citet{zaman2023oracle} extended this two-timescale model-free approach with model learning and proved its non-asymptotic convergence.
In contrast, our methods employ the same timescale for population and policy estimates ($\beta_{t} \asymp \alpha _{t}$), substantially distinguishing our methodology.
}

% To stabilize FPI, researchers have proposed various techniques, including fictitious play \citep{cardaliaguet2017Learningmean,perrin2020Fictitiousplay}, online mirror descent \citep{perolat2021ScalingMean}, and entropy regularization \citep{guo2022Entropyregularization,anahtarci2023Qlearningregularized}.
%Some papers re-formulate mean field games as a population Markov decision process (MDP) \citep{yang2018LearningDeep,chen2023Learningdual} or mean-field MDP \citep{carmona2021ModelFreeMeanField} to avoid solving a forward-backward process in the FPI method. 
% We include a more detailed discussion of related work in \cref{sec:apx-rel}.

\section{Preliminaries} \label{sec:pre} % 2 pages

\subsection{Mean Field Games} \label{sec:mfg}

We consider an infinite-horizon discounted Markov decision process (MDP) denoted by \(\mathcal{M} = (\mathcal{S},\mathcal{A}, r, P, \gamma)\), where \(\mathcal{S}\) and \(\mathcal{A}\) are the finite state and action spaces respectively,
with their cardinality denoted by $S \coloneqq |\mathcal{S}|$ and $A \coloneqq |\mathcal{A}|$,
$r$ is the reward function, $\gamma \in (0,1)$ is the discount factor, and $P$ is the transition kernel such that \(P(s'\given s, a)\) represents the probability that an agent transitions to state \(s'\) when it takes action \(a\) at state \(s\).
A policy (also referenced as a strategy or response) $\pi$ maps a state to a distribution on the action space, guiding the action choices of an agent.
When the policy $\pi$ is fixed, we use $P_{\pi}$ to denote the transition kernel and write $P_{\pi}(s,s') \coloneqq \sum_{a\in \mathcal{A}}P(s'\given s, a)\pi(a\given s)$.
% The trajectory of an agent following a fixed policy is a stationary Markov chain.

In MFGs, agents are considered indistinguishable with individually negligible influence. Thus, an MFG encapsulates the impact of all agents on a given one through the concept of \emph{population}.
In this work, we consider reward functions that depend on the population distribution over the state space \(\mu\in \Delta(\mathcal{S})\coloneqq \{ \text{distributions on } \mathcal{S} \}\).
Specifically, a reward function \(r: \mathcal{S}\times \mathcal{A}\times \Delta(\mathcal{S})\to[0,R]\) signals a reward at each state-action pair based on the population distribution.

In MFGs, agents are rational and aim to maximize their expected cumulative reward.
Our goal is to find an \emph{optimal} policy---one that cannot be improved given that other agents' policies are fixed.
We utilize a value-based approach to calculate policies. A Q-value function returns the expected cumulative reward starting from a state following the current policy $\pi$ and population distribution $\mu$:
\[\label{eq:q}
	Q_{\pi,\mu}(s, a)=\EE_{\pi}\!\!\left[\sum_{t=0}^{\infty} \gamma^t r\left(s_t, a_t,\mu\right) \Given s_0\!=\!s, a_0\!=\!a\right]\!\!
	,\]
where the expectation is taken w.r.t. the transition kernel $P_{\pi}$.
Given a value function, we can choose the action accordingly, e.g., greedily select the action that maximizes the value function or use an \(\epsilon\)-greedy selection. For broader adaptability, we presume access to a \emph{policy operator} \(\Gamma_{\pi}\) that yields a policy based on a value function. Thus, the optimal policy can be characterized through a value function, translating our goal into discovering an optimal value function.
We are now ready to present the optimality conditions.

\begin{definition}[Mean field Nash equilibrium]\label{def:mfne}
	A value function-population distribution pair \((Q,M)\) is a \emph{mean field Nash equilibrium (MFNE)} if it satisfies
	% the following equations 
	\begin{equation}\label{eq:qm}
		% \begin{cases}
		% 	Q & = \T{Q,M} Q, \\
		% 	M & = \P{Q} M,
		% \end{cases}
			Q = \T{Q,M} Q \quad \text{and} \quad
			M = \P{Q} M,
	\end{equation}
	where \(\T{Q,M}\) is the Bellman operator:
	\[\label{eq:bellman}
		\T{Q,M}Q(s,a) = \E{Q}[r(s,a,M) + \gamma Q(s',a')],
	\]
	where \(\E{Q}\) denotes the expectation over \(a,a' \sim \Gamma_{\pi}(Q)\) and \(s' \sim P(\cdot\given s,a)\); and \(\P{Q}\) is the transition operator:
	\[\label{eq:transition}\textstyle
		\P{Q}M(s') = \sum_{s\in \mathcal{S}}\PP{Q}(s,s')M(s)
		,\]
	where we write $\PP{Q} \!\coloneqq\! \PP{\Gamma_{\pi}(Q)}$, as the policy is determined by the value function given a fixed policy operator.
\end{definition}

In \cref{def:mfne}, $Q\in\R^{S\times A}$ denotes a generic value function table, which is not necessarily an actual value function defined per \cref{eq:q}.
Similarly, $M \in \Delta(\mathcal{S})$, where $\Delta(\mathcal{S})$ is the probability simplex over $\mathcal{S}$, represents a generic population distribution which is not necessarily an actual policy-induced population distribution.
Analogous to the Q-value function, we refer to this generic population distribution as the \emph{M-value function}.
We use subscripts, e.g., $\q{M}$ and $\m{Q}$, to indicate actual BRs and IPs w.r.t. specific population distributions and value functions.
% We reserve lowercase symbols $q$ and $\pi$ for actual value functions and population distributions.

%In this study, to focus on the main ideas, we only consider the optimality conditions corresponding to the steady distribution.%\footnote{A steady distribution is also referred as a stationary distribution, and its optimality conditions yield a stationary MFNE. The term ``steady'' is used here to differentiate from stationarity within the learning process.}

% Some prior work considers spatial-temporal MFGs \placeholder{ [citations needed]}, where agents' policies and population distributions evolve over time, introducing another layer of nonstationarity.

%Despite our methods being readily extensible to non-stationary MFGs, reestablishing the analysis is more involved.
% While our methods can be readily extended to address spatial-temporal MFGs, reestablishing the analysis is more involved.

\subsection{Fixed-Point Iteration for MFG}

Fixed-point iteration (FPI), a classic method for learning MFGs, comprises two steps: evaluating the \emph{best response} (BR) and the \emph{induced population distribution} (IP).
Fixing a population distribution $M$, the game reduces to a standard RL problem, which has a unique optimal value function \citep{bertsekas1996Neurodynamicprogramming}, i.e., the BR w.r.t. the population distribution $M$.
If the transition kernel $\PP{Q}$ yields a steady state distribution, this distribution is referred to as the IP w.r.t. the value function $Q$.
Decomposing \cref{eq:qm} gives formal definitions of these two operations.

\begin{definition}[FPI operators] \label{def:fpi-op}
	The BR operator,
	\[
		\Gamma_{\mathrm{BR}}: \Delta(\mathcal{S})\to \R^{S\times A},\  M \mapsto \q{M}
		,\]
	returns the unique solution to the Bellman equation $\q{M} = \T{Q_{M},M}\q{M}$ for any population distribution $M$.
	The IP operator,
	\[
		\Gamma_{\mathrm{IP}}: \R^{S\times A} \to \Delta(\mathcal{S}),\ Q\mapsto \m{Q}
		,\]
	returns the unique fixed point of the transition operator $\P{Q}$ defined in \cref{eq:transition} for any value function $Q$.
	Then, the FPI operator is the composition of the above two operators:
	\(
	\Gamma \coloneqq \Gamma_{\mathrm{IP}} \circ \Gamma_{\mathrm{BR}} : \Delta(\mathcal{S})\to \Delta(\mathcal{S})
	.\)
\end{definition}

% We remark that BRs and IPs are actual value function and population distributions respectively.
% It is worth noting that
% the second equation in \cref{eq:qm} is linear since $\P{Q}$ is a linear operator, while 
% the Bellman equation (the first equation in \cref{eq:qm}) is non-linear due to $\T{Q_{M},M}$'s dependency on $\q{M}$.
% Additionally, 
Notably,
the optimality in the BR is determined by the policy operator $\Gamma_{\pi}$.
For example,
When $\Gamma_{\pi}$ is the greedy selector: $\Gamma_{\pi}^{(\mathrm{max})}\!(Q)[a\!\given \!s] \!\!=\!\! \mathbbm{1}(a \!=\! \argmax_{a}Q(s,a))$, \cref{eq:bellman} becomes the Bellman optimality operator:
\[\label{eq:bellman-opt}
	\T{M}Q(s,a) = \mathbb{E}[r(s,a,M) + \gamma \max_{a'} Q(s',a')],
\]
making BRs deterministic optimal policies.
When $\Gamma_{\pi}$ is the softmax function: $\Gamma_{\pi}^{(\mathrm{softmax})}(Q)[a\given s] = e^{LQ(s,a)}/\sum_{a'}e^{LQ(s,a')}$, where $L$ is the inverse temperature parameter, the optimality corresponds to the MFG with entropy regularization \citep{cui2022ApproximatelySolving,guo2022Entropyregularization,anahtarci2023Qlearningregularized}.

To focus on the main ideas, we consider \emph{contractive} MFGs in this paper, where FPI is guaranteed to converge to the unique MFNE.
Then, in \cref{sec:anlys}, we show that our methods approximate FPI, thus enjoying a similar convergence guarantee.
Without the contraction condition,
stabilization techniques like fictitious play and online mirror descent need to be applied to FPI. We envision that our algorithms can be extended to incorporate these techniques with our analysis applying with minimal adjustment.

\begin{assumption}[Contractive MFG]\label{asmp:contract}
	The FPI operator is $(1\!-\!\kappa)$-contractive ($\kappa\in (0,1]$), i.e., for any $M_1,M_2\in \Delta(\mathcal{S})$, it holds that
	\[
		\|\Gamma M_1 - \Gamma M_2\|_{2} \le (1-\kappa) \|M_1-M_2\|_{2}
		.\footnotemark\]
\end{assumption}
\footnotetext{
	% We consider finite state spaces in this work. A distribution on a finite state space $\mathcal{S}$ is specified by a vector $\mu$ in the probability simplex $\Delta(\mathcal{S})$. 
	We consider $L_1$ and $L_2$ distances for probability measures in this work.
	For a finite state space with the trivial metric, the total variation distance equals the 1-Wasserstein distance \citep{gibbs2002ChoosingBounding}, with the $L_1$ distance being twice as large as them.
	Without loss of generality, we redefine the total variation distance as twice its standard definition, and use it interchangeably with the $L_1$ distance.
	% Without loss of generality, we focus on $L_1$ and $L_2$ distances for probability measures in this work.
}

% \subsection{Temporal Difference Control}
\section{Online Stochastic Updates} \label{sec:m-update}

Without prior knowledge of the environment or the population, the online agent maintains two estimates---the Q-value function for the BR and the M-value function for the IP---which it updates using online stochastic observations. We first extend temporal difference (TD) control methods, a classic model-free RL framework covering Q-learning and SARSA \citep{sutton2018Reinforcementlearninga}, to learn BRs, and then derive an online stochastic update rule for the IPs in the same vein.

\paragraph{Q-value function update.}
Guided by the Bellman operator \cref{eq:bellman}, TD control gives an online stochastic update for the Q-value function:
\[\label{eq:qupdate}
	\begin{aligned}
		Q(s,a) \leftarrow        & Q(s,a) - \alpha g_{Q}(s,a,s',a'),    \\
		\text{with}\quad g_{Q} = & Q(s,a) - r(s,a,M) - \gamma Q(s',a'),
	\end{aligned}
\]
where $\alpha$ is the step size, $s' \sim P(\cdot \given s,a)$, and $a' \sim \Gamma_{\pi}(Q)$.
If the policy operator is greedy and the behavior policy is fixed, the above update rule gives rise to off-policy Q-learning \citep{watkins1992Qlearning}; for general policy operators, if the behavior policy updates in accordance to the value function, i.e., $\Gamma_{\pi}(Q)$ is the behavior policy and $a'$ is the actual action, the above update rule gives rise to on-policy SARSA \citep{rummery1994OnlineQlearning,singh1996Reinforcementlearning}.
We defer further discussion on these two TD control methods to \cref{sec:alg,sec:anlys}.

% It is clear that a single online agent can learn BRs using TD control methods.
% For the IPs, 
% % an important observation is that due to the fact that 
% since all agents are homogeneous, a single agent actually carries the information of the entire population.
% Specifically, if the current population distribution is $\mu$, then an agent's current state $s$ follows this distribution: $s \sim \mu$. Therefore, we can use an agent's state distribution to approximate the population distribution.
% % To approximate the FPI using a single online agent, we still need to derive an online stochastic update rule for the population distribution.
% We now formally derive an online stochastic update rule for the IP estimate.

\paragraph{Population estimate.}

TD control replaces the expectation in the Bellman operator \cref{eq:bellman} with a stochastic approximation using online observations. Likewise, for the M-value function update,
we first rewrite the transition operator using expectation:
\begin{align}
	\P{Q}M(z) = & \sum_{s'\in\mathcal{S}}\delta_{s'}(z)\P{Q}M(s')                                 \\
	=           & \sum_{s'\in \mathcal{S}} \sum_{s\in\mathcal{S}} \delta _{s'}(z)\PP{Q}(s,s')M(s) \\
	=           & \E{Q,M}[\delta _{s'}(z)], \label{eq:pm-exp}
\end{align}
where $\delta_{s'}$ is the indicator probability vector in $\Delta(\mathcal{S})$ such that $\delta_{s'}(z) = \mathbbm{1}(z=s')$, and the expectation is taken over \(s \sim M\) and \(s' \sim \PP{Q}(s,\cdot )\).
Mimicking TD control and stochastic gradient descent, we remove the expectation and use the observed next successive state $s'$ to stochastically approximate $\P{Q}M$. This
gives an online stochastic update for the M-value function:
\[\label{eq:mupdate}
	M \leftarrow M - \beta g_{M}(s') = M + \beta (\delta_{s'} - M)
	.\]
where $\beta$ is the step size.
Please refer to \cref{sec:apx-update} for the full derivation of this update rule.
% , and with a slight abuse of notation, we denote $\delta_{s'} \in \Delta(\mathcal{S})$ the indicator vector such that $(\delta_{s'})_{s} = \mathbbm{1}(s=s')$.
Similar to TD control, we anticipate that this update rule drives the M-value function to converge to the population distribution induced by $\PP{Q}$.
Furthermore, selecting a step size of $\beta_t = 1 /(t+1)$ simplifies it to a Markov chain Monte Carlo (MCMC) method, validating its correctness.

For online stochastic updates \cref{eq:qupdate,eq:mupdate} and general online learning methods to yield optimal solutions, the environment must be readily \emph{explorable}.
Unlike offline methods which rely on pre-collected data, an online agent learns and acts based on its real-time observations.
Hence, the efficient learning of optimal policy becomes unfeasible if certain states are inaccessible, leading to potential suboptimal solutions. To avoid this, we impose the following condition on the MDP.

\begin{assumption}[Ergodic MDP] \label{asmp:ergodic}
	For any $Q\in \R^{S\times A}$, the Markov chain induced by the transition kernel $\PP{Q}$ is ergodic with a uniform mixing rate. In other words, there exists a steady state distribution $\m{Q}$ for any policy $\Gamma_{\pi}(Q)$, with constants $m \ge 1$ and $\rho \in (0, 1)$, such that
	\[
		\sup_{s\in \mathcal{S}} \sup_{Q\in \R^{S\times A}} \left\| P_{\pi}\left( S_t = \cdot \Given S_0 = s \right) - \m{Q} \right\|_{\mathrm{TV}} \le m \rho^{t}
		.\]
	For future reference, we define an auxiliary constant $\sigma = \hat{n} + m \rho^{\hat{n}} / (1-\rho)$, where $\hat{n} = \left\lceil \log _{\rho} m^{-1} \right\rceil$.
	And the probability of visiting a state-action pair under a steady distribution is lower bounded:
	\[
		\inf_{(s,a)\in \mathcal{S}\times \mathcal{A}}\inf_{Q\in\R^{S\times A}} \mu_{Q}(s) \cdot \Gamma_{\pi}(Q)[a\given s] \ge \lambda_{\min} > 0
		.
	\]
\end{assumption}

\cref{asmp:ergodic} is a standard assumption for online methods \citep{bhandari2018FiniteTime,zou2019Finitesampleanalysis,qu2020Finitetimeanalysis}.
% One can also replace the geometric mixing rate with a general mixing time.

\section{Proposed Methods} \label{sec:alg} % 2 page
\subsection{Off-Policy QM Iteration} %\label{sec:off}

% Given online stochastic update rules \cref{eq:qupdate,eq:mupdate}, one can already construct a simple online single-agent model-free method by \emph{iteratively} evaluating the BR and IP in the FPI method. However, 
A significant advantage of our online stochastic formulation of the update rules, \cz{over the \emph{iterative} BR and IP evaluation typical of FPI methods}, is that it enables \emph{simultaneous} updates of both the Q-value and M-value functions using the same batch of observations.

Taking one step in this direction, we first present an algorithm that simultaneously evaluates both the BR and IP, with the agent's behavior policy being fixed within each outer iteration.
Since the behavior policy is not updated along with the Q-value function, we use off-policy Q-learning to learn the BR, and term this method \emph{off-policy} QM iteration.
The method is presented in \cref{alg} with input option \texttt{off-policy} and the greedy policy operator $\Gamma_{\pi}^{(\mathrm{max})}$.

Our algorithm showcases marked simplicity. At each time-step, the online agent observes a state transition and a reward; it then uses this information to update the Q-value and M-value function tables using \cref{eq:qupdate,eq:mupdate}, respectively, which only involves two elementary operations---scaling and addition.
It is noteworthy that in the Q-value function update, $a_{t+1}$, which follows the fixed behavior policy, is not used. Instead, $a' \sim \Gamma_{\pi}^{(\mathrm{max})}(Q_{k,t})$ is used according to \cref{eq:qupdate}.
The discrepancy accounts for the naming of ``off-policy'' Q-learning.

\begin{algorithm}[ht]
	\caption{QM Iteration} \label{alg}
    \begin{algorithmic}[1]
	\STATE {\bfseries Input:} initial value functions $Q_{-1,T} = Q_0$ and $M_{-1,T} = M_0$; initial state $s_0$; option \texttt{off-policy} or \texttt{on-policy}
	\FOR {$k = 0,1,..., K$}
	\STATE $Q_{k,0} = Q_{k-1,T}, M_{k,0} = M_{k-1,T}$
	\STATE $\pi_{k,0} = \Gamma_{\pi}(Q_{k,0})$
	\FOR {$t = 0,1, \ldots ,T$}
	\STATE sample one Markovian observation tuple $(s_t,a_t,s_{t+1},a_{t+1})$ following policy $\pi_{k,t}$
	\STATE observe the reward $r(s_t,a_t,M_{k,0})$
	\STATE $Q_{k,t+1}(s_t,a_t)\!=\!Q_{k,t}(s_t,a_t)\!-\!\alpha_{t} g_{Q_{k,t}}$ \label{line:q-update}
	\STATE $M_{k,t+1} = M_{k,t} - \beta_t g_{M_{k,t}}(s_{t+1})$\; \label{line:m-update}
	\IF {\texttt{off-policy}}
	\STATE $\pi_{k,t+1} = \pi_{k,0}$
	\ELSIF {\texttt{on-policy}}
	\STATE $\pi_{k,t+1} = \Gamma_{\pi}\left(\operatorname{mix}\left(\{Q_{k,l}\}_{l=0}^{t+1}\right)\right)$ \label{line:p-update}
	\ENDIF
	\ENDFOR
	\ENDFOR
        \STATE {\bfseries return} $Q_{K,T},M_{K,T}$
        \end{algorithmic}
\end{algorithm}

The stationary nature of the transition kernel within each outer iteration directly gives the convergence guarantee of off-policy QMI and suggests its analogy with FPI (see \cref{sec:anlys}).
Nevertheless, fixed transition kernels make off-policy QMI learn BRs and IPs \emph{parallelly}. That is, at $k$th iteration, the BR w.r.t. $M_{k,0}$ is approximated by $Q_{k,T} = Q_{k+1,0}$, whose corresponding population distribution is then approximated by $M_{k+1,T} = M_{k+2,0}$, rather than $M_{k+1,0}$.
Let $Q_{k} \coloneqq Q_{k,0}$ and $M_{k}\coloneqq M_{k,0}$. Then, off-policy QMI generates two non-interacting parallel policy-population sequences: $\{(Q_{2k},M_{2k+1})\}_{k=0}^{K/2}$ and $\{(Q_{2k+1}, M_{2k})\}_{k=0}^{K/2}$.
This observation also implies that off-policy QMI is at least twice as \cz{data-}efficient as FPI; see \cref{fig:qmi} for an illustration.

To establish the convergence guarantee of off-policy QMI, we leverage the theoretical results of off-policy Q-learning.
% Specifically, we choose the policy operator in off-policy QMI as the greedy selector: $\Gamma_{\pi}^{(\mathrm{max})}(Q)[s] = \mathbbm{1}(a = \argmax_{a}Q(s,a))$.
However, the greedy policy operator used is too \emph{nonsmooth}: a slight difference in the Q-value function can lead to completely different action choices.
As a result, the induced population distributions can drastically differ between outer iterations, leading to unstable convergence performance.
Since we do not incorporate stabilization techniques, we make the following assumption.

\begin{assumption}[Lipschitz continuous transition kernels for Q-learning] \label{asmp:ql}
	For any $Q_1,Q_2\in \R^{S\times A}$, it holds that
	\[
		\|P_{Q_1} - P_{Q_2}\|_{\mathrm{TV}} \le L\left\| Q_1-Q_2 \right\|_{2}
		,\]
	where $\|P_{Q}\|_{\mathrm{TV}} \coloneqq \sup_{\|q\|_{\mathrm{TV}} = 1} \left\| \sum_{s\in\mathcal{S}} q(s)P_{Q}(s,\cdot ) \right\|_{\mathrm{TV}}$.
\end{assumption}

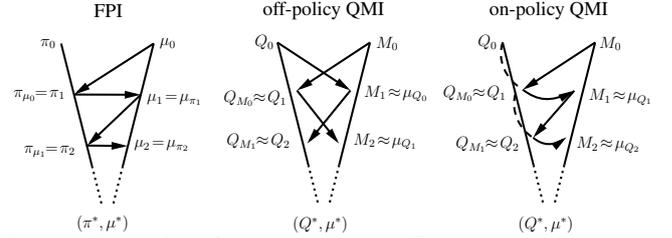
\begin{figure}[H]\centering
	\resizebox{0.482\textwidth}{!}{
		\tikzset{every picture/.style={line width=1.75pt}} %set default line width to 0.75pt
	\begin{tikzpicture}[x=0.75pt,y=0.75pt,yscale=-1,xscale=1]
		%uncomment if require: \path (0,473); %set diagram left start at 0, and has height of 473

		%Straight Lines [id:da5082529078425857] 
		\draw [line width=1.5]    (100.63,50) -- (130.63,170) ;
		%Straight Lines [id:da320061539369769] 
		\draw [line width=1.5]    (190.63,50) -- (160.63,170) ;
		%Straight Lines [id:da1118190061880795] 
		\draw [line width=1.5]  [dash pattern={on 1.69pt off 2.76pt}]  (130.63,170) -- (140.63,210) ;
		%Straight Lines [id:da895099956736473] 
		\draw [line width=1.5]  [dash pattern={on 1.69pt off 2.76pt}]  (160.63,170) -- (150.63,210) ;
		%Straight Lines [id:da3999577313201359] 
		\draw [line width=1.5]    (190.63,50) -- (116.59,97.83) ;
		\draw [shift={(113.23,100)}, rotate = 327.14] [fill={rgb, 255:red, 0; green, 0; blue, 0 }  ][line width=0.08]  [draw opacity=0] (15.6,-3.9) -- (0,0) -- (15.6,3.9) -- cycle    ;
		%Straight Lines [id:da7201243914254833] 
		\draw [line width=1.5]    (113.23,100) -- (174.03,100.19) ;
		\draw [shift={(178.03,100.2)}, rotate = 180.18] [fill={rgb, 255:red, 0; green, 0; blue, 0 }  ][line width=0.08]  [draw opacity=0] (15.6,-3.9) -- (0,0) -- (15.6,3.9) -- cycle    ;
		%Straight Lines [id:da9433672650363301] 
		\draw [line width=1.5]    (178.03,100.2) -- (128.92,147.23) ;
		\draw [shift={(126.03,150)}, rotate = 316.24] [fill={rgb, 255:red, 0; green, 0; blue, 0 }  ][line width=0.08]  [draw opacity=0] (15.6,-3.9) -- (0,0) -- (15.6,3.9) -- cycle    ;
		%Straight Lines [id:da7404549067519004] 
		\draw [line width=1.5]    (126.03,150) -- (161.23,150.36) ;
		\draw [shift={(165.23,150.4)}, rotate = 180.58] [fill={rgb, 255:red, 0; green, 0; blue, 0 }  ][line width=0.08]  [draw opacity=0] (15.6,-3.9) -- (0,0) -- (15.6,3.9) -- cycle    ;
		%Straight Lines [id:da8640576085239269] 
		\draw [line width=1.5]    (530.03,49) -- (560.03,169) ;
		%Straight Lines [id:da5479416758630966] 
		\draw [line width=1.5]    (620.03,49) -- (590.03,169) ;
		%Straight Lines [id:da6218490369034411] 
		\draw [line width=1.5]  [dash pattern={on 1.69pt off 2.76pt}]  (560.03,169) -- (570.03,209) ;
		%Straight Lines [id:da35664537932016516] 
		\draw [line width=1.5]  [dash pattern={on 1.69pt off 2.76pt}]  (590.03,169) -- (580.03,209) ;
		%Straight Lines [id:da25543922423178] 
		\draw [line width=1.5]    (620.03,49) -- (553.06,93.05) ;
		\draw [shift={(549.72,95.25)}, rotate = 326.66] [fill={rgb, 255:red, 0; green, 0; blue, 0 }  ][line width=0.08]  [draw opacity=0] (15.6,-3.9) -- (0,0) -- (15.6,3.9) -- cycle    ;
		%Straight Lines [id:da26417092678060894] 
		\draw [line width=1.5]    (599.72,96.25) -- (562.35,139.23) ;
		\draw [shift={(559.72,142.25)}, rotate = 311.01] [fill={rgb, 255:red, 0; green, 0; blue, 0 }  ][line width=0.08]  [draw opacity=0] (15.6,-3.9) -- (0,0) -- (15.6,3.9) -- cycle    ;
		%Curve Lines [id:da346489113972442] 
		\draw [line width=1.5]  [dash pattern={on 5.63pt off 4.5pt}]  (530.03,49) .. controls (524.6,65.4) and (532.22,84.75) .. (549.72,95.25) ;
		%Curve Lines [id:da8614158367682028] 
		\draw [line width=1.5]    (549.72,95.25) .. controls (563.9,106.12) and (576.62,107.7) .. (596.23,98.05) ;
		\draw [shift={(599.72,96.25)}, rotate = 151.79] [fill={rgb, 255:red, 0; green, 0; blue, 0 }  ][line width=0.08]  [draw opacity=0] (15.6,-3.9) -- (0,0) -- (15.6,3.9) -- cycle    ;
		%Curve Lines [id:da2210279389459402] 
		\draw [line width=1.5]  [dash pattern={on 5.63pt off 4.5pt}]  (542.63,99) .. controls (537.2,115.4) and (546.72,134.75) .. (559.72,142.25) ;
		%Curve Lines [id:da2146932868757241] 
		\draw [line width=1.5]    (559.72,142.25) .. controls (569.11,150.4) and (579.12,150.35) .. (589.08,143.97) ;
		\draw [shift={(592.22,141.75)}, rotate = 142.31] [fill={rgb, 255:red, 0; green, 0; blue, 0 }  ][line width=0.08]  [draw opacity=0] (15.6,-3.9) -- (0,0) -- (15.6,3.9) -- cycle    ;
		%Straight Lines [id:da28847888471987493] 
		\draw [line width=1.5]    (311.03,49) -- (341.03,169) ;
		%Straight Lines [id:da8489497483153192] 
		\draw [line width=1.5]    (401.03,49) -- (371.03,169) ;
		%Straight Lines [id:da2108734424518912] 
		\draw [line width=1.5]  [dash pattern={on 1.69pt off 2.76pt}]  (341.03,169) -- (351.03,209) ;
		%Straight Lines [id:da00031720694772374713] 
		\draw [line width=1.5]  [dash pattern={on 1.69pt off 2.76pt}]  (371.03,169) -- (361.03,209) ;
		%Straight Lines [id:da6499887199801666] 
		\draw [line width=1.5]    (401.03,49) -- (333.34,93.79) ;
		\draw [shift={(330,96)}, rotate = 326.51] [fill={rgb, 255:red, 0; green, 0; blue, 0 }  ][line width=0.08]  [draw opacity=0] (15.6,-3.9) -- (0,0) -- (15.6,3.9) -- cycle    ;
		%Straight Lines [id:da2996672293080438] 
		\draw [line width=1.5]    (382,96) -- (342.51,144.89) ;
		\draw [shift={(340,148)}, rotate = 308.93] [fill={rgb, 255:red, 0; green, 0; blue, 0 }  ][line width=0.08]  [draw opacity=0] (15.6,-3.9) -- (0,0) -- (15.6,3.9) -- cycle    ;
		%Straight Lines [id:da2953666007634068] 
		\draw [line width=1.5]    (311.03,49) -- (378.66,93.79) ;
		\draw [shift={(382,96)}, rotate = 213.51] [fill={rgb, 255:red, 0; green, 0; blue, 0 }  ][line width=0.08]  [draw opacity=0] (15.6,-3.9) -- (0,0) -- (15.6,3.9) -- cycle    ;
		%Straight Lines [id:da6252356353882813] 
		\draw [line width=1.5]    (330,96) -- (368.33,144.85) ;
		\draw [shift={(370.8,148)}, rotate = 231.88] [fill={rgb, 255:red, 0; green, 0; blue, 0 }  ][line width=0.08]  [draw opacity=0] (15.6,-3.9) -- (0,0) -- (15.6,3.9) -- cycle    ;

		% Text Node
		\draw (113.43,216.1) node [anchor=north west][inner sep=0.75pt]  [font=\Large]  {$\left( \pi ^{*} ,\mu ^{*}\right)$};
		% Text Node
		\draw (146.63,18) node  [font=\LARGE] [align=left] {FPI};
		% Text Node
		\draw (88.04,50) node  [font=\Large]  {$\pi _{0}$};
		% Text Node
		\draw (204.43,49.8) node  [font=\Large]  {$\mu _{0}$};
		% Text Node
		\draw (79.5,101) node  [font=\Large]  {$\pi _{\mu _{0}} \!\!=\!\pi _{1}$};
		% Text Node
		\draw (211.03,105) node  [font=\Large]  {$\mu _{1} \!=\!\mu _{\pi _{1}}$};
		% Text Node
		\draw (89.5,155.5) node  [font=\Large]  {$\pi _{\mu _{1}} \!\!=\!\pi _{2}$};
		% Text Node
		\draw (198.43,150) node  [font=\Large]  {$\mu _{2} \!=\!\mu _{\pi _{2}}$};
		% Text Node
		\draw (543,217.4) node [anchor=north west][inner sep=0.75pt]  [font=\Large]  {$\left( Q^{*} ,\mu ^{*}\right)$};
		% Text Node
		\draw (575,18) node  [font=\LARGE] [align=left] {on-policy QMI};
		% Text Node
		\draw (515.44,49) node  [font=\Large]  {$Q_{0}$};
		% Text Node
		\draw (634.83,49.8) node  [font=\Large]  {$M_{0}$};
		% Text Node
		\draw (502.9,99) node  [font=\Large]  {$Q_{M_{0}} \!\!\approx\! Q_{1}$};
		% Text Node
		\draw (645,102.5) node  [font=\Large]  {$M_{1} \!\approx\! \mu _{Q_{1}}$};
		% Text Node
		\draw (515.5,149) node  [font=\Large]  {$Q_{M_{1}} \!\!\approx\! Q_{2}$};
		% Text Node
		\draw (632.83,149) node  [font=\Large]  {$M_{2} \!\approx\! \mu _{Q_{2}}$};
		% Text Node
		\draw (323.83,217.1) node [anchor=north west][inner sep=0.75pt]  [font=\Large]  {$\left( Q^{*} ,\mu ^{*}\right)$};
		% Text Node
		\draw (356,18) node  [font=\LARGE] [align=left] {off-policy QMI};
		% Text Node
		\draw (298.44,49) node  [font=\Large]  {$Q_{0}$};
		% Text Node
		\draw (417.83,49.8) node  [font=\Large]  {$M_{0}$};
		% Text Node
		\draw (289.5,104) node  [font=\Large]  {$Q_{M_{0}} \!\!\approx\! Q_{1}$};
		% Text Node
		\draw (426.43,99) node  [font=\Large]  {$M_{1} \!\approx\! \mu _{Q_{0}}$};
		% Text Node
		\draw (292.5,146) node  [font=\Large]  {$Q_{M_{1}} \!\!\approx\! Q_{2}$};
		% Text Node
		\draw (416,146) node  [font=\Large]  {$M_{2} \!\approx\! \mu _{Q_{1}}$};

	\end{tikzpicture}
	}
	\caption{Illustration of learning processes. Each arrow represents one iteration in FPI or one outer iteration in QMI, matching the end BR or IP with the population distribution or value function at the start. The dashed line in on-policy QMI represents behavior policy updates, making $M_k$ match the updated BR estimate $Q_{k}$.
		%     $Q$ and $M$ denote the BR and IP (estimates) respectively, with numerical subscripts marking the outer iteration number and alphabetic subscripts signifying their associated population/policy. Superscript $*$ indicates MFNEs. Off-policy QMI generates two parallel policy-population sequences, $\{(Q_{2k},M_{2k+1})\}$ and $\{(Q_{2k+1}, M_{2k})\}$, where $k\in\NN$. The dashed line in on-policy QMI represents behavior policy updates, making $M_k$ match the updated BR estimate $Q_{k}$.
	}\label{fig:qmi}
\end{figure}

\subsection{On-Policy QM Iteration} %\label{sec:on}

Still, BR and IP evaluations are executed parallelly in off-policy QMI, and thus its efficiency boost indirectly attributes to parallel computing. This naturally raises a question: can we directly approximate the FPI operator $\Gamma$ in one outer iteration?
The \emph{on-policy} variant of QM iteration provides a positive response.
This time, we pass to \cref{alg} the option \texttt{on-policy} and a general policy operator satisfying \cref{asmp:sarsa}, facilitating dynamic updates of agent's behavior policy within each outer iteration.
By syncing the behavior policy in accordance with the Q-value function, the policy learning process is governed by on-policy SARSA.
Additionally, since the agent now observes the state transition induced by the updated policy, the M-value function is updated towards the population distribution induced by the updated policy.
The learning process of on-policy QMI is illustrated in \cref{fig:qmi}.

On the other hand, constantly changing behavior policies in on-policy QMI yield nonstationary Markov chains.
Such nonstationarity renders the convergence guarantee of off-policy QMI not applicable here and complicates the convergence analysis of on-policy QMI.
Nonetheless, we establish a similar convergence guarantee for on-policy QMI (\cref{lem:sarsa}).
To achieve the sharp logarithmic dependency on $T$, we \emph{mix} the Q-value functions obtained in an outer iteration:
\(\operatorname{mix}\left( \{ Q_{k,l} \}_{l=0}^{t+1} \right) \coloneqq \sum_{l=0}^{t} (w_l /\sum_{l=0}^{t}w_l) Q_{k,l}\),
where $w_l \asymp t$, and use this convex combination to determine the behavior policy.

Theoretical results of on-policy SARSA are used to establish the convergence guarantee of on-policy QMI.
While the instability issue persists as in off-policy QMI, on-policy SARSA's adaptability and versatility---facilitated by its use of general policy operators---outstrip those of Q-learning, thus allowing us to directly impose the smoothness condition on $\Gamma_{\pi}$.

\begin{assumption}[Lipschitz continuous policy operator for SARSA] \label{asmp:sarsa}
	For any $Q_1,Q_2\in\R^{S\times A}$ and $s\in \mathcal{S}$, it holds that
	\[
		\|\Gamma_{\pi}(Q_1)[\cdot\given s] - \Gamma_{\pi}(Q_1)[\cdot\given s]\|_{\mathrm{TV}} \le L\|Q_1-Q_2\|_{2}
		.\]
	Furthermore, the Lipschitz constant satisfies $L\le \lambda_{\min}(1-\gamma)^2 /(2R\sigma)$.
\end{assumption}

\begin{remark}% [Remark on \cref{asmp:ql,asmp:sarsa}]
\cref{asmp:ql} is weaker than \cref{asmp:sarsa} as the latter implies the former (see \cref{lem:diff}) and requires a small Lipschitz constant.
	However, verifying \cref{asmp:ql} can be difficult as the dependence of $P_{Q}$ on $Q$ can be intricate and the model is unknown, whereas \cref{asmp:sarsa} is more achievable given the flexibility in choosing policy operators for SARSA.
	For instance, the softmax function with an apt temperature parameter satisfies \cref{asmp:sarsa} \citep{gao2018PropertiesSoftmax}.
	Actually, a softmax policy operator imposes entropy regularization to the greedy selection \citep{gao2018PropertiesSoftmax}, a common technique used to stabilize the MFG learning process \citep{cui2022ApproximatelySolving,guo2022Entropyregularization,anahtarci2023Qlearningregularized}.
	Absent such regularization, \cref{asmp:ql} ensures training stability.
	Other assumptions have been posited for this purpose, such as a strongly convex Bellman operator \citep{anahtarci2019FittedQlearning}.
	Notably, \cref{asmp:contract} and \ref{asmp:ql} or \ref{asmp:sarsa} are not mutually exclusive; either \cref{asmp:ql} or \ref{asmp:sarsa} with some conditions on the reward function's smoothness and Lipschitz constants can imply \cref{asmp:contract} \citep{guo2019Learningmeanfield,anahtarci2019FittedQlearning}.
	% Here we impose \cref{asmp:ql} to simplify the representation and be consistent with the later assumption in \cref{sec:on}.
\end{remark}

\subsection{Comparison of Off- and On-Policy QMI}

\renewcommand{\check}{\CheckmarkBold}
\newcommand{\uncheck}{\XSolidBrush}
\begin{table}[H]
	\caption{Comparison of off- and on-policy QMI.
		% 	Off-policy QMI fixes the agent's policy per outer iteration, allowing uncorrelated BR and IP estimates (parallel computing). On-policy QMI, in contrast, adaptively updates the behavior using BR estimates, correlating BR and IP estimates (concurrent computing). Also, on-policy QMI learns more robust policies and can handle population-dependent transition kernels, but it requires a regularization assumption for convergence.
	} \label{tab:comp}
	\centering
	\resizebox{0.4\textwidth}{!}{
		\begin{tabular}{l|cc}
			           & \textbf{Off-Policy} & \textbf{On-Policy} \\
			\midrule
			\makecell[l]{Behavior policy within                   \\an outer iteration} & fixed               & adaptive           \\\hline\\[-0.9em]
			\cz{Policy type} & greedy            & soft             \\\hline\\[-0.9em]
			MFNE       & original            & regularized        \\\hline\\[-0.9em]
			\makecell[l]{Sample efficiency                        \\ boost mechanism} & parallel            & concurrent        \\\hline\\[-0.9em]
			\makecell[l]{Population-dependent                     \\transition kernels} & \uncheck            & \check
		\end{tabular}
	}
\end{table}

% "Comparison of" is more frequently used, and also makes this section title one line. see https://ell.stackexchange.com/a/1189

% We conclude this section by summarizing the differences between off- and on-policy variants of QMI.
% \cref{tab:comp} gives an overview of our comparison.
\cref{tab:comp} gives an overview of the differences between off- and on-policy variants of QMI.
By utilizing Q-learning with a greedy policy operator, off-policy QMI can learn a deterministic optimal policy of the original MFG.
% \footnotemark
On-policy QMI, on the other hand, utilizes SARSA with a \emph{soft} (non-deterministic) policy operator, meaning that the learned MFNE depends on the policy operator and corresponds to an \cz{implicitly} regularized MFG. Nevertheless, on-policy QMI affords flexibility in choosing a wider range of policy operators, and the soft policies it acquires exhibit greater robustness \citep{sutton2018Reinforcementlearninga}.
Furthermore, off-policy QMI boosts the sample efficiency by learning two policy-population sequences parallelly, while on-policy QMI directly boosts it by amalgamating the two steps in FPI into one.
Last but not least, on-policy QMI and its convergence guarantee can directly accommodate transition kernels that are dependent not only on behavior policy but also on population distribution. However, such a dependence breaks the parallel procedure in off-policy QMI. See \cref{sec:gmfg} for more discussion on population-dependent transition kernels.

% \footnotetext{\cref{asmp:contract} would be harder to satisfy in this case. Likely other regularization is required \citep{anahtarci2023Qlearningregularized} \placeholder{ [other in-series citations needed]}. There is no free lunch.}

\section{Sample Complexity Analysis} \label{sec:anlys} % 1 page

In this section, we establish the sample complexity guarantee for both the off- and on-policy variants of \cref{alg}, given our assumptions on MDPs (\cref{asmp:ergodic}), MFGs (\cref{asmp:contract}), and smoothness (\cref{asmp:ql} or \cref{asmp:sarsa}).
To assist the analysis, we define the operators presented in \cref{alg}, which correspond to those in \cref{def:fpi-op}.

\begin{definition}[QMI operators] \label{def:qmi-op}
	For off-policy QMI, the Q- and M-value function operators,
	\begin{align}
		\Gamma_{Q}(T) & : \Delta(\mathcal{S})\to \R^{S\times A},\  M_{k,0} \mapsto Q_{k,T} \quad \text{and} \\
		\Gamma_{M}(T) & : \R^{S\times A} \to \Delta(\mathcal{S}),\ Q_{k,0} \mapsto M_{k,T},
	\end{align}
return the updated Q- and M-value function after an outer iteration of \cref{alg}, consisting of $T$ online stochastic updates using Lines~\ref{line:q-update} and \ref{line:m-update}.
	Then, the off-policy QMI operator is the composition of the above two operators:
	\(
	\hat{\Gamma}_{\mathrm{off}}(T) \coloneqq \Gamma_{M}(T) \circ \Gamma_{Q}(T)%: \Delta(\mathcal{S}) \to \Delta(\mathcal{S})
	.\)

	The on-policy QMI operator,
	\[
		\hat{\Gamma}_{\mathrm{on}}(T) : \Delta(\mathcal{S}) \to \Delta(\mathcal{S}),\ M_{k,0} \mapsto M_{k,T}
		,\]
	returns the updated M-value function after an outer iteration of \cref{alg}, consisting of $T$ online stochastic updates using Lines~\ref{line:q-update} and \ref{line:m-update}, as well as the policy updates using Line~\ref{line:p-update}.
\end{definition}

As mentioned in previous sections, the equivalence between off-policy QMI and FPI comes from the convergence of off-policy Q-learning and MCMC. Specifically, we have the following two lemmas.

\begin{lemma}[Sample complexity of Q-learning {\citep[Theorem 7]{qu2020Finitetimeanalysis}}] \label{lem:ql}
	Suppose \cref{asmp:ergodic,asmp:ql} hold for the greedy policy operator.
	With a step size of $\alpha_t \!\asymp\! 1 /(\lambda_{\min}(1\!-\!\gamma)t)$, for any $M\!\!\in\!\!\Delta(\mathcal{S})$, we have
	\[
		\EE\left\| \Gamma_{Q}(T)M - \Gamma_{\mathrm{BR}}M \right\|_{2}^2 = O\left( \frac{SAR^2\log T}{\lambda_{\min}^2(1-\gamma)^{5}T} \right)
		,\]
	where MDP components $\mathcal{S},\mathcal{A},R$, and $\gamma$ are defined in \cref{sec:mfg}, with $S$ and $A$ denote the cardinality of $\mathcal{S}$ and $\mathcal{A}$. $\sigma$ and $\lambda_{\min}$ are defined in \cref{asmp:ergodic}, and $L$ is defined in \cref{asmp:ql}.
\end{lemma}

\begin{lemma}[Sample complexity of stationary MCMC {\citep[Theorem 3.1]{latuszynski2013Nonasymptoticbounds}}] \label{lem:mcmc}
	Suppose \cref{asmp:ergodic} holds.
	With a step size of $\beta_t \asymp 1 /t$, for any $Q\in \R^{S\times A}$, we have
	\[
		\EE\left\| \Gamma_{M}(T)Q - \Gamma_{\mathrm{IP}}Q \right\|_{2}^2 = O\left( \frac{SA}{(1-\rho)^2T} \right)
		.\]
\end{lemma}

The preceding lemmas demonstrate that off-policy QMI efficiently approximates FPI, with the Q-value and M-value updates evaluating the BR and IP, respectively.

\cref{lem:ql,lem:mcmc} are not applicable to on-policy QMI, where transition kernels are nonstationary.
Nonetheless, we can establish the following lemma.

\begin{lemma}[Sample complexity of nonstationary MCMC with SARSA] \label{lem:sarsa}
	Suppose \cref{asmp:ergodic,asmp:sarsa} hold for the chosen policy operator.
	With a step size of $\alpha_t \asymp 1 /(\lambda_{\min}(1-\gamma)t)$ and $\beta_t \asymp 1 /t$, for any $M\in \Delta(\mathcal{S})$, we have
	\[
		\EE\left\| \hat{\Gamma}_{\mathrm{on}}(T)M - \Gamma M \right\|_{2}^2 = O\left( \frac{SAR^2L^2\sigma^2\log T}{\lambda^2_{\min}(1-\gamma)^{4}T} \right)
		.\]
\end{lemma}

% fig:rr
\begin{figure*}
	\centering
	\begin{subfigure}[b]{0.26\textwidth}
		\centering
		\includegraphics[width=\linewidth]{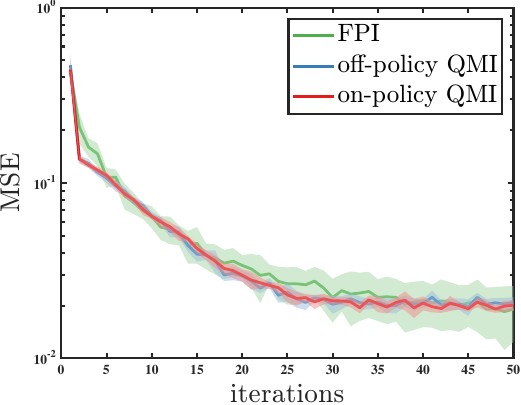}
		\caption{Mean squared error}
	\end{subfigure}
	% \hfill
	\hspace{0.5em}
	\begin{subfigure}[b]{0.26\textwidth}
		\centering
		\includegraphics[width=\linewidth]{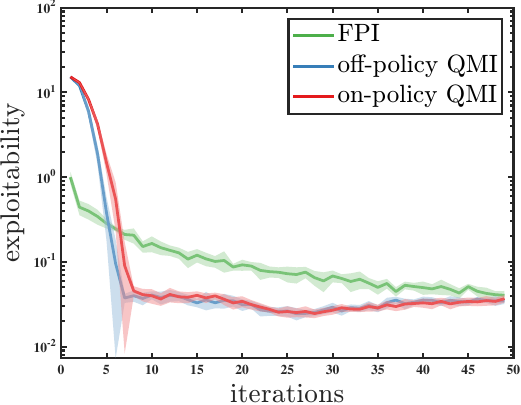}
		\caption{Exploitability}
	\end{subfigure}
	% \hfill
	\begin{subfigure}[b]{0.288\textwidth}
		\centering
		\includegraphics[width=0.888\linewidth]{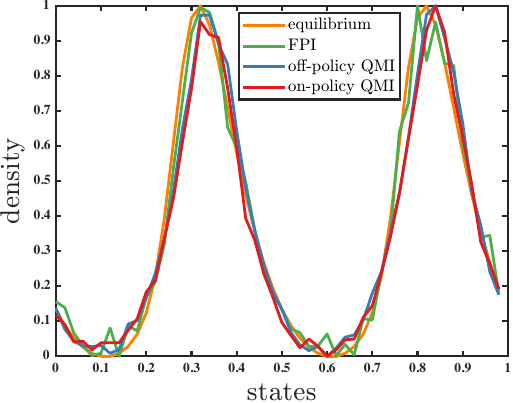}
		\caption{Learned population distributions}
	\end{subfigure}
 \vspace{0.5cm}
	\caption{Performance comparison of FPI, off-policy QMI, and on-policy QMI on ring road speed control.
		MSE represents the mean squared $L_2$ error between the current M-value function and the MFNE population distribution.
		Exploitability refers to the disparity between the current value function and the BR w.r.t. the current population distribution. %(see \cref{sec:apx-exp} and \citet{perrin2020Fictitiousplay}).
		Learned population distributions are scaled for better visualization.
	}
	\label{fig:rr}
\end{figure*}

% fig:graph
\begin{figure*}
	\centering
	\begin{subfigure}{0.26\textwidth}
		\centering
		\includegraphics[width=\linewidth]{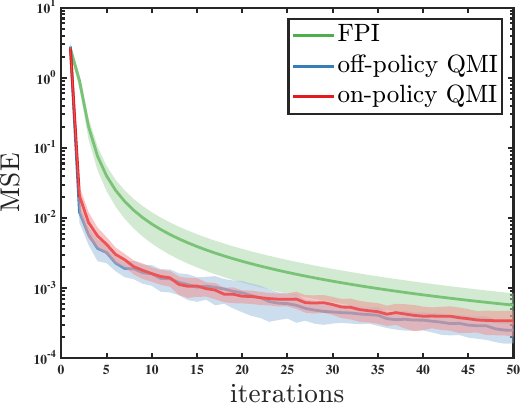}
		\caption{Mean squared error}
	\end{subfigure}
	% \hfill
	\hspace{0.5em}
	\begin{subfigure}[b]{0.26\textwidth}
		\centering
		\includegraphics[width=\linewidth]{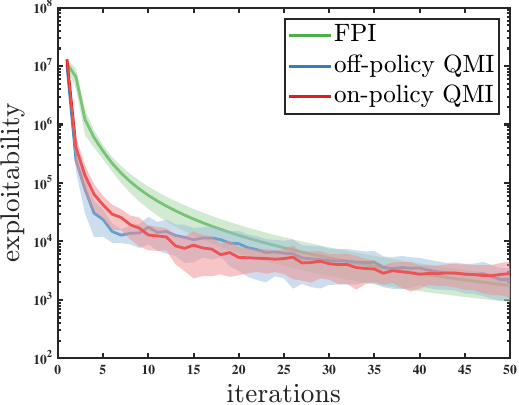}
		\caption{Exploitability}
	\end{subfigure}
	% \hfill
	\hspace{0.5em}
	\begin{subfigure}[b]{0.25\textwidth}
		\centering
		\includegraphics[width=0.98\linewidth]{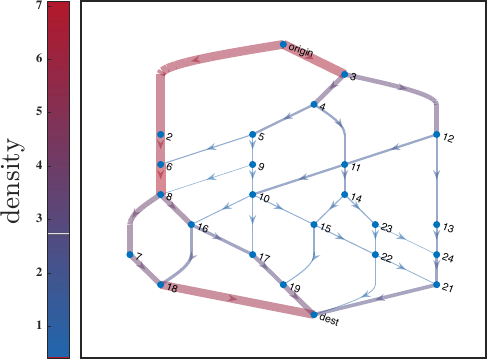}
		\vspace{0.3cm}
		\caption{Learned population distribution} %\label{fig:g-mu}
	\end{subfigure}\hspace{1em}
 \vspace{0.5cm}
	\caption{Performance comparison of FPI, off-policy QMI, and on-policy QMI on Sioux Falls network routing. Only the population distribution learned by off-policy QMI is shown in (c); other methods% and the MFNE 
	give similar population distributions (please refer to \cref{sec:exp-rr}).}
	\label{fig:graph}
	\vspace{0.5em}
\end{figure*}

An outer iteration of on-policy QMI corresponds to a nonstationary MCMC. Thus,
to prove \cref{lem:sarsa}, we employ a \emph{backtracking} procedure, a technique developed in \citet{zou2019Finitesampleanalysis} to address nonstationarity in stochastic approximation methods.
The key idea is that in order to exploit the mixing property of stationary Markov chains (\cref{asmp:ergodic}), we virtually backtrack a period $\tau$, and generate a virtual trajectory where the agent follows the fixed behavior policy $\pi_{t-\tau}:= \Gamma_{\pi}(Q_{t-\tau})$ after time step $t-\tau$.
This virtual trajectory is stationary after time step $t-\tau$ and rapidly converges to the steady distribution induced by $\pi_{t-\tau}$, denoted by $\mu_{t-\tau}$. Next, the convergence of SARSA confirms that $\mu_{t-\tau}$ converges to the steady distribution induced by the BR w.r.t. $M$, denoted by $\mu_{M}\coloneqq \Gamma M$.
Let $s_t$ and $\tilde{s}_t$ be the state at time step $t$ on the actual and virtual trajectories, respectively.
Let $\pi_t$ be the (actual) behavior policy at time step $t$, with its induced steady distribution denoted as $\mu_t$.
Then, the proof sketch for \cref{lem:sarsa} can be succinctly portrayed as:
\[
	\mathrlap{\underbrace{\phantom{\strut s_t \approx \tilde{s}_t}}_{H_1,\text{backtrack}}}
	\ \ s_t\,\approx \,\overbrace{\tilde{s}_t \overunderset{d}{\scriptscriptstyle\tau\to\infty}{\longrightarrow} s}^{H_2,\text{mix}}
	\sim
	\mathrlap{\underbrace{\phantom{\strut\mu_{t-\tau} \approx \mu_t}}_{H_{3},\text{progress}}}
	\mu_{t-\tau} \approx \overbrace{\mu_t \overunderset{L_2}{\scriptscriptstyle t\to\infty}{\longrightarrow} \mu_{M}}^{H_4,\text{SARSA}}
	,\]
	where the backtracking discrepancy $H_1$ and the distribution progress $H_3$ are controlled by the virtual period $\tau$ (\cref{lem:progress,lem:back}), while the two convergence rates $H_2$ and $H_4$ are characterized by the geometric ergodicity of stationary MDPs and the sample complexity of SARSA (\cref{lem:control,lem:mix}), respectively.
In brief, we show that the agent's state distribution, and thus its M-value function, rapidly converges to the IP $\mu_{M}$.

Given the above lemmas, we are ready to compose the convergence guarantee and sample complexity of QMI.

\begin{theorem}[Sample complexity of QMI] \label{thm}
	Suppose \cref{asmp:contract,asmp:ergodic} hold, and \cref{asmp:ql,asmp:sarsa} hold for off- and on-policy QMI, respectively.
	Let $\mu^*$ be the MFNE population distribution.
	Then \cref{alg} returns an $\epsilon$-approximate MFNE, that is,
	\[
		\EE\left\| M_{K,T} - \mu^* \right\|_{2}^2 =
		\EE\| \hat{\Gamma}(T)^{K}M_0 - \mu^* \|_{2}^2 \le \epsilon^2
		,\]
	where $\hat{\Gamma}$ can be either $\hat{\Gamma}_{\mathrm{off}}$ or $\hat{\Gamma}_{\mathrm{on}}$,
	with the number of iterations being at most
	\[
		K = O\left( \kappa^{-1}\log \epsilon^{-1} \right),\quad
		T = C \cdot O\left( \kappa^{-2}\epsilon^{-2} \log \epsilon^{-1} \right)
		,\]
	where
	\[
		C \le \frac{SAR^2L^2\sigma^2}{\lambda^2_{\min}(1-\gamma)^{5}}
		.\]
	% where the asymptotic notation $\widetilde{O}$ suppresses the logrithmatic dependences on $\epsilon$.
\end{theorem}

Our complexity results match the prior work on learning MFGs \citep{anahtarci2023Qlearningregularized} and are consistent with stochastic approximation methods \citep{latuszynski2013Nonasymptoticbounds,zou2019Finitesampleanalysis,qu2020Finitetimeanalysis}.

\ifSubfilesClassLoaded{\bibliography{mfg}}{}

\section{Numerical Experiments} \label{sec:exp} % 1 page

In this section, we present two experiments demonstrating the effectiveness of our methods.
We compare our methods with model-based FPI using two key metrics: the mean squared error (MSE) of the population distribution and the exploitability of the policy.
For a finite state space with the trivial metric, the total variation distance equals the 1-Wasserstein distance \citep{gibbs2002ChoosingBounding}, and is equivalent to the Euclidean norms.
Thus, we consider the $L_2$ MSE between the current M-value function and the MFNE population distribution:
\[
	\operatorname{MSE}(M) \coloneqq \|M_{k} - \mu^*\|_{2}^2 = \sum_{s\in \mathcal{S}} \left(M_{k}(s) - \mu^*(s)\right)^2
	.\]
\citet{perrin2020Fictitiousplay} defines the exploitability of a policy as follows:
\[
	\operatorname{exploitability}(\pi)  \coloneqq \max_{\pi'} \mathbb{E}_{s\sim \mu_{\pi}} V(s;\pi',\mu_{\pi}) - \mathbb{E}_{s\sim \mu_{\pi}}V(s;\pi,\mu_{\pi})
	,\]
where $V(s;\pi,\mu)$ is the value function determined by policy $\pi$ and population distribution $\mu$.
Given a policy operator, the exploitability quantifies the gap between the current value function $Q$ and the BR w.r.t. the population distribution induced by $Q$. We denote this BR by $Q_{\m{Q}}$, and calculate $\|Q_{\m{Q}}-Q\|$ as the exploitability in practice.

For model-based FPI, we use value iteration to calculate BRs \citep{sutton2018Reinforcementlearninga}:
\[\label{eq:V}
	V_{t+1}(s) = \max_{a\in \mathcal{A}} \left\{r(s,a,\mu) + \gamma \sum_{s'} P(s'\given s,a)V_t(s')\right\}
	,\]
and the induced population distributions are directly calculated using \cref{eq:transition}.
Model-based FPI assumes full knowledge of the state-action space, reward function, as well as transition dynamics.
During each iteration of value iteration and population update using \cref{eq:transition}, all $S$ values are updated without any random sampling, and we refer to such an iteration as a \emph{sweep}. It is expected that in order for online sampling to replicate the effects of a sweep, the number of samples should be at least $S$. Furthermore, since QMI assumes no knowledge of the action space and the reward function, it may require $A$ samples to achieve the same effect as the $\max$ calculation in \cref{eq:V}. The randomness in sampling can further impact efficiency.
Therefore, we introduce a \emph{sample compensation factor} $\eta$ to relate the number of samples to the number of sweeps. Specifically, let $T_{\mathrm{QMI}}$ and $T_{\mathrm{FPI}}$ be the number of inner iterations of QMI (\cref{alg}) and the number of sweeps of value iteration and population update in FPI respectively; we let
\[
	T_{\mathrm{QMI}} = \eta S T_{\mathrm{FPI}}
	.\]

	Please refer to \cref{sec:apx-exp} for additional experiments on different sample compensation factors, which suggest that a small sample compensation factor is sufficient for QMI.
In this section, we fix $\eta$ as an algorithmic parameter.
Please note that we do not claim that our methods outperform FPI in all scenarios as they are different types of algorithms designed for different situations.

\paragraph{Speed control on a ring road.}

\iffalse
We consider a speed control game of autonomous vehicles on a ring road, the example presented in \cref{fig:toy}. We design the cost function for this goal based on the Lighthill-Whitham-Richards function:
\[
	r(s,a,\mu) = - \frac{1}{2}\left(b(s) + \frac{1}{2}\left(1-\frac{\mu(s)}{\mu_{\mathrm{jam}}}\right) - \frac{a}{a_{\mathrm{max}}}\right)^2  \Delta s
	,\]
where $b$ encodes the location preference, $\mu_{\mathrm{jam}}$ is the jam density, and $a_{\mathrm{max}}$ is the maximum speed.
The performance comparison is reported in \cref{fig:rr}.
% Specifically, we choose the stimulus function as $b(s) = 0.2 (\sin(4\pi s) + 1)$ and set $\mu_{\mathrm{jam}} = 3 / S$ and $a_{\mathrm{max}}=1$.
\fi

We consider a speed control game of autonomous vehicles on a ring road, i.e., the unit circle $\SS^{1} \cong [0,1)$, as illustrated in \cref{fig:toy}.
At location $s\in \SS^{1}$, the representative vehicle selects a speed $a$, and then moves to the next location following transition \(s' = s+a\Delta t \pmod{1}\),
where $\Delta t$ is the time interval between two consecutive decisions.
Without loss of generality, we assume that the speed is bounded by $1$, i.e., the speed space is also $[0,1)$.
Then we discretize both the location space and the speed space using a granularity of $\Delta s = \Delta a = 0.02$.
Thus, both our discretized state (location) space and action (speed) space can be represented by $\mathcal{S} = \mathcal{A} = \{0, 0.02, \ldots, 0.98\} \cong [50]$.
By the Courant-Friedrichs-Lewy condition, we choose the time interval to be $\Delta t = 0.02 \le \Delta s / \max a$.
The objective of a vehicle is to maintain some desired speed while avoiding collisions with other vehicles. Thus, it needs to reduce the speed in areas with high population density.
A classic cost function for this goal is the Lighthill-Whitham-Richards function:
\[
	r^{(\mathrm{LWR})}(s,a,\mu) = - \frac{1}{2}\left(\left(1-\frac{\mu(s)}{\mu_{\mathrm{jam}}}\right) - \frac{a}{a_{\mathrm{max}}}\right)^2  \Delta s
	,\]
where $\mu_{\mathrm{jam}}$ is the jam density, and $a_{\mathrm{max}}$ is the maximum speed.
However, in an infinite horizon game, this cost function induces a \emph{trivial} MFNE, where the equilibrium policy and population are both constant across the state space.
Therefore, we introduce a stimulus term $b$ that varies across different locations:
\[
	r(s,a,\mu) = - \frac{1}{2}\left(b(s) + \frac{1}{2}\left(1-\frac{\mu(s)}{\mu_{\mathrm{jam}}}\right) - \frac{a}{a_{\mathrm{max}}}\right)^2  \Delta s
	,\]
where the factor of one-half before the population distribution term is included to account for the presence of the new stimulus term.
This new cost function makes the MFNE more complex and corresponds to real-world situations where vehicles may have distinct desired speeds at different locations due to environmental variations.
Specifically, we choose the stimulus term as $b(s) = 0.2 (\sin(4\pi s) + 1)$, and set $\mu_{\mathrm{jam}} = 3 / S$ and $a_{\mathrm{max}}=1$. The performance comparison is reported in \cref{fig:rr}.

% The other algorithmic parameters are chosen as follows: the discount factor is set as $\gamma = 1-\Delta s = 0.98$, the initial value function is set as all-zero, the initial state and initial population distribution is randomly generated. We use the softmax function with an increasing temperature of $50k$ as the policy operator in the $k$th outer iteration.
% The effective number of outer iterations is $K = 50$ and the number of sweeps in FPI is $T_{\mathrm{FPI}} = 20$.
% All the results are averaged over $10$ independent runs.
%
% With a sample compensation factor of $\eta_{\mathrm{off}} = 2$ for off-policy QMI and $\eta_{\mathrm{on}} = 3$ for on-policy QMI, both variants achieve a similar efficacy as FPI, validating our methods. The results are reported in \cref{fig:rr}.

\paragraph{Routing game on a network.}
We consider a routing game on the Sioux Falls network, a graph with $24$ nodes and $74$ directed edges. We designate node $1$ as the starting point and node $20$ as the destination.
To construct an infinite-horizon game, we add a \emph{restart} edge $e_{75}$ from the destination back to the starting point.
On each edge, a vehicle selects its next edge to travel to. We consider a deterministic environment, meaning that the vehicle will follow the chosen edge without any randomness.
Therefore, both the state space and the action space can be represented by the edge set, i.e., $\mathcal{S} = \mathcal{A} = \{e_1,\ldots,e_{75}\} \cong [75]$, where $e_{75}$ is the restart edge.
It is worth noting that a vehicle can only select from the outgoing edges of its current location as its next edge.
% We denote $\mathcal{A}_{s}$ as the outgoing edges of $s\in \mathcal{S}$, i.e., the action subset at state $s$.

The objective of a vehicle is to reach the destination as fast as possible. Due to congestion, a vehicle spends a longer time on an edge with higher population distribution. Specifically, the cost (time) on a non-restart edge is \(r^{(\text{cong.})}(s,a,\mu) = - c_1\mu(s)^2 \mathbbm{1}  \{s \neq e_{75}\} \),
where $c_1$ is a cost constant. To drive the vehicle to the destination, we impose a reward at the restart edge: $r^{(\text{term.})}(s,a,\mu) = c_2 \mathbbm{1}\{s = e_{75}\}$. Together, we get the cost function:
\[
	r(s,a,\mu) = \underbrace{- c_1\mu(s)^2 \mathbbm{1} \{s \neq e_{75}\}}_{\text{congestion cost}}  + \underbrace{c_2 \mathbbm{1}\{s = e_{75}\}
	}_{\text{terminal reward}}
	.\]
The performance comparison is reported in \cref{fig:graph}.

All numerical results are averaged over $10$ independent runs. 
\cz{They demonstrate that QMI efficiently approximates FPI and achieves comparable performance, validating our fully online model-free approach.}
Please refer to \cref{sec:apx-exp} for the full setups of two experiments and additional results.

\section{Conclusion}
This study introduces the QM iteration (QMI), a novel online single-agent model-free learning scheme for mean field games, offering a practical alternative to traditional fixed-point iteration methods. QMI provides both theoretically and numerically confirmation of the statement: \emph{a single online agent can efficiently learn the equilibria of mean field games}, without any prior knowledge of the environment.
We anticipate that our methods can provide a benchmark for online model-free learning in MFGs, and serve as a base scheme for further extensions, generalizations, and applications.

%%%%%%%%%%%%%%%%%%%%%%%%%%%%%%%%%%%%%%%%%%%%%%%%%%%%%%%%%%%%%%%%%%%%%%%%

%%% Use this environment to include acknowledgements (optional).
%%% This will be omitted in doubleblind mode.

%%%%%%%%%%%%%%%%%%%%%%%%%%%%%%%%%%%%%%%%%%%%%%%%%%%%%%%%%%%%%%%%%%%%%%%%

%%% Use this command to include your bibliography file.

\bibliography{mfg}
%\lipsum
\onecolumn
\part*{Appendix}
\renewcommand{\thesection}{\Alph{section}}

\section{Additional Experiments} \label{sec:apx-exp}

\subsection{Speed Control on a Ring Road} \label{sec:exp-rr}

% fig:rr-eta
\begin{figure}%[H]
	\centering
	\begin{subfigure}[b]{0.4\textwidth}
		\centering
		\includegraphics[width=\textwidth]{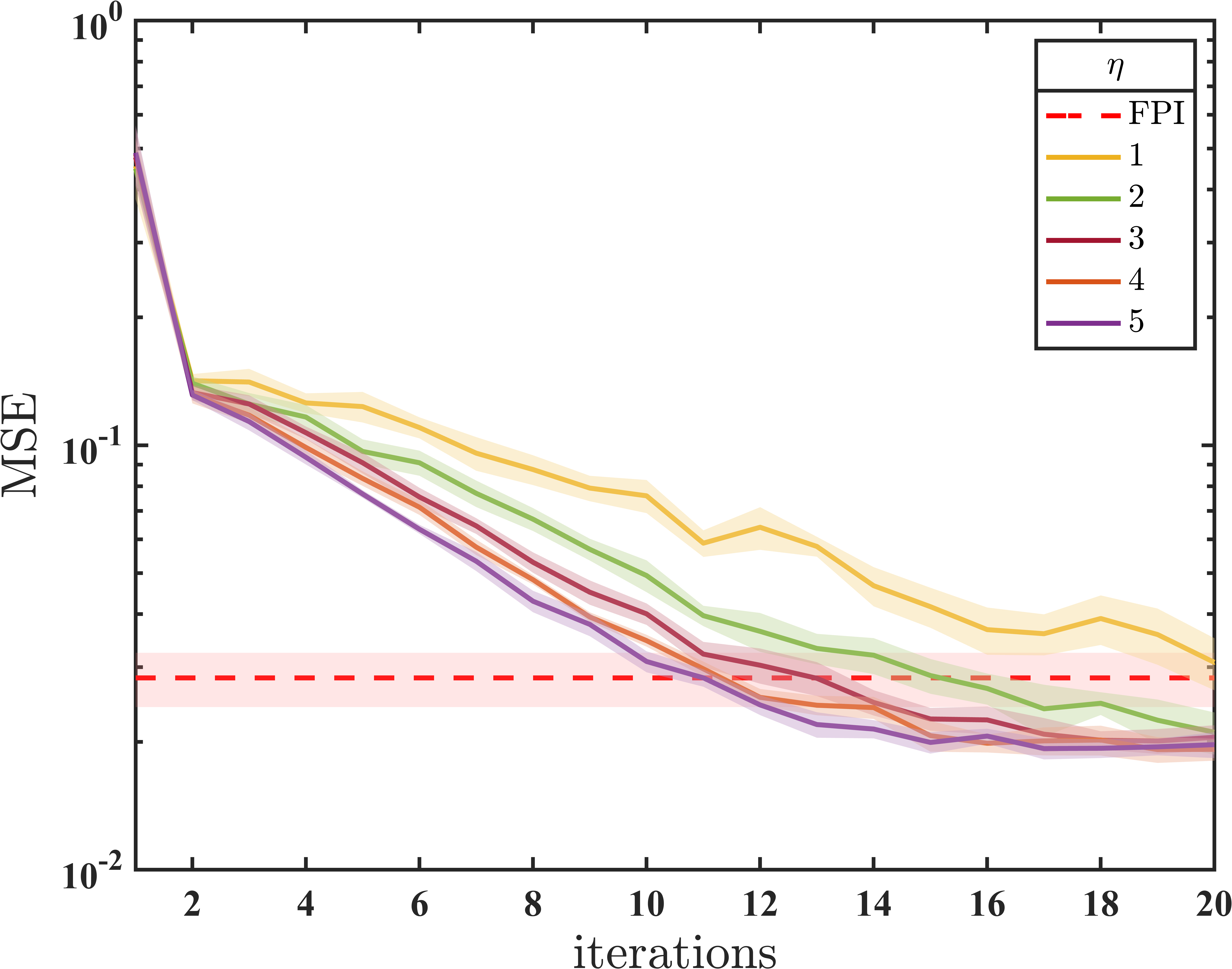}
		\caption{Mean squared error of off-policy QMI}
	\end{subfigure}
	\begin{subfigure}[b]{0.4\textwidth}
		\centering
		\includegraphics[width=\textwidth]{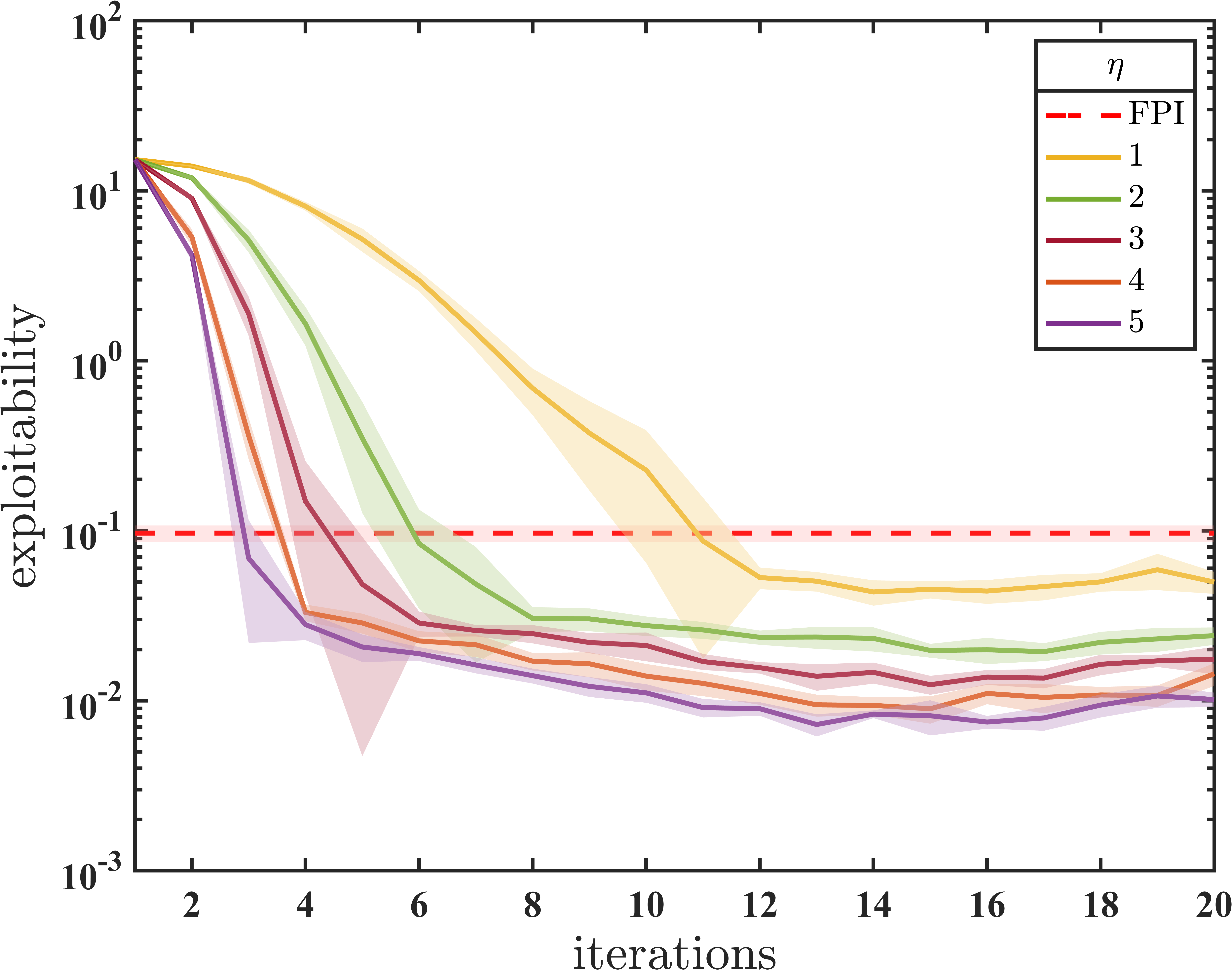}
		\caption{Exploitability of off-policy QMI}
	\end{subfigure}
	\par\bigskip
	\begin{subfigure}[b]{0.4\textwidth}
		\centering
		\includegraphics[width=\textwidth]{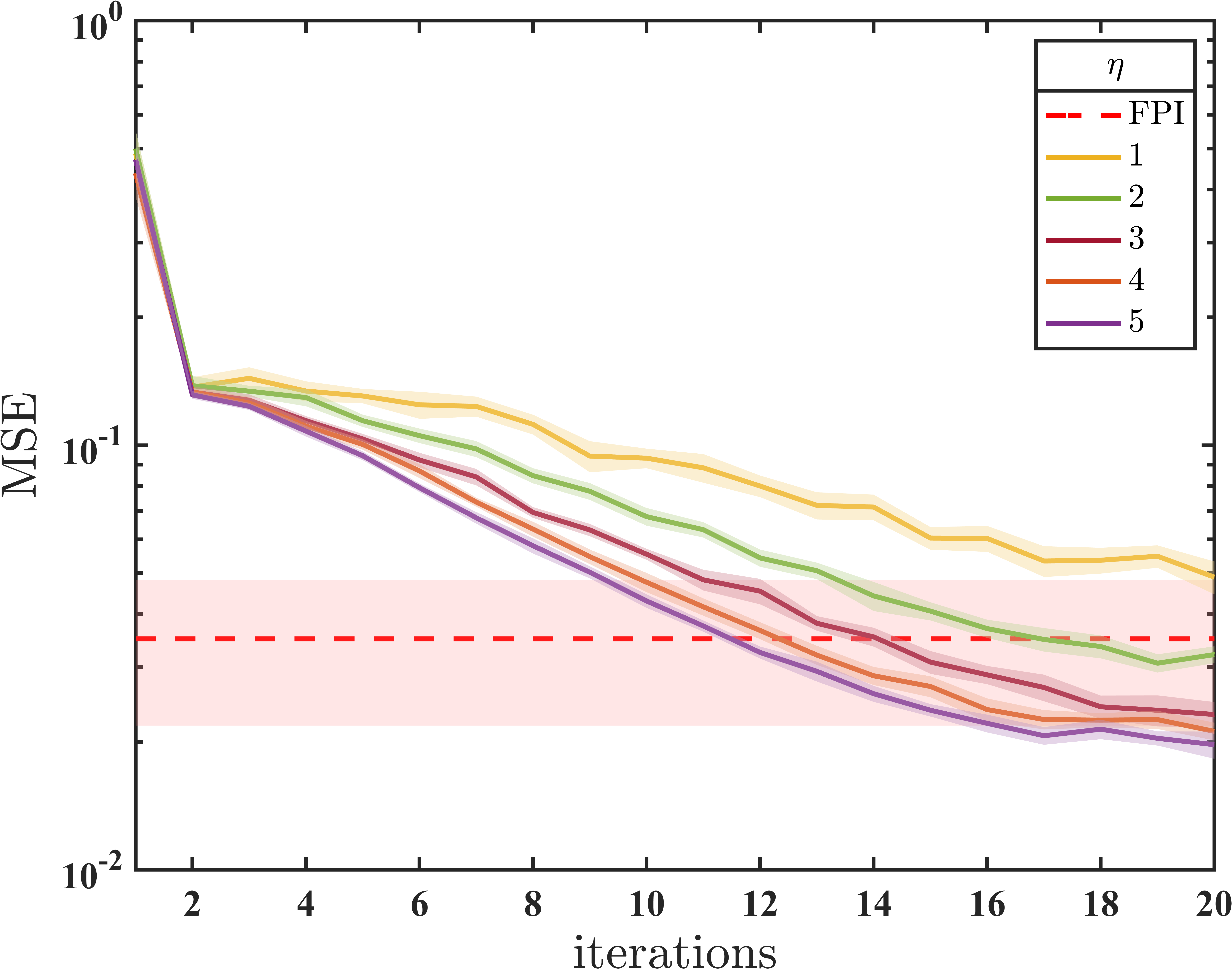}
		\caption{Mean squared error of on-policy QMI}
	\end{subfigure}
	\begin{subfigure}[b]{0.4\textwidth}
		\centering
		\includegraphics[width=\textwidth]{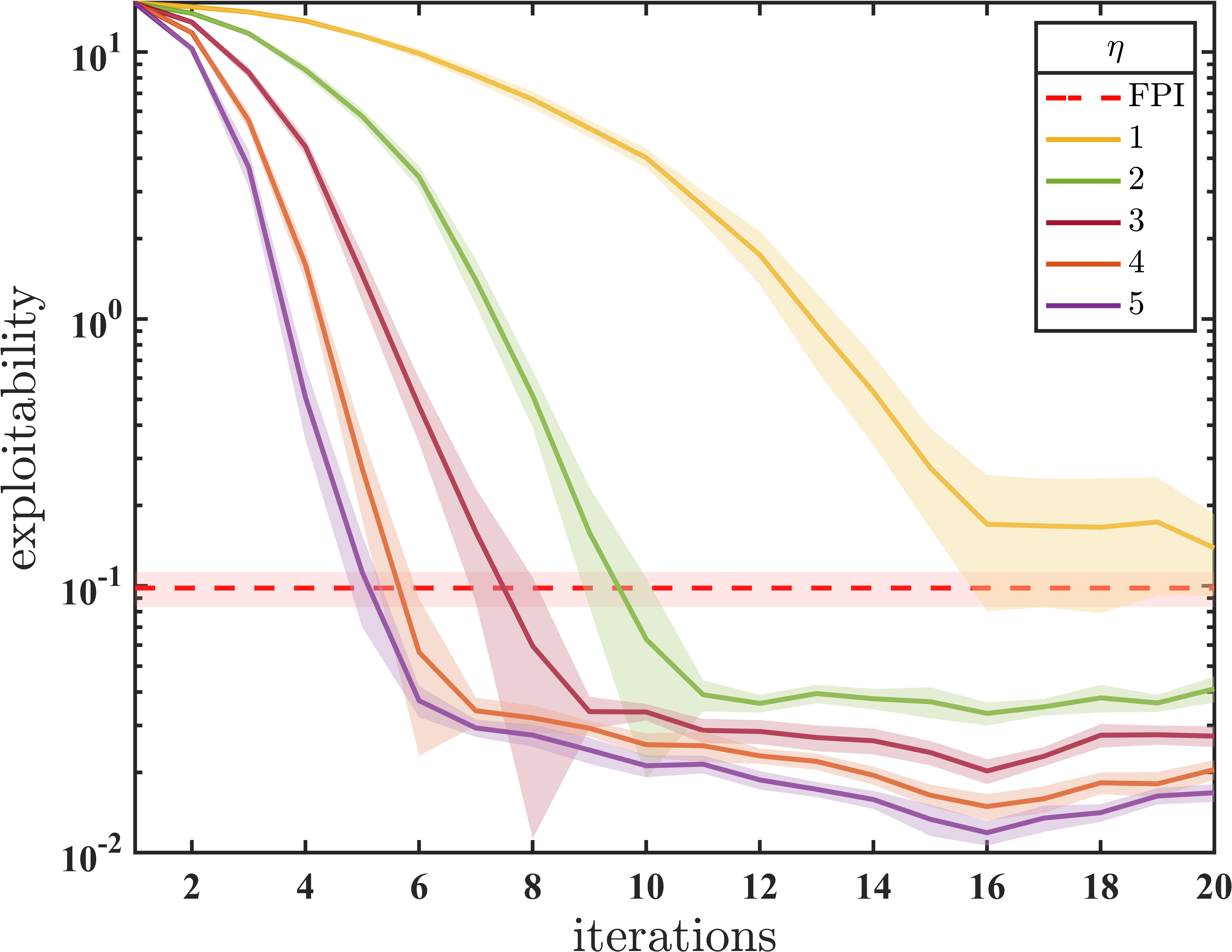}
		\caption{Exploitability of on-policy QMI}
	\end{subfigure}
 \vspace{0.5cm}
	\caption{Performance comparison of different sample compensation factors on ring road speed control. As a baseline, the performance of FPI after 20 iterations is plotted as dashed lines.}
	\label{fig:rr-eta}
 \vspace{0.5cm}
\end{figure}

% fig:rr-t
\begin{figure}[ht]%[H]
	\centering
	\begin{subfigure}[b]{0.4\textwidth}
		\centering
		\includegraphics[width=\textwidth]{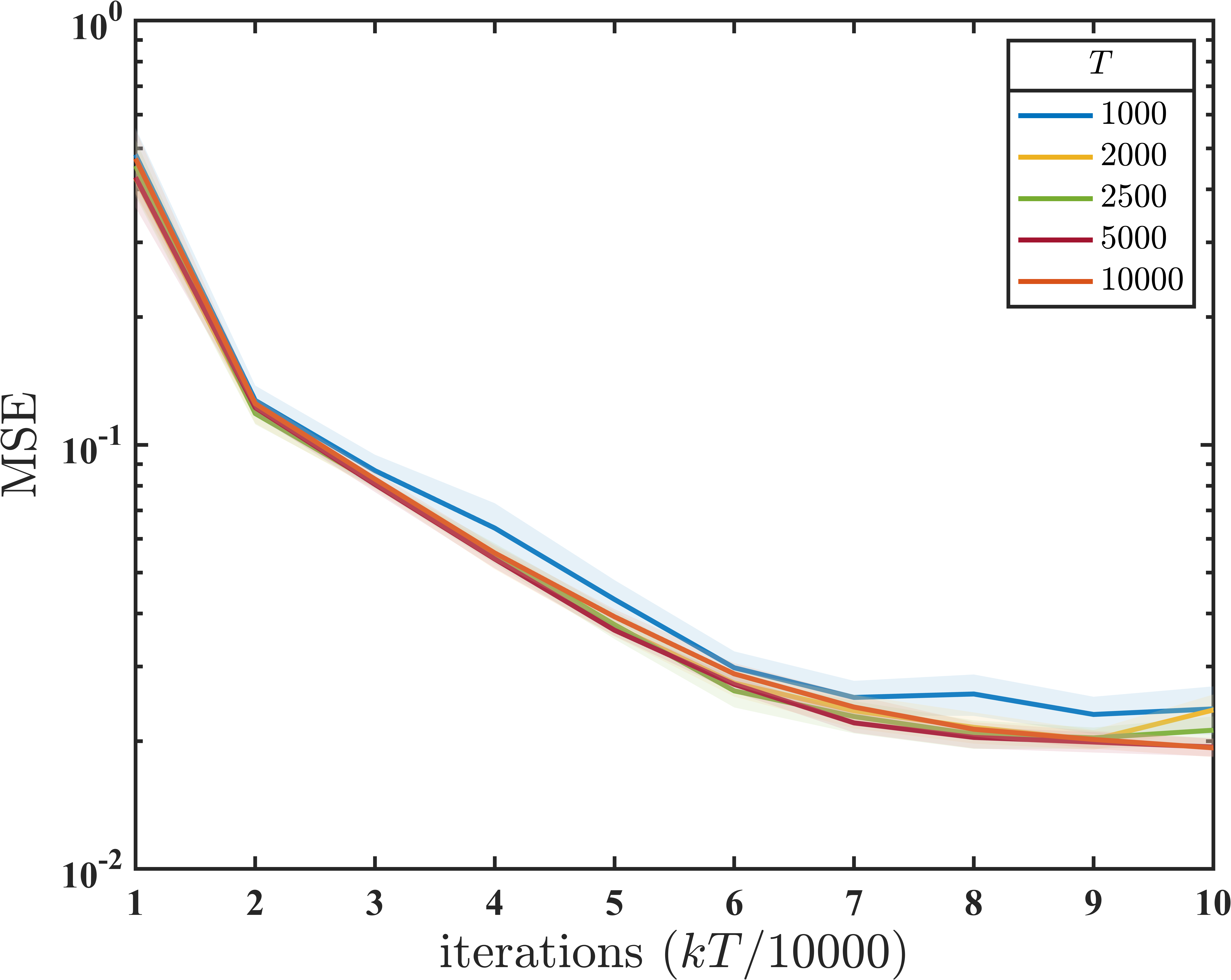}
		\caption{Mean squared error of off-policy QMI}
	\end{subfigure}
	\begin{subfigure}[b]{0.4\textwidth}
		\centering
		\includegraphics[width=\textwidth]{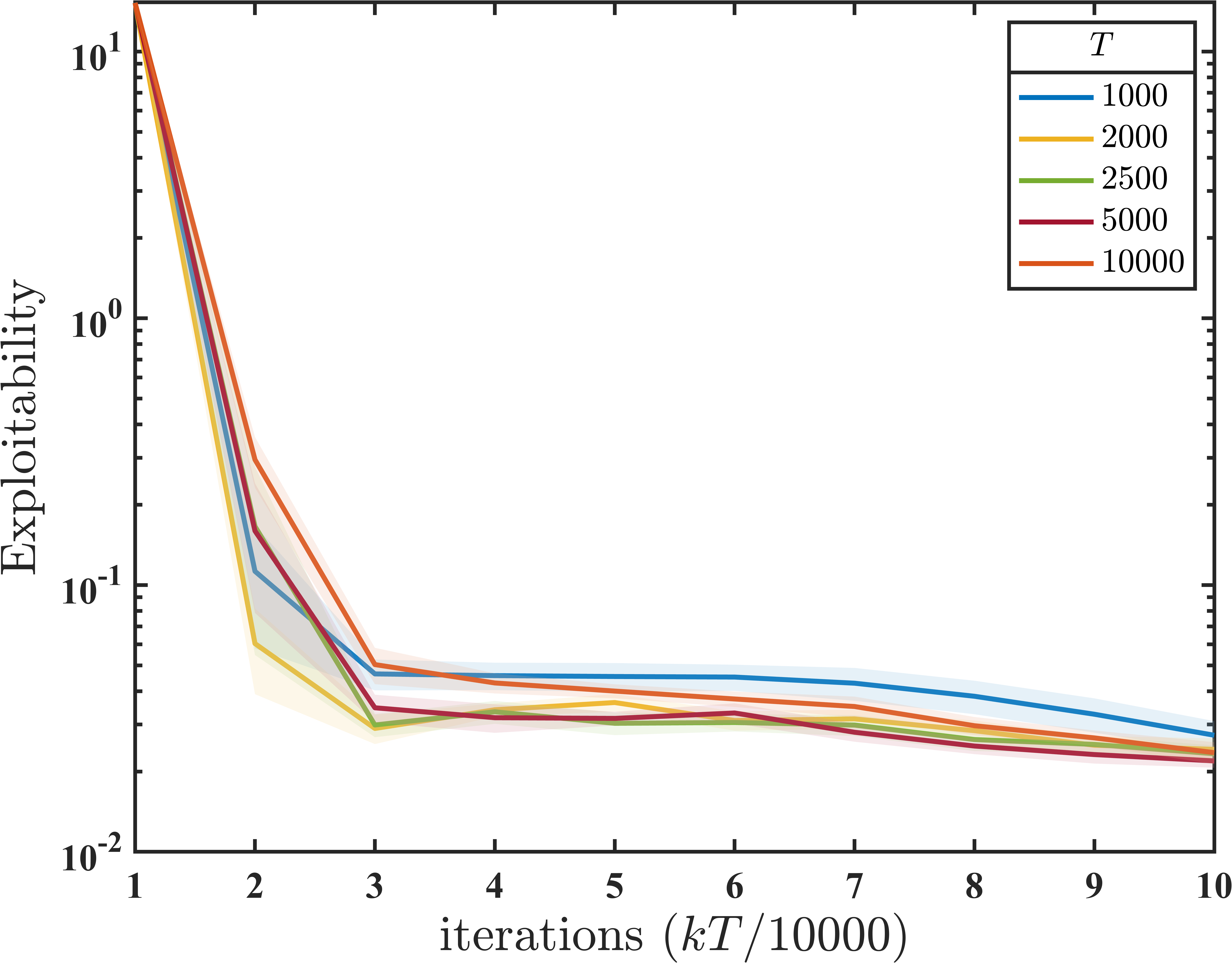}
		\caption{Exploitability of off-policy QMI}
	\end{subfigure}
	\par\bigskip
	\begin{subfigure}[b]{0.4\textwidth}
		\centering
		\includegraphics[width=\textwidth]{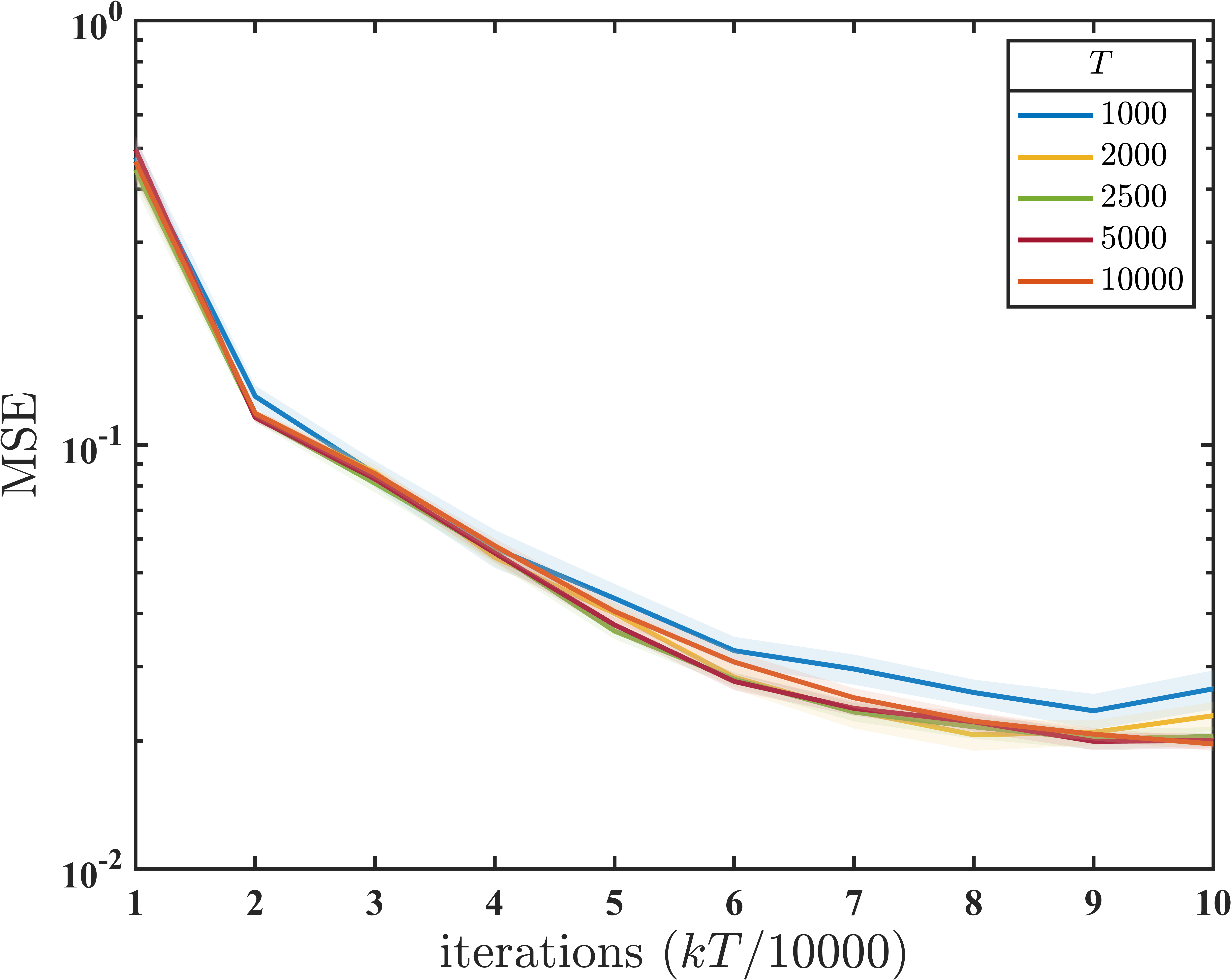}
		\caption{Mean squared error of on-policy QMI}
	\end{subfigure}
	\begin{subfigure}[b]{0.4\textwidth}
		\centering
		\includegraphics[width=\textwidth]{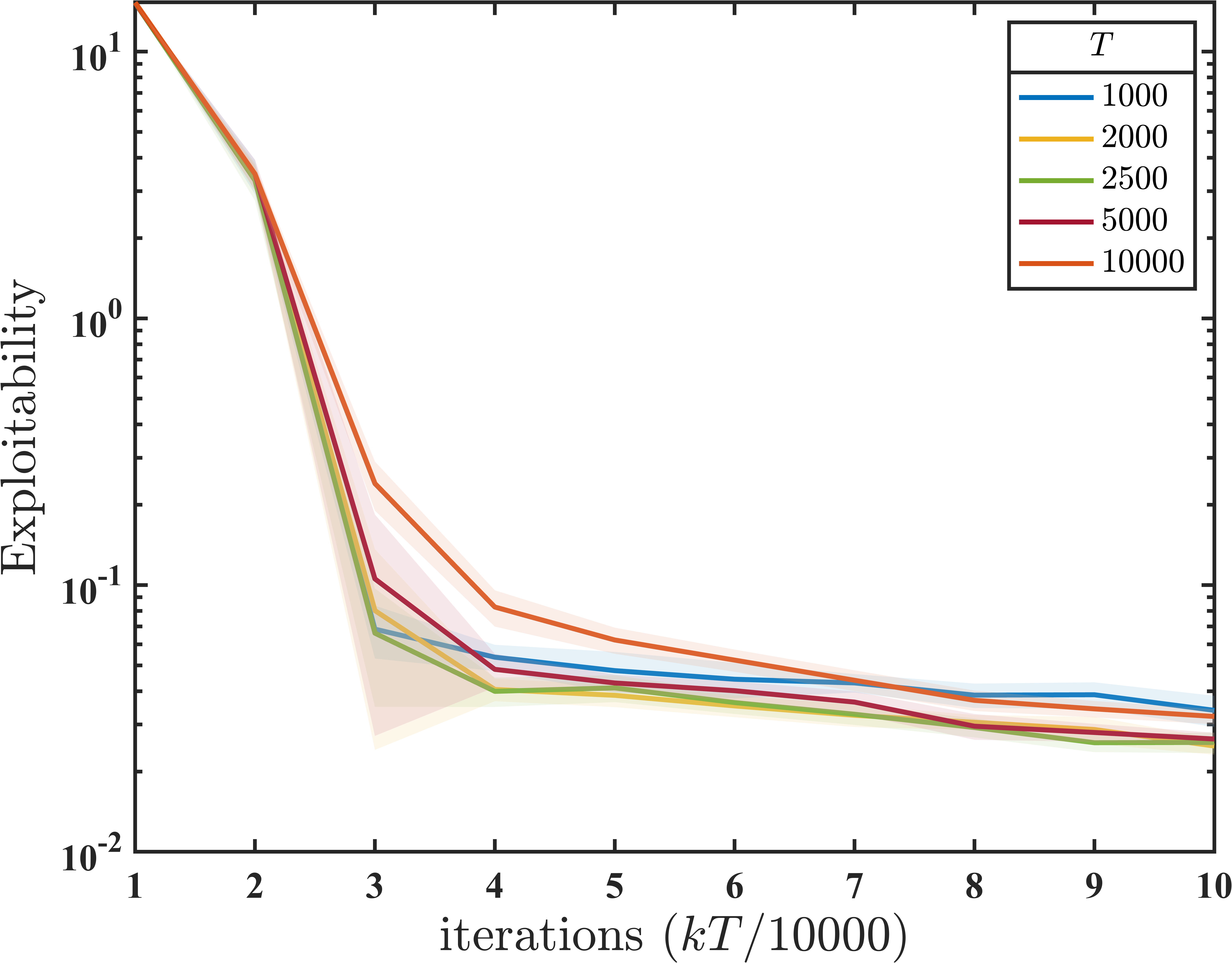}
		\caption{Exploitability of on-policy QMI}
	\end{subfigure}
 \vspace{0.5cm}
	\caption{Performance comparison of different number of inner iterations given a fixed total sample size: $KT=10^5$ on ring road speed control.}
	\label{fig:rr-t}
 \vspace{0.5cm}
\end{figure}

We consider a speed control game of autonomous vehicles on a ring road, i.e., the unit circle $\SS^{1} \cong [0,1)$, as illustrated in \cref{fig:toy}.
At location $s\in \SS^{1}$, the representative vehicle selects a speed $a$, and then moves to the next location following transition \(s' = s+a\Delta t \pmod{1}\),
where $\Delta t$ is the time interval between two consecutive decisions.
Without loss of generality, we assume that the speed is bounded by $1$, i.e., the speed space is also $[0,1)$.
Then we discretize both the location space and the speed space using a granularity of $\Delta s = \Delta a = 0.02$.
Thus, both our discretized state (location) space and action (speed) space can be represented by $\mathcal{S} = \mathcal{A} = \{0, 0.02, \ldots, 0.98\} \cong [50]$.
By the Courant-Friedrichs-Lewy condition, we choose the time interval to be $\Delta t = 0.02 \le \Delta s / \max a$.
The objective of a vehicle is to maintain some desired speed while avoiding collisions with other vehicles. Thus, it needs to reduce the speed in areas with high population density.
A classic cost function for this goal is the Lighthill-Whitham-Richards function:
\[
	r^{(\mathrm{LWR})}(s,a,\mu) = - \frac{1}{2}\left(\left(1-\frac{\mu(s)}{\mu_{\mathrm{jam}}}\right) - \frac{a}{a_{\mathrm{max}}}\right)^2  \Delta s
	,\]
where $\mu_{\mathrm{jam}}$ is the jam density, and $a_{\mathrm{max}}$ is the maximum speed.
However, in an infinite horizon game, this cost function induces a \emph{trivial} MFNE, where the equilibrium policy and population are both constant across the state space.
Therefore, we introduce a stimulus term $b$ that varies across different locations:
\[
	r(s,a,\mu) = - \frac{1}{2}\left(b(s) + \frac{1}{2}\left(1-\frac{\mu(s)}{\mu_{\mathrm{jam}}}\right) - \frac{a}{a_{\mathrm{max}}}\right)^2  \Delta s
	,\]
where the factor of one-half before the population distribution term is included to account for the presence of the new stimulus term.
This new cost function makes the MFNE more complex and corresponds to real-world situations where vehicles may have distinct desired speeds at different locations due to environmental variations.
Specifically, we choose the stimulus term as $b(s) = 0.2 (\sin(4\pi s) + 1)$, and set $\mu_{\mathrm{jam}} = 3 / S$ and $a_{\mathrm{max}}=1$.

We also experiment with different sample compensation factors. The results are reported in \cref{fig:rr-eta}.
As we can see, both MSE and exploitability decrease as the sample compensation factor increases, which is expected since a larger $T_{\mathrm{QMI}}$ leads to a more accurate approximation of FPI.
Furthermore, the improvement plateaus for large sample compensation factors, suggesting that a small sample compensation factor is sufficient, which is also more sample-efficient.

The numerical results of the experiments on the sample compensation factor confirm the intuition: as $T_{\mathrm{QMI}}$ increases, performance improves since the QMI operators more accurately approximate the FPI operators.
However, a larger $T_{\mathrm{QMI}}$ necessitates more samples, and thus the sample efficiency decreases.
Therefore, we also investigated the impact of $T_{\mathrm{QMI}}$ given a fixed total number of samples $K\cdot T$. The results are reported in \cref{fig:rr-t}.
In this experiment, we use a constant step size of $\alpha_t = 0.001$ for Q-value function updates.
The results demonstrate that different numbers of inner iterations offers nearly identical performance, implying that QMI is robust against the inexactness of the BR and IP approximation, as long as a sufficient total number of samples is present.

\subsection{Routing Game on a Network} % \label{sec:exp-graph}

% fig:graph-mu
\begin{figure}[ht]%[H]
	\centering
	\begin{subfigure}[b]{0.32\textwidth}
		\centering
		\includegraphics[width=\textwidth]{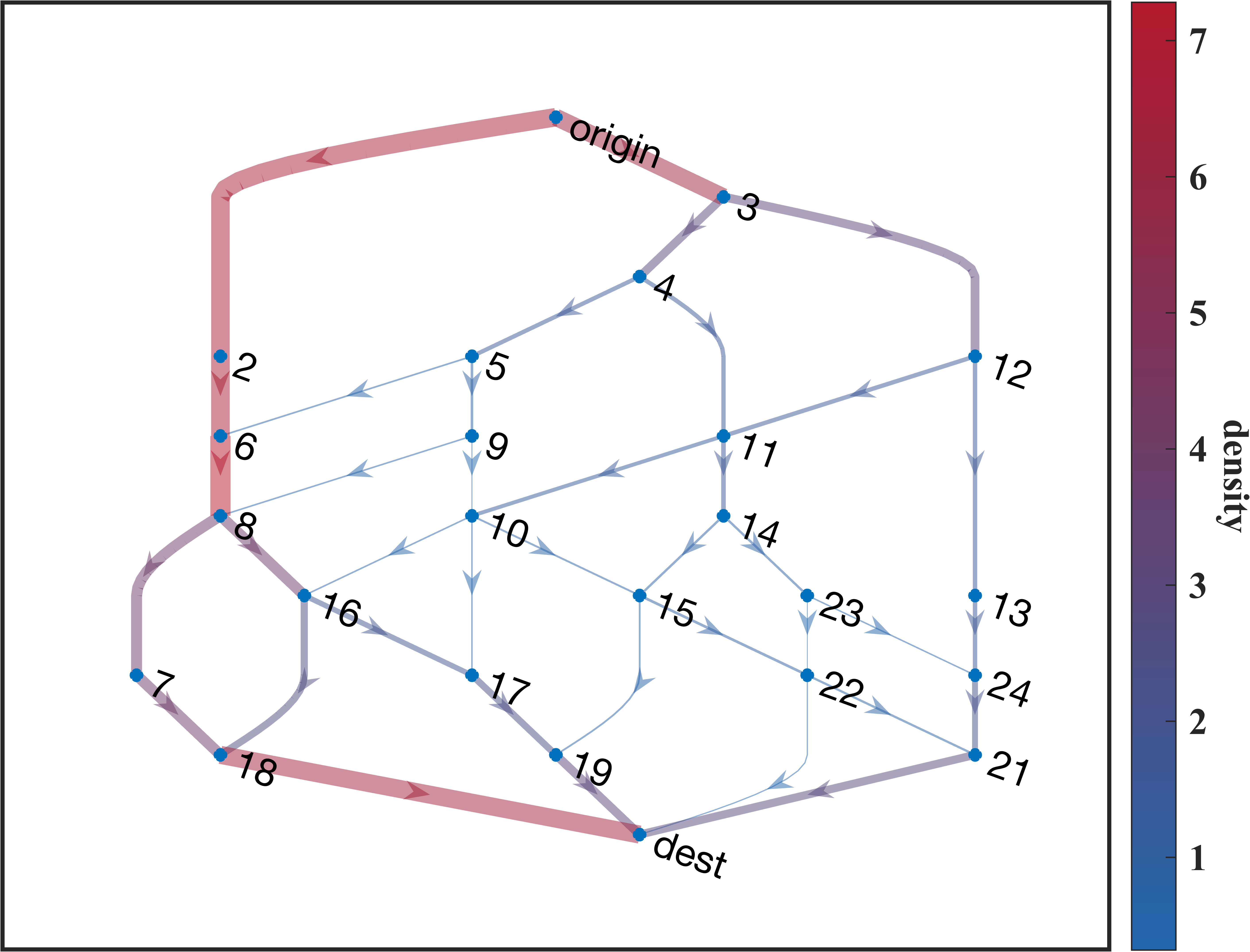}
		\caption{On-policy QMI}
	\end{subfigure}
	\hfill
	\begin{subfigure}[b]{0.32\textwidth}
		\centering
		\includegraphics[width=\textwidth]{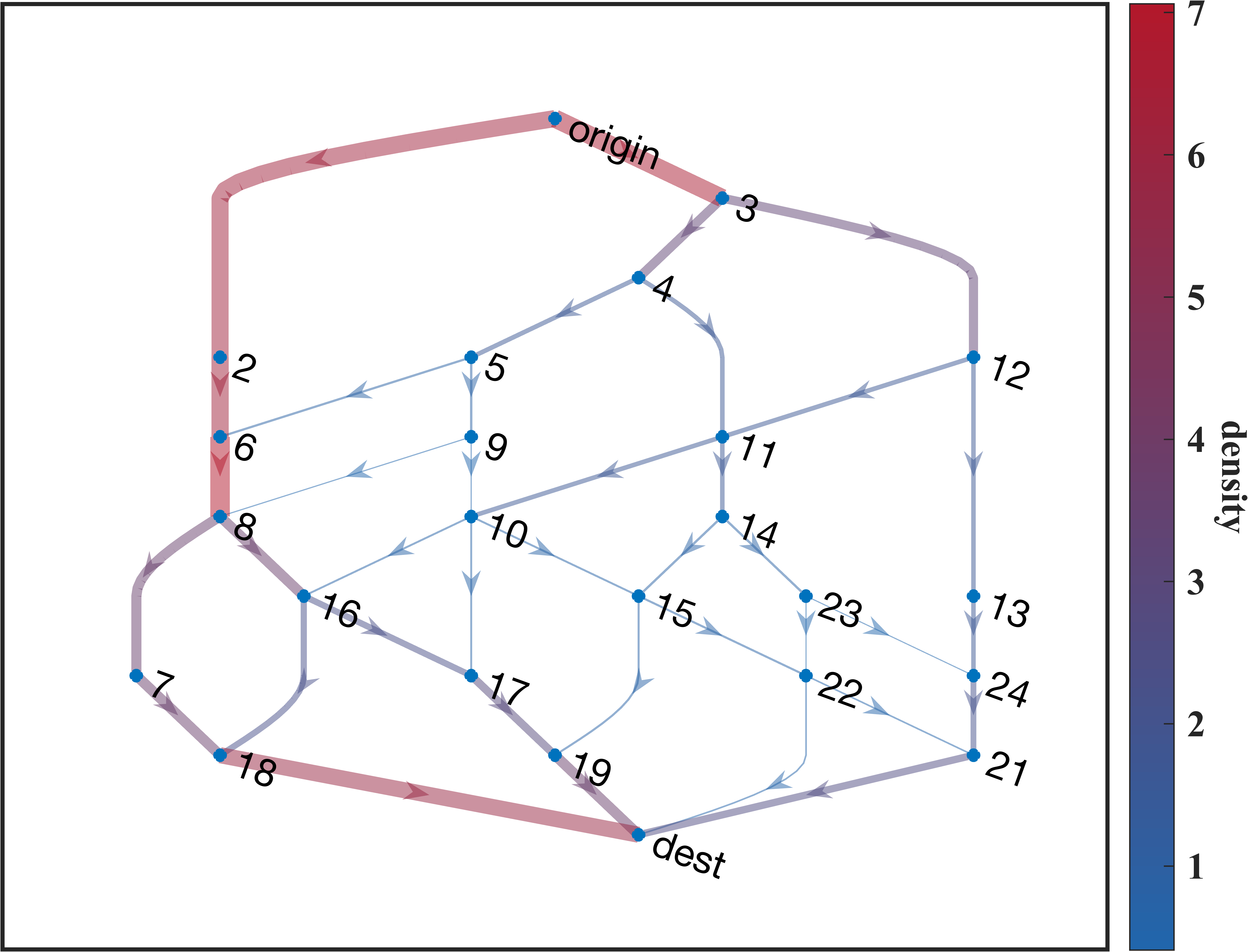}
		\caption{FPI}
	\end{subfigure}
	\hfill
	\begin{subfigure}[b]{0.32\textwidth}
		\centering
		\includegraphics[width=\textwidth]{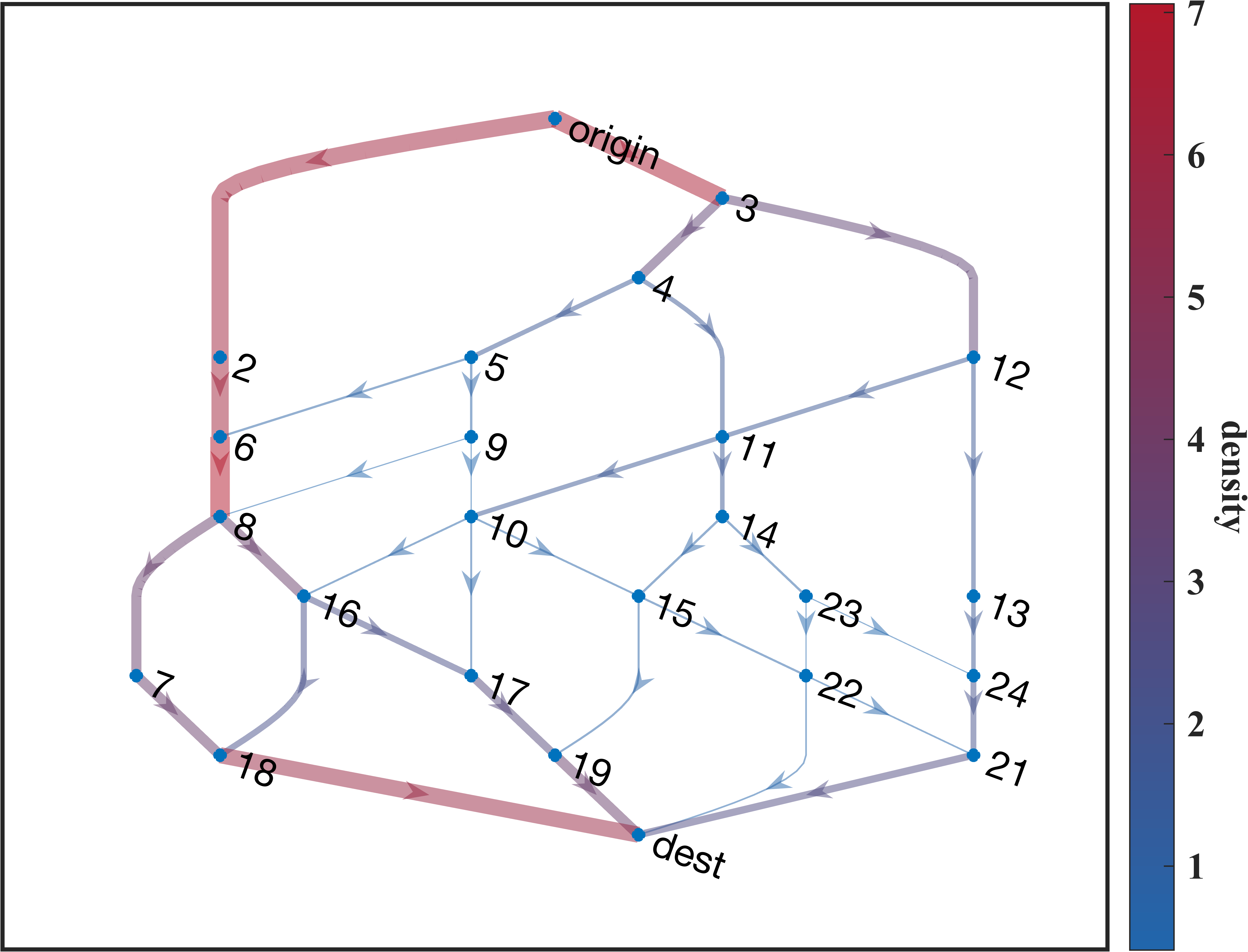}
		\caption{MFNE}
	\end{subfigure}
 \vspace{0.5cm}
	\caption{Learned population distributions and MFNE population distribution.}
	\label{fig:graph-mu}
	\vspace{0.5cm}
\end{figure}

% fig:graph-eta
\begin{figure}[ht]%[H]
	\centering
	\begin{subfigure}[b]{0.4\textwidth}
		\centering
		\includegraphics[width=\textwidth]{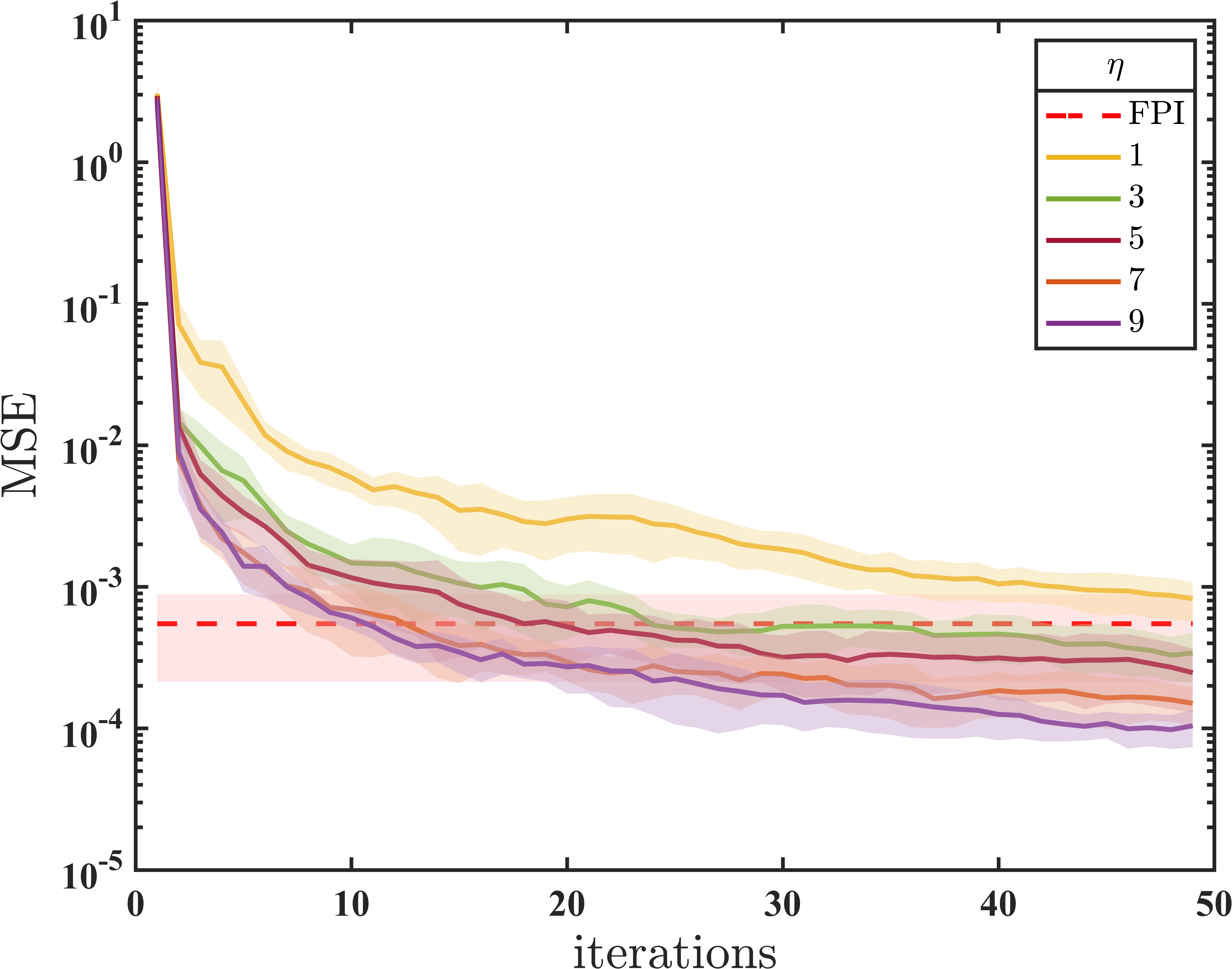}
		\caption{Mean squared error of off-policy QMI}
	\end{subfigure}
	\begin{subfigure}[b]{0.4\textwidth}
		\centering
		\includegraphics[width=\textwidth]{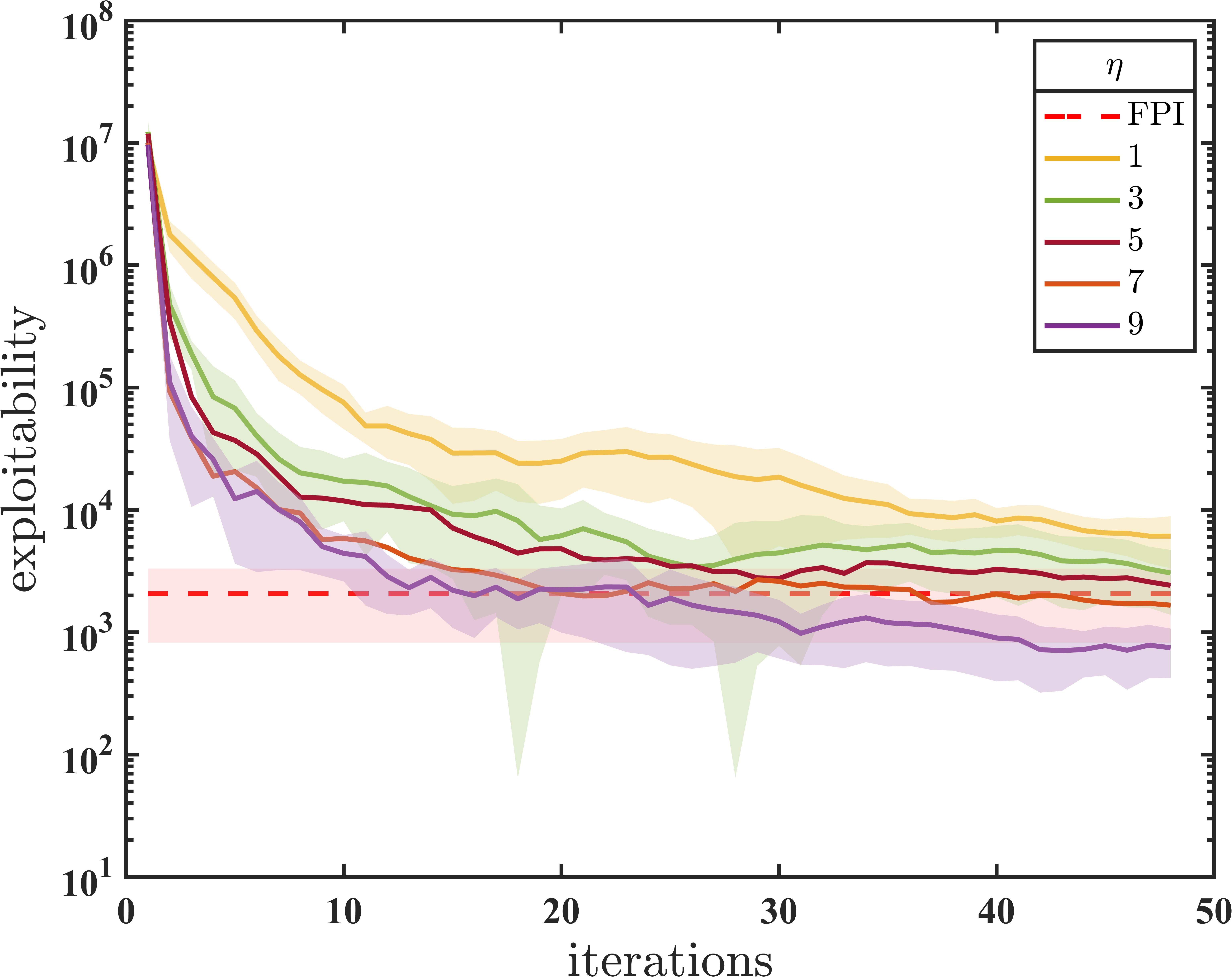}
		\caption{Exploitability of off-policy QMI}
	\end{subfigure}
	\par\bigskip
	\begin{subfigure}[b]{0.4\textwidth}
		\centering
		\includegraphics[width=\textwidth]{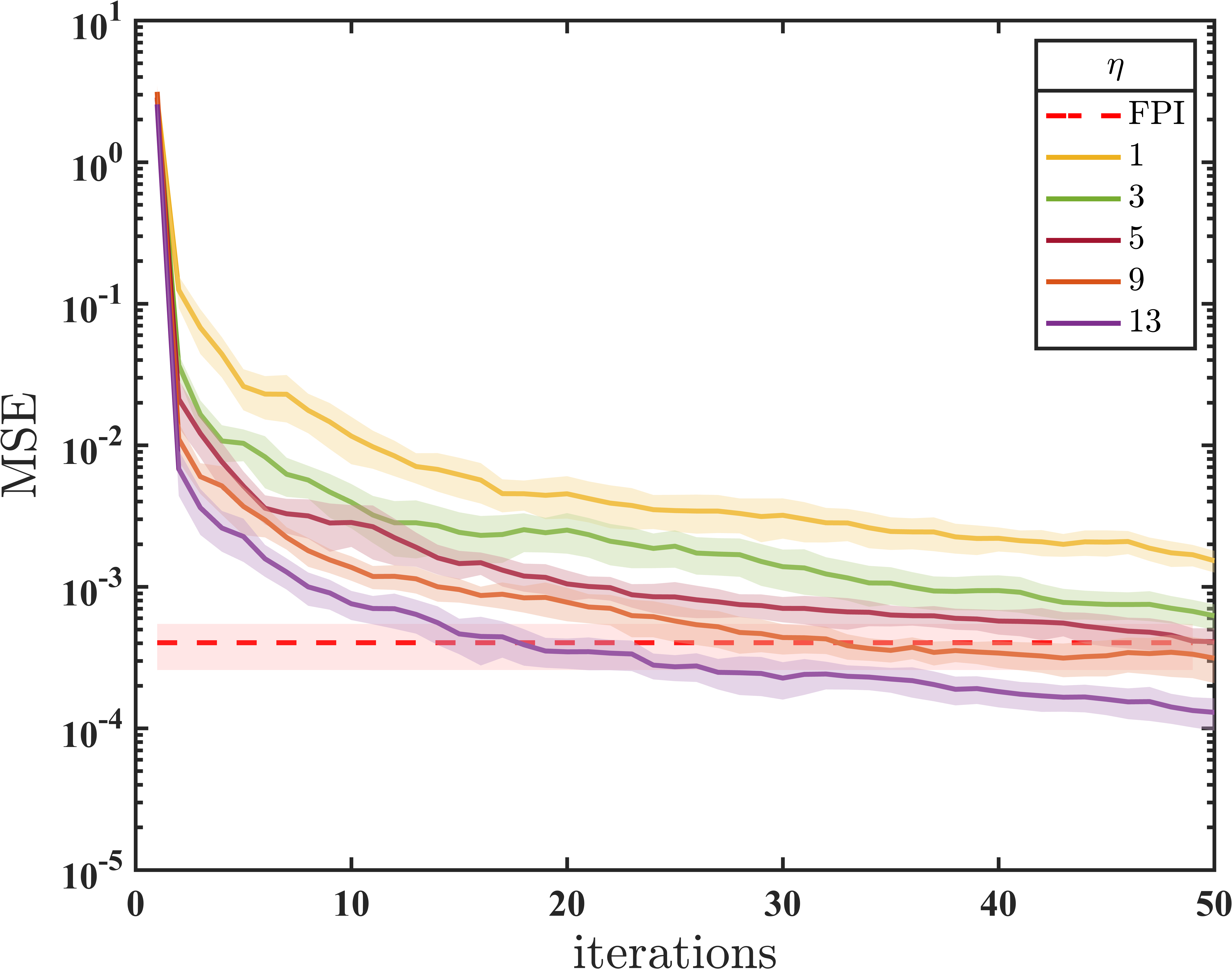}
		\caption{Mean squared error of on-policy QMI}
	\end{subfigure}
	\begin{subfigure}[b]{0.4\textwidth}
		\centering
		\includegraphics[width=\textwidth]{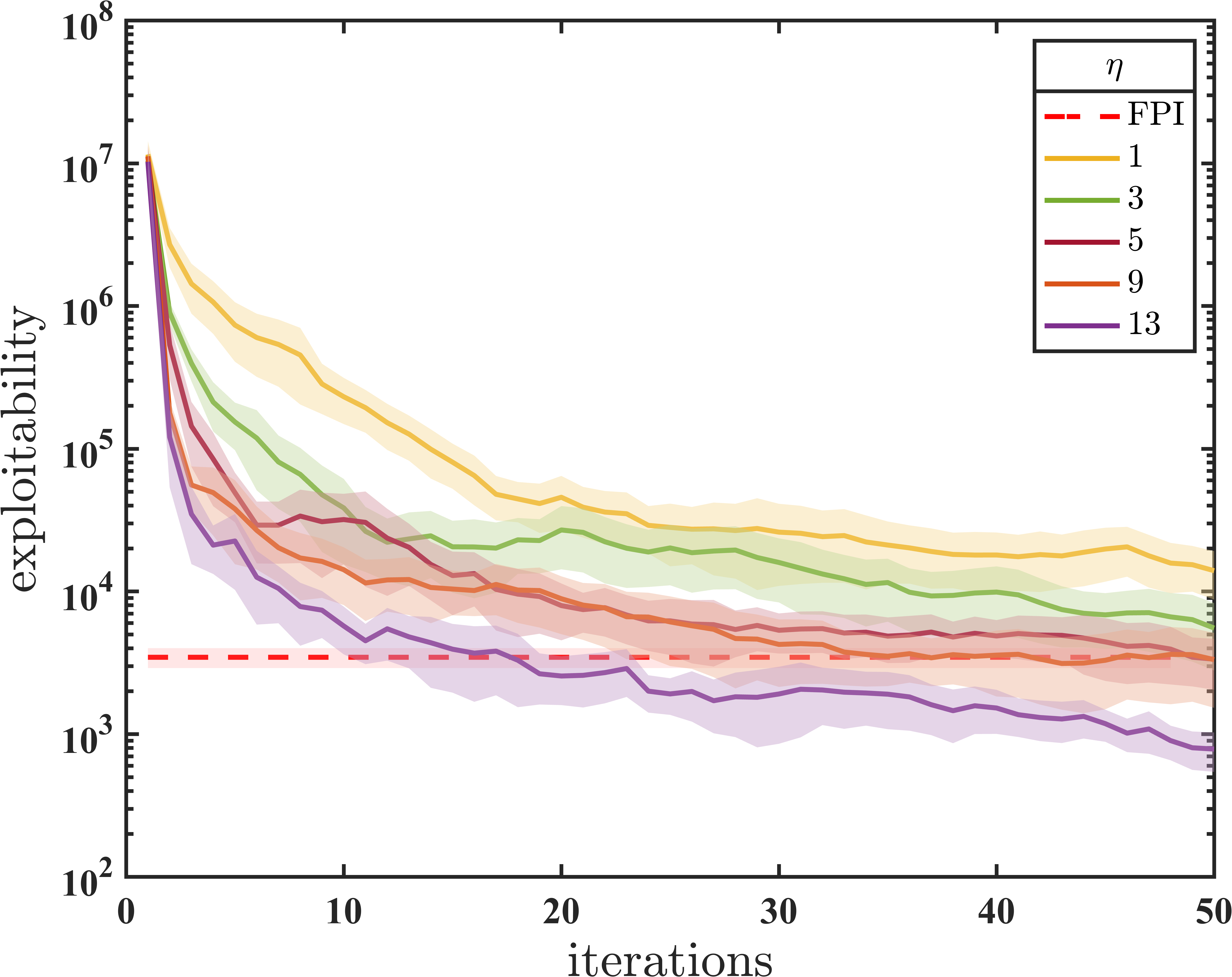}
		\caption{Exploitability of on-policy QMI}
	\end{subfigure}
 \vspace{0.5cm}
	\caption{Performance comparison of different sample compensation factors on Sioux Falls network routing. As a baseline, the performance of FPI after 50 iterations is plotted as dashed lines.}
	\label{fig:graph-eta}
 \vspace{0.5cm}
\end{figure}

We consider a routing game on the Sioux Falls network,\footnote{The topology of the network is available at \url{https://github.com/bstabler/TransportationNetworks}.} a graph with $24$ nodes and $74$ directed edges. We designate node $1$ as the starting point and node $20$ as the destination.
To construct an infinite-horizon game, we add a \emph{restart} edge $e_{75}$ from the destination back to the starting point.
On each edge, a vehicle selects its next edge to travel to. We consider a deterministic environment, meaning that the vehicle will follow the chosen edge without any randomness.
Therefore, both the state space and the action space can be represented by the edge set, i.e., $\mathcal{S} = \mathcal{A} = \{e_1,\ldots,e_{75}\} \cong [75]$, where $e_{75}$ is the restart edge.
It is worth noting that a vehicle can only select from the outgoing edges of its current location as its next edge.
% We denote $\mathcal{A}_{s}$ as the outgoing edges of $s\in \mathcal{S}$, i.e., the action subset at state $s$.

The objective of a vehicle is to reach the destination as fast as possible. Due to congestion, a vehicle spends a longer time on an edge with higher population distribution. Specifically, the cost (time) on a non-restart edge is \(r^{(\text{cong.})}(s,a,\mu) = - c_1\mu(s)^2 \mathbbm{1}  \{s \neq e_{75}\} \),
where $c_1$ is a cost constant. To drive the vehicle to the destination, we impose a reward at the restart edge: $r^{(\text{term.})}(s,a,\mu) = c_2 \mathbbm{1}\{s = e_{75}\}$. Together, we get the cost function:
\[
	r(s,a,\mu) = \underbrace{- c_1\mu(s)^2 \mathbbm{1} \{s \neq e_{75}\}}_{\text{congestion cost}}  + \underbrace{c_2 \mathbbm{1}\{s = e_{75}\}
	}_{\text{terminal reward}}
	.\]
We set $c_1 = 10^5$ and $c_2=10$.
The other algorithmic parameters are chosen as follows: the discount factor $\gamma = 0.8$, the initial state is uniformly sampled, the initial value function is set as all-zero, the initial population is randomly generated.
We use the softmax function with a fixed temperature of $10^{-4}$ as the policy operator.
The effective number of outer iterations is $K = 50$ and the number of sweeps in FPI is $T_{\mathrm{FPI}} = 5$. All the results are averaged over 10 independent runs.

With a sample compensation factor of $\eta_{\mathrm{off}} = 4$ for off-policy QMI and $\eta_{\mathrm{on}} = 7$ for on-policy QMI, both variants achieve a similar efficacy as QMI, again, validating our methods. The results are reported in \cref{fig:graph}.
The population distributions learned by on-policy QMI and FPI, as well as the MFNE population distribution, are shown in \cref{fig:graph-mu}.

Similar to \cref{sec:exp-rr}, we also experiment with different sample compensation factors and different number of inner iterations given a fixed total number of samples. The results are reported in \cref{fig:graph-eta,fig:graph-t}.
The results are consistent with those in \cref{sec:exp-rr}, reinforcing the validation, efficiency, and robustness of our methods.

% fig:graph-t
\begin{figure}[ht]%[H]
	\centering
	\begin{subfigure}[b]{0.4\textwidth}
		\centering
		\includegraphics[width=\textwidth]{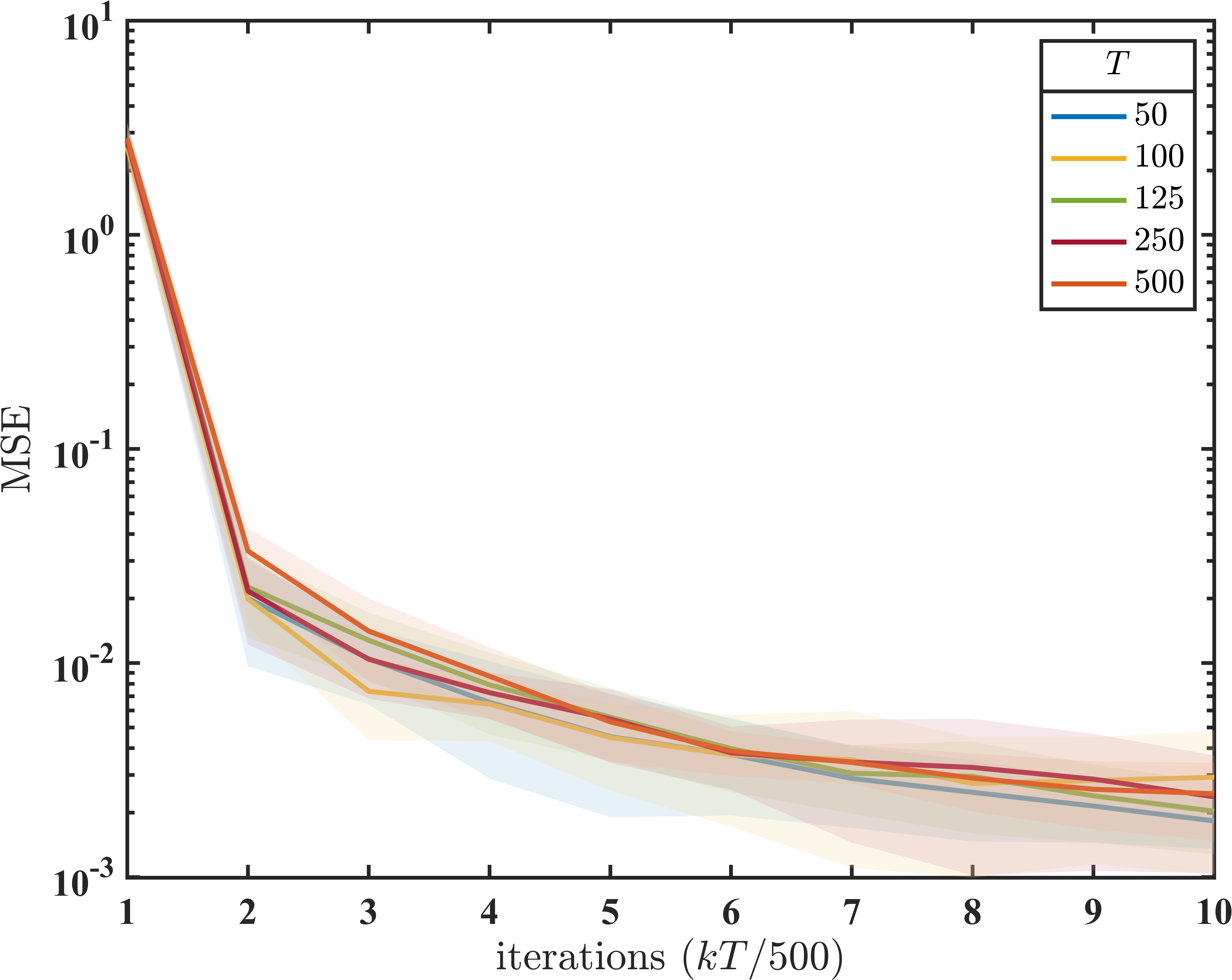}
		\caption{Mean squared error of off-policy QMI}
	\end{subfigure}
	\begin{subfigure}[b]{0.4\textwidth}
		\centering
		\includegraphics[width=\textwidth]{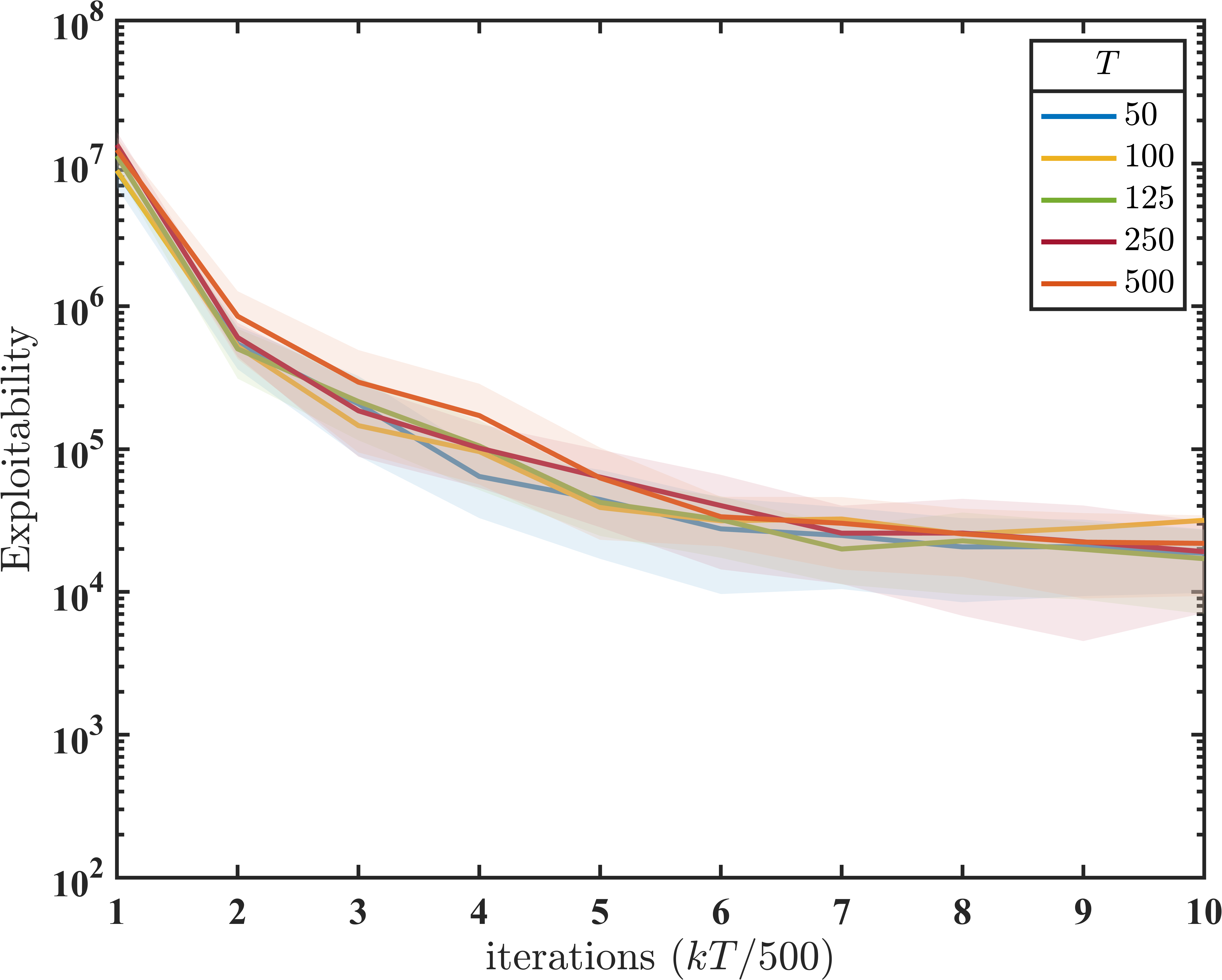}
		\caption{Exploitability of off-policy QMI}
	\end{subfigure}
	\par\bigskip
	\begin{subfigure}[b]{0.4\textwidth}
		\centering
		\includegraphics[width=\textwidth]{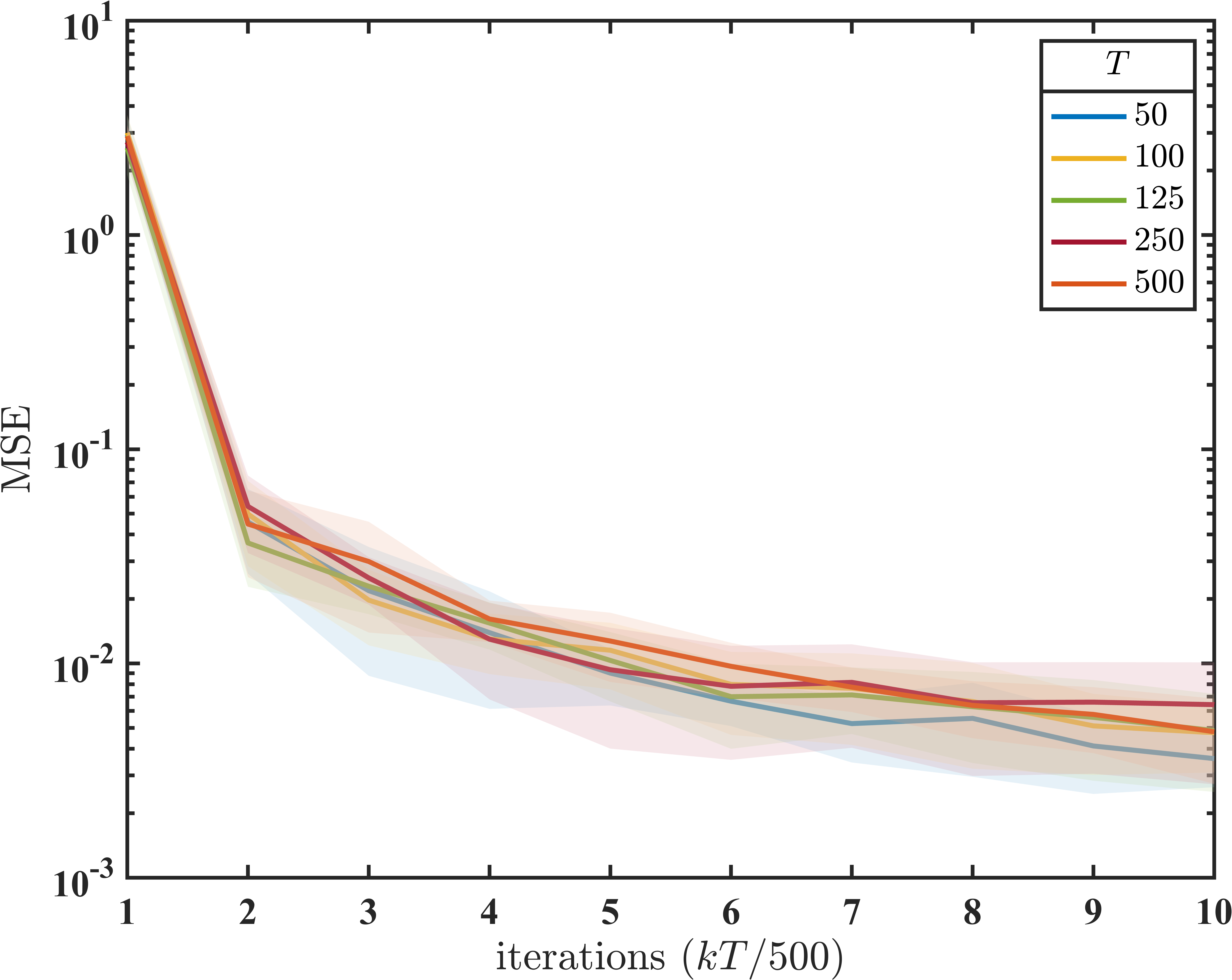}
		\caption{Mean squared error of on-policy QMI}
	\end{subfigure}
	\begin{subfigure}[b]{0.4\textwidth}
		\centering
		\includegraphics[width=\textwidth]{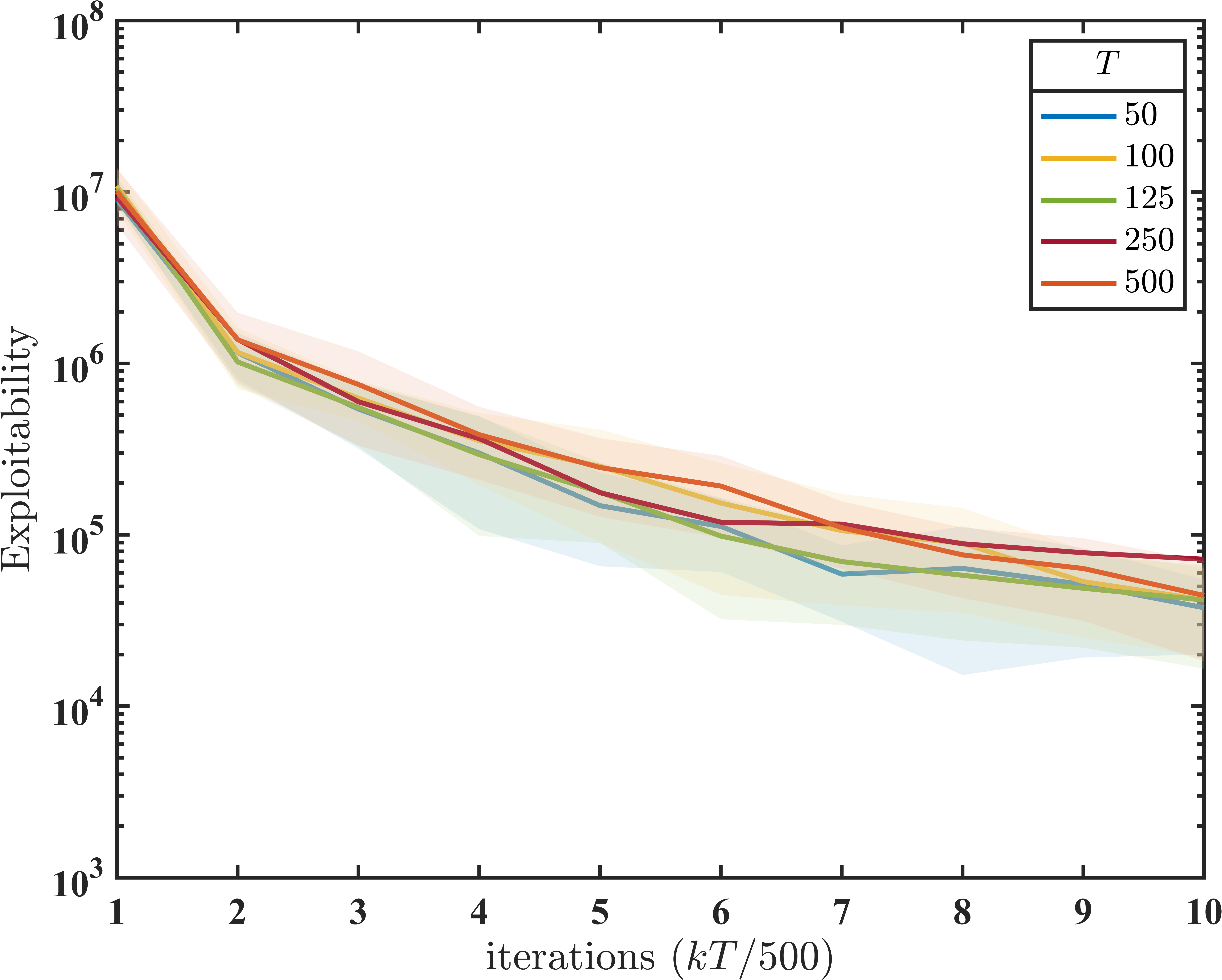}
		\caption{Exploitability of on-policy QMI}
	\end{subfigure}
 \vspace{0.5cm}
	\caption{Performance comparison of different number of inner iterations given a fixed total sample size: $KT=5000$ on Sioux Falls network routing.}
	\label{fig:graph-t}
 \vspace{0.5cm}
\end{figure}

\ifSubfilesClassLoaded{\bibliography{mfg}}{}

\section{Notation}

We summarize the notations used in this paper in \cref{tab:notation}.
Additionally, the $L_2$ norm is used when we omit the subscript.
For value function sequences generated by \cref{alg}, we write $M_{k} \coloneqq M_{k,0}$ and $Q_{k}\coloneqq Q_{k,0}$ for notational simplicity.
And when we restrict our discussion within an outer iteration, we omit the superscript $k$ and write $M_{t} \coloneqq M_{k,t}$ and $Q_{t} \coloneqq Q_{k,t}$.

\begin{table}[ht]
	\caption{Notation.} \label{tab:notation}
	\centering
	\begin{tabular}{cll}
		\textbf{Symbol}                                                                   & \textbf{Definition}                            & \textbf{Reference}        \\\midrule
		$\mathcal{S},\mathcal{A}$                                                         & state and action space                         & \cref{sec:mfg}            \\
		$S,A$                                                                             & cardinality of $\mathcal{S}$ and $\mathcal{A}$ & \cref{sec:mfg}            \\
		$r$                                                                               & reward function                                & \cref{sec:mfg}            \\
		$R$                                                                               & reward cap                                     & \cref{sec:mfg}            \\
		$P$                                                                               & transition kernel                              & \cref{sec:mfg}            \\
		$(s,a,r,s',a')$                                                                   & online observation tuple                       & \cref{sec:m-update}       \\
		$\gamma$                                                                          & discount factor                                & \cref{sec:mfg}            \\
		$\pi$                                                                             & policy                                         & \cref{sec:mfg}            \\
		$\mu$                                                                             & population distribution                        & \cref{sec:mfg}            \\
		$\Delta(\mathcal{S})$                                                             & probability simplex over $\mathcal{S}$         & \cref{sec:mfg}            \\
		$\pi_{M}, Q_{M}$                                                                  & the BR w.r.t. $M\in \Delta(\mathcal{S})$       & \cref{sec:mfg}            \\
		$\mu_{\pi}, \mu_{Q}$                                                              & the IP w.r.t. $\pi$ and $Q\in\R^{S\times A}$   & \cref{sec:mfg}            \\
		$Q,M$                                                                             & Q-value and M-value function                   & \cref{sec:mfg}            \\
		$\T{Q,M},\P{Q,M}$                                                                 & Bellman operator and transition operator       & \cref{def:mfne}           \\
		$\Gamma_{\pi}$                                                                    & policy operator                                & \cref{sec:mfg}            \\
		$\Gamma_{\mathrm{BR}}, \Gamma_{\mathrm{ID}}$                                      & FPI operators                                  & \cref{def:fpi-op}         \\
		$\Gamma_{Q}, \Gamma_{M}, \hat{\Gamma}_{\mathrm{off}}, \hat{\Gamma}_{\mathrm{on}}$ & QMI operators                                  & \cref{def:qmi-op}         \\
		$\kappa$                                                                          & contraction parameter                          & \cref{asmp:contract}      \\
		$\alpha,\beta$                                                                    & step sizes                                     & \cref{sec:m-update}       \\
		$\delta_{s'}$                                                                     & indicator probability vector                   & \cref{sec:m-update}       \\
		$m,\rho,\sigma,\lambda_{\min}$                                                    & MDP constants                                  & \cref{asmp:ergodic}       \\
		$K,T$                                                                             & number of outer and inner iterations           & \cref{alg}                \\
		$L$                                                                               & Lipschitz constant for training smoothness     & \cref{asmp:ql,asmp:sarsa} \\
		$\eta$                                                                            & sample compensation factor                     & \cref{sec:apx-exp}        \\
		$\tau$                                                                            & backtracking period                            & \cref{sec:apx-on}         \\
	\end{tabular}
\end{table}

\section{Derivation of Online Stochastic Update for Population Distribution} \label{sec:apx-update}

% Due to the constraint that \(M\) is a probability distribution, we need to update all entries of \(M\) in one update. Therefore, 
We consider the squared $L_{2}$ error w.r.t. the MFNE $(q^*,\mu^*)$:
\[
	\frac{1}{2}\|M - \mu^{*}\|_{2}^{2} = \frac{1}{2}\sum_{s\in\mathcal{S}}(M(s)-\mu^{*}(s))^{2},
\]
whose gradient w.r.t. $M$ is
\[
	g_{M}^{(\mathrm{real})} = M - \mu^* \overset{\cref{eq:qm}}{=} M - \mathcal{P}_{q^*} \mu^* \overset{\cref{eq:pm-exp}}{=} M - \EE_{q^{*},\mu^*}[\delta _{s'}]
	.\]
Since the information about MFNE is unavailable, we substitute it with the current policy-population pair $(Q,M)$ (bootstrapping), giving the following \emph{semi-gradient}:
\[
	g^{(\mathrm{semi})}_{M} \coloneqq M - \E{Q,M}[\delta _{s'}]
	.\]
Therefore, a single online agent can update the population distribution estimate using the stochastic semi-gradient:
\begin{equation}
	g_{M}(s') \coloneqq M - \delta _{s'}
	,\end{equation}
which gives the M-value function update rule in \cref{eq:mupdate}.
The Q-value function stochastic update rule can be derived similarly (see \citet{sutton2018Reinforcementlearninga}).

\section{Preliminary Lemmas}

This section provides two preliminary lemmas to assist the subsequent analysis.

\begin{lemma}[Steady distribution difference {\citep[Corollary 3.1]{mitrophanov2005SensitivityConvergence}}] \label{lem:diff}
	For any two Q-value functions $Q_1,Q_2\in\R^{S\times A}$, the difference between their induced steady distributions is bounded by
	\[
		\|\mu_{Q_1} - \mu_{Q_2}\|_{\mathrm{TV}} \le \sigma\|P_{Q_1} - P_{Q_2}\|_{\mathrm{TV}},
	\]
	where we denote $P_{Q}\coloneqq P_{\Gamma_{\pi}(Q)}$, and
	\[
		\|P_{Q}\|_{\mathrm{TV}} \coloneqq \sup_{\|q\|_{\mathrm{TV}} = 1} \left\| \sum_{s\in\mathcal{S}} q(s)P_{Q}(s,\cdot ) \right\|_{\mathrm{TV}}
		.\]
	And the constant $\sigma$ is defined as \(\sigma = \hat{n} + m\rho^{\hat{n}} /(1-\rho)\),
	where $\hat{n} = \left\lceil \log_{\rho}m^{-1} \right\rceil$.

	Furthermore, by \citet[Lemma 3]{zou2019Finitesampleanalysis}, if the policy operator is Lipschitz continuous as specified in \cref{asmp:sarsa}, we have
	\[
		\|\mu_{Q_1} - \mu_{Q_2}\|_{\mathrm{TV}} \le L\sigma\|Q_1 - Q_2\|_{2}
		.\]
\end{lemma}

\begin{lemma}[Young's inequality] \label{lem:ineq}
	For two points $x,y$ in an inner product space, we have
	\[
		2\left< x,y \right> \le \theta \|x\|^2 + 1/\theta \|y\|^2, \quad \forall \theta\in \mathbb{R}
		,\]
	where the norm is induced by the inner product.
	The above inequality further gives
	\[
		\|x + y\|^2 \le (1 + \theta)\|x\|^2 + (1 + 1 /\theta)\|y\|^2, \quad \forall \theta\in\R
		.\]
\end{lemma}

\ifSubfilesClassLoaded{\bibliography{mfg}}{}

\section{Analysis for Off-Policy QMI} %\label{sec:apx-off}

\subsection{Remark of Lemma~\ref{lem:ql}}

For \cref{lem:ql}, we utilize the result from \citet[Theorem 7]{qu2020Finitetimeanalysis}, a finite time high probability $L_{\infty}$ error bound for tabular Q-learning, which reads: suppose \cref{asmp:ergodic} holds and the step size is $\alpha_t = h /(t+t_0)$ where $h \ge 4 / (\lambda_{\min}(1-\gamma))$ and $t_0 \ge \max \{ 4h, \left\lceil \log_{2} 2/\lambda_{\min} \right\rceil \log (4m) / \log \rho^{-1} \}$;\footnotemark then with probability at least $1-\delta$,
\[
	\|\Gamma_{Q}(T)M - \Gamma_{\mathrm{BR}}M\|_{\infty}^{2} = O\left( \frac{R^2 \log (T /\delta)}{\lambda_{\min}^2(1-\gamma)^{5}T} \right)
	,\]
where we omit the logarithmic dependencies on $S,A$ and $\lambda_{\min}$.
A high probability bound is stronger than a mean squared error bound. To see this, we have
\[
	\EE \left\| \Gamma_{Q}(T)M - \Gamma_{\mathrm{BR}}M  \right\|^2_{\infty} \le (1-\delta)\cdot  O\left( \frac{R^2 (\log T + \log\delta^{-1})}{\lambda_{\min}^2(1-\gamma)^{5}T} \right) + \delta \cdot (2R)^2
	.\]
Substituting $\delta$ with $O(\log T / (\lambda_{\min}(1-\gamma)^{5}T))$ gives the mean squared error bound we desire.
By the relationship between the $L_{\infty}$ norm and $L_{2}$ norm, we get the result presented in \cref{lem:ql}:
\[\label{eq:ql}
	\EE \left\| \Gamma_{Q}(T)M - \Gamma_{\mathrm{BR}}M  \right\|^2_{2} = O\left( \frac{SA R^2 \log T }{\lambda_{\min}^2(1-\gamma)^{5}T} \right)
	.\]

\footnotetext{\citet{qu2020Finitetimeanalysis} uses a general mix time constant $t_{\mathrm{mix}}$, which is bounded by $\log (4m) / \log \rho^{-1}$ by \cref{asmp:ergodic}.}

We acknowledge that there are other finite time error bounds for Q-learning. For example,
\citet{bhandari2018FiniteTime} established a finite time mean squared $L_2$ error bound for Q-learning with linear function approximation for optimal stopping problems,
and \citet{li2023QLearningMinimax} sharpened the bound in \cref{eq:ql} by a factor of $\frac{1}{\lambda_{\min}(1-\gamma)}$ with a constant step size.
Since improving the constant dependencies in the finite error bounds is not our focus, we utilize the result from \citet{qu2020Finitetimeanalysis} as their setting is the closest to ours.

\subsection{Remark of Lemma~\ref{lem:mcmc}}

To directly invoke the results from \citet{latuszynski2013Nonasymptoticbounds} for \cref{lem:mcmc}, we need to recast \cref{asmp:ergodic} using the terminology of small sets: the state space is $(1-\rho)$-small (please refer to \citet[Chapter 5]{meyn2012Markovchains} for the definitions).
By \citet[Theorem 16.2.4]{meyn2012Markovchains}, a $(1-\rho)$-small state space implies \cref{asmp:ergodic}.
Then, for a step size of $\beta_t = 1 /(t+1)$, \citet[Theorem 3.1]{latuszynski2013Nonasymptoticbounds} gives a finite time mean squared $L_{\infty}$ error bound for stationary Markov chain Monte Carlo methods. \citet[Remark 4.3]{latuszynski2013Nonasymptoticbounds} claims that
\[
	\EE |(\Gamma_{M}(T)Q)_{s} - (\Gamma_{\mathrm{IP}}Q)_{s}|^2 = O\left( \frac{1}{(1-\rho)^2T} \right)
	,\]
where $(x)_{s}$ represents the $s$th element of vector $s$.
The result in \cref{lem:mcmc} follows by relating the $L_{\infty}$ norm and $L_{2}$ norm.

\subsection{Proof of Theorem~\ref{thm} for Off-Policy QMI}

\begin{proof}
	In this proof, we omit the subscript for the $L_2$ norm.
	We first bound the difference between $\hat{\Gamma}_{\mathrm{off}}$ and $\Gamma$.
	For any $M\in\Delta(\mathcal{S})$, we have the decomposition
	\begin{align}
		    & \left\| \hat{\Gamma}_{\mathrm{off}}(T)M - \Gamma M \right\|                                                             \\
		=   & \left\| \left( \Gamma_{M}(T) \circ \Gamma_{Q}(T) \right) M - (\Gamma_{\mathrm{IP}}\circ \Gamma_{\mathrm{BR}})M \right\| \\
		\le & \left\| \left( \Gamma_{M}(T) \circ \Gamma_{Q}(T) \right) M
		- \left( \Gamma_{\mathrm{IP}} \circ \Gamma_{Q}(T) \right) M\right\|
		+ \left\|\left( \Gamma_{\mathrm{IP}} \circ \Gamma_{Q}(T) \right) M
		- (\Gamma_{\mathrm{IP}}\circ \Gamma_{\mathrm{BR}})M \right\|                                                                  \\
		=   & \left\| \Gamma_{M}(T)Q - \Gamma_{\mathrm{IP}}Q \right\| + \left\| \m{Q} - \mu_{q_{M}} \right\| \label{eq:off-decomp}
		,\end{align}
	where we denote $Q \coloneqq \Gamma_{Q}(T)M$, and recall that $\mu_{Q}$ is the IP w.r.t. policy $\Gamma_{\pi}(Q)$.
	By \cref{lem:diff}, we have
	\[
		\|\mu_{Q} - \mu_{q_{M}}\|_{2} \le \|\mu_{Q} - \mu_{q_{M}}\|_{1} \le \sigma \|P_{Q} - P_{q_{M}}\|_{\mathrm{TV}}
		.\]
	Then, by \cref{asmp:ql}, we have $\|P_{Q_1} - P_{q_{M}}\|_{\mathrm{TV}} \le L \|Q - q_{M}\|$.
	% \begin{align}
	% 	\|P_{Q} - P_{q_{M}}\|_{\mathrm{TV}}
	% 	=   & \sup_{\substack{q\in\Delta(\mathcal{S})                           \\\|q\|_{\mathrm{TV}} = 1}}\left\| \sum_{s\in \mathcal{S}}(P_{Q}(s,\cdot ) - P_{q_{M}}(s,\cdot ))q(s) \right\|_{\mathrm{TV}}\\
	% 	\le & \sup_{\substack{q\in\Delta(\mathcal{S})                           \\\|q\|_{\mathrm{TV}} = 1}}\sum_{s\in \mathcal{S}}\sum_{s'\in \mathcal{S}}q(s)| P_{Q}(s,s') - P_{q_{M}}(s,s')|\\
	% 	\le & L \|Q - q_{M}\|_{2} \cdot \sup_{\substack{q\in\Delta(\mathcal{S}) \\\|q\|_{\mathrm{TV}} = 1}}\sum_{s\in \mathcal{S}}q(s)\\
	% 	=   & L \|Q-q_{M}\|_{2}
	% 	.\end{align}
	Plugging the above bound back into \cref{eq:off-decomp} gives
	\[
		\left\| \hat{\Gamma}_{\mathrm{off}}(T)M - \Gamma M \right\|
		\le \|\Gamma_{M}(T)Q - \Gamma_{\mathrm{IP}}Q\| + L\sigma\|\Gamma_{Q}(T)M - \Gamma_{\mathrm{BR}}M\|
		.\]
	Then, by \cref{lem:ineq,lem:ql,lem:mcmc}, we get
	\begin{align}
		\EE \left\| \hat{\Gamma}_{\mathrm{off}}(T)M-\Gamma M \right\|^2 \le &
		2\EE \left\| \Gamma_{M}(T)Q-\Gamma_{\mathrm{IP}} Q \right\|^2
		+ 2 (L\sigma)^2\EE \left\| \Gamma_{Q}(T)M-\Gamma_{\mathrm{BR}} M \right\|^2                                                                                                                     \\
		=                                                                   & O\left( \frac{SA}{(1-\rho)^2 T} \right) + (L\sigma)^2O \left( \frac{SAR^2\log T}{\lambda^2_{\min}(1-\gamma)^{5}T} \right) \\
		=                                                                   & O\left( \frac{SA R^2L^2\sigma^2 \log T}{\lambda^2_{\min}(1-\gamma)^{5}T} \right), \label{eq:off-2}
	\end{align}
	where the asymptotic notation holds when $T$ is large enough such that $\frac{R^2L^2\sigma^2\log T}{\lambda^2_{\min}(1-\gamma)^{5}} \gg (1-\rho)^{-2}$.

	We now bound the mean squared error of the M-value function after $K$ outer iterations. Without loss of generality, we assume $K$ is even. In this proof, we write $M_{k} = M_{k,0}$ for notational simplicity.
	By \cref{lem:ineq}, we have
	\begin{align}
		\EE \left\| \hat{\Gamma}_{\mathrm{off}}^{K}M_0 - \mu^{*} \right\|^2
		=   & \EE \left\| \hat{\Gamma}_{\mathrm{off}}M_{K-2} - \Gamma\mu^{*} \right\|^2                                                                                                   \\
		=   & \EE \left\| \hat{\Gamma}_{\mathrm{off}}M_{K-2} - \Gamma M_{K-2} + \Gamma M_{K-2} - \Gamma\mu^{*} \right\|^2                                                                 \\
		\le & (1+\kappa)\EE\left\| \Gamma M_{K-2} - \Gamma \mu^{*} \right\|^2 + (1 + 1 /\kappa)\EE\left\| \hat{\Gamma}_{\mathrm{off}}M_{K-2} - \Gamma M_{K-2} \right\|^2 \label{eq:off-3}
		.\end{align}
	By \cref{asmp:contract,eq:off-2}, we get
	\begin{align}
		\EE \left\| \hat{\Gamma}_{\mathrm{off}}^{K}M_0 - \mu^{*} \right\|^2
		\le & (1+\kappa)(1-\kappa)^2\EE\left\| M_{K-2} - \mu^{*} \right\|^2 + (1 + 1 /\kappa)O\left( \frac{L^2SA R^2\sigma^2 \log T}{\lambda^2_{\min}(1-\gamma)^{5}T} \right) \\
		\le & (1-\kappa)\EE \left\| M_{K-2}-\mu^{*} \right\|^2 + O\left( \frac{SA R^2L^2\sigma^2 \log T}{\kappa\lambda^2_{\min}(1-\gamma)^{5}T} \right)
		.\end{align}
	Recursively applying the above inequality gives
	\begin{align}
		\EE \left\| \hat{\Gamma}_{\mathrm{off}}^{K}M_0 - \mu^{*} \right\|^2
		\le & (1-\kappa)^{K/2}\EE\left\| M_{0}-\mu^{*} \right\|^2 + O\left( \frac{SA R^2 L^2\sigma^2\log T}{\kappa\lambda^2_{\min}(1-\gamma)^{5}T} \right) \sum_{k=1}^{K /2} (1-\kappa)^{k} \\
		=   & O\left(\exp\left(-\frac{\kappa K}{2}\right) + \frac{SA R^2L^2\sigma^2 \log T}{\kappa^2\lambda^2_{\min}(1-\gamma)^{5}T} \right)
		.\end{align}
\end{proof}

\ifSubfilesClassLoaded{\bibliography{mfg}}{}

\section{Analysis for On-Policy QMI with Population-Dependent Transition Kernel} \label{sec:apx-on}
\subsection{Population-Dependent Transition Kernel} \label{sec:gmfg}

In this subsection, we generalize our setting to incorporate the influence of the population on agents' state transitions.
Specifically, the transition kernel \(P(s'\given s,a,\mu)\) represents the probability of an agent transitioning to state \(s'\) when it takes action \(a\) at state \(s\) with the ``current'' population distribution being \(\mu\). When both the policy $\pi$ and population distribution $\mu$ are fixed, we denote the transition kernel by $P_{\pi,\mu}$, which makes an agent's trajectory a stationary Markov chain, and $P_{\pi,\mu}(s,s')$ represents the probability of an agent transitioning to state \(s'\) from $s$ when the dynamics are determined by $\pi$ and $\mu$.

This generalization necessitate several modifications in our setup: in \cref{eq:q}, the expectation $\EE_{\pi,\mu}$ is taken w.r.t. the transition kernel $\PP{\pi,\mu}$; and the optimality equations \cref{eq:bellman} become
\begin{equation}
	\begin{cases}
		Q & = \T{Q,M} Q, \\
		M & = \P{Q,M} M,
	\end{cases}
\end{equation}
where the Bellman operator now integrates w.r.t. the transition kernel \(\PP{Q,M}\):
\[
	\T{Q,M}Q(s,a) = \E{Q,M}[r(s,a,M) + \gamma Q(s',a')],
\]
and the transition operator becomes
\[
	\P{Q,M}M(s') = \sum_{s\in \mathcal{S}}\PP{Q,M}(s,s')M(s)
	,\]
where we denote $\PP{Q,M} \coloneqq \PP{\Gamma_{\pi}(Q),M}$.

Since now the transition operator also depends on the population, we need to redefine the IP operator in \cref{def:fpi-op} as
\[
	\Gamma_{\mathrm{IP}}: \R^{S\times A} \times \Delta(\mathcal{S}) \to \Delta(\mathcal{S}),\ (Q,M)\mapsto \m{Q,M}
	,\]
which returns the unique fixed point of the transition operator $\P{Q,M}$ for any value function $Q$ and population distribution $M$.
To composite the BR operator and IP operator, we need to extend the BR operator:
\[
	\Gamma_{\mathrm{BR}}: \Delta(\mathcal{S})\to \R^{S\times A} \times \Delta(\mathcal{S}),\  M \mapsto (\q{M}, M)
	,\]
i.e., keep a copy of the original population while calculating the BR.
The FPI operator is still the composition of the two operators: $\Gamma = \Gamma_{\mathrm{IP}}\circ\Gamma_{\mathrm{BR}}$.

Although the definition of the transition operator $\P{Q,M}$ has changed, the same derivation in \cref{sec:m-update} yields the same online stochastic update rule for the population distribution specified in \cref{eq:mupdate}.
Consequently, the only modification required in \cref{alg} is that, within the $k$th outer iteration, we fix the reference population distribution as $M_{k,0}$ in the transition kernel $\PP{\pi_{k,t},M_{k,0}}$, similar to that we fix it in the reward function $r(s_t,a_t,M_{k,0})$.

\subsection{Proof of Lemma~\ref{lem:sarsa}}

In this subsection, we prove \cref{lem:sarsa} considering general population-dependent transition kernels, which contain population-independent transition kernels as special cases.
Note that the on-policy QMI operator is defined within a single outer iteration, so we omit all the subscripts related to the outer iteration. We make the following notational simplification:
\[
	M_t \coloneqq \hat{\Gamma}_{\mathrm{on}}(t)M, \quad
	\mu^{\dagger} \coloneqq \Gamma_{\mathrm{IP}}(q^{\dagger},M) \coloneqq \Gamma_{\mathrm{IP}}(\Gamma_{\mathrm{BR}}M) = \Gamma M, \quad
	\Delta_t :=  M_t - \mu^{\dagger}
	.\]
Also, although $\hat{\Gamma}_{\mathrm{on}}$ does not directly return the Q-value function, it will update the Q-value function, and thus the behavior policy, along the process.
We denote $\hat{Q}_t$ the \emph{mixed} Q-value function at inner time step $t$ (see \cref{alg}), and $\mu_t = \Gamma_{\mathrm{IP}}(\hat{Q}_t,M)$ the population distribution induced by the transition kernel $\PP{\hat{Q}_t,M}$.

Since the on-policy QMI trajectories are not stationary anymore, we need a special virtual \emph{backtracking} process to assist our analysis, which enables us to use the mixing property of stationary Markov chains \citep{zou2019Finitesampleanalysis}.
Specifically, we consider a virtual trajectory, where we backtrack a period $\tau$ and fix the behavior policy as $\Gamma_{\pi}(\hat{Q}\tt)$ after time step $t-\tau$. Thus, this virtual trajectory is stationary after time step $t-\tau$, and its observation distribution will rapidly converge to the steady distribution w.r.t. $\hat{Q}\tt$ due to the mixing property of stationary MDPs (\cref{asmp:ergodic}).
We denote $\tilde{\delta}_{t}$ the indicator function/vector such that $\tilde{\delta}_{t}(s) = \mathbbm{1}(s = \tilde{s}_{t})$, where $\tilde{s}_{t}$ is the virtual state observation on this virtual trajectory at time step $t$.
For the actual state $s_t$ at time step $t$ induced by the actual nonstationary trajectory, we denote $\delta_{t} \coloneqq \delta_{s_t}$.
In this section, we omit the subscript for the $L_2$ norm.

To prove \cref{lem:sarsa}, we first decompose the difference $\Delta_t$ recursively.

\begin{lemma}[Error decomposition] \label{lem:decomp}
	\begin{align}
		\EE\left\| \Delta_{t+1} \right\|^2
		\le & (1-2\beta_t)\EE\left\| \Delta_t \right\|^2 + 4\beta_{t}^2                                                  \\
		    & + 2\beta _{t}\EE\langle\Delta_t, \mu_{t} - \mu^{\dagger}\rangle               &  & \textup{(control)}      \\
		    & + 2\beta _{t}\EE\langle\Delta_t, \mu_{t-\tau} - \mu_t\rangle                  &  & \textup{(progress)}     \\
		    & + 2\beta _{t}\EE\langle\Delta_t, \tilde{\delta}_{t+1} - \mu_{t-\tau}\rangle   &  & \textup{(mixing)}       \\
		    & + 2\beta _{t}\EE\langle\Delta_t,  \delta_{t+1} - \tilde{\delta}_{t+1}\rangle. &  & \textup{(backtracking)}
	\end{align}
\end{lemma}
\begin{proof}
	By the update rule \cref{eq:mupdate}, we have
	\[
		\EE\left\| \Delta_{t+1} \right\|^2 = \EE\left\| M_t + \beta_t (\delta_{t+1}-M_t) - \mu^{\dagger} \right\| ^2
		= \EE\|\Delta_t\|^2 + 2\beta_t\EE\left<\Delta_t,\delta_{t+1} - M_t \right> + \beta_{t}^2\EE\|\delta_{t+1}-M_t\|^2
		.\]
	Since $\delta_{t+1}$ and $M_t$ are both probability vectors, we have $\|\delta_{t+1}-M_t\| \le \|\delta_{t+1}\| + \|M_t\| \le \|\delta_{t+1}\|_1 + \|M_t\|_1 = 2$.
	Then, we apply the following decomposition:
	\[
		\delta_{t+1} - M_t = -(M_t - \mu^{\dagger}) + (\mu_t - \mu^{\dagger}) + (\mu\tt - \mu_t) + (\tilde{\delta}_{t+1} - \mu\tt) + (\delta_{t+1} - \tilde{\delta}_{t+1})
		,\]
	which gives the result.
\end{proof}

We next provide four lemmas, each bounding one term in the above decomposition.

\begin{lemma}[Control] \label{lem:control}
	With a step size of $\alpha _{t} = \frac{b}{\lambda_{\min}(1-\gamma)(t+1+a)}$ for the Q-value function update, where $a$ and $b$ are constants ensuring that the initial step size $\alpha_0$ is small enough (see \citet{zhang2023federated}), we have
	\[
		\EE\left\| \mu_{t} - \mu^{\dagger} \right\|^2 = O \left( \frac{SAR^2L^2\sigma^2 \log t}{\lambda_{\min}^2(1-\gamma)^4t} \right)
		.\]
\end{lemma}
\begin{proof}
	First, by \cref{lem:diff}, we can bound the distribution difference by the control (Q-value function) difference:
	\[
		\EE\left\| \mu_t - \mu^{\dagger} \right\|^2 \le (L\sigma)^2 \EE \left\| \hat{Q}_t - q^{\dagger} \right\|^2
		.\]
	Then, by \citet[Corollary 2.2]{zhang2023federated}, we have
	\[
		\EE\left\| \hat{Q}_t - q^{\dagger} \right\|^2 = O\left( \frac{SAR^2\log t}{\lambda_{\min}^2(1-\gamma)^4t } \right)
		.\]
	Plugging this bound back gives the result.
\end{proof}

\begin{lemma}[Progress] \label{lem:progress}
	% If the step size for Q-value update is non-increasing, then we have
	Given the policy update rule in \cref{alg}, we have
	\[
		% \EE\left\| \mu_t - \mu\tt \right\|^2 \le \frac{4\alpha\tt^2 \tau^2 L^2R^2\sigma^2}{(1-\gamma)^2}
		\EE\left\| \mu_t - \mu\tt \right\|^2 = O\left( \frac{R^2L^2\sigma^2 \tau^2}{(1-\gamma)^2(t-\tau)^2} \right)
		.\]
\end{lemma}
\begin{proof}
	First, by \cref{lem:diff}, we can bound the distribution progress by the Q-value function progress:
	\[
		\EE\left\| \mu_t - \mu\tt \right\|^2 \le (L\sigma)^2 \EE \left\| \hat{Q}_t - \hat{Q}\tt \right\|^2
		.\]
	We first bound the one-step mixed Q-value function progress. According to the convex combination in \citet[Corollary 2.2]{zhang2023federated}, $w_t = t + a$ and $W_t = \sum_{l=0}^{t}w_{l} \asymp t^2$. We have
	\begin{align}
		\left\| \hat{Q}_{t+1} - \hat{Q}_t \right\| = &
		\left\| \frac{w_{t+1}}{W_{t+1}}Q_{t+1} +\frac{W_t}{W_{t+1}}\hat{Q}_{t} - \hat{Q}_t\right\|                                                                          \\
		\le                                          & \frac{w_{t+1}}{W_{t+1}}\left\| Q_{t+1} \right\| + \left| \frac{W_t}{W_{t+1}} - 1\right|\left\| \hat{Q}_{t}  \right\| \\
		\le                                          & \frac{R}{1-\gamma} \cdot O\left( 1 /t \right) + \frac{R}{1-\gamma} \cdot O\left( 1 /t \right)                        \\
		=                                            & O\left( \frac{R}{(1-\gamma)t} \right)
		,\end{align}
	where we use the fact that the Q-value function is uniformly bounded by $R /(1-\gamma)$ and $W_t \asymp t^2$.
	%    Then, by the Q-value function update rule \cref{eq:qupdate}, we have
	% By the Q-value function update rule \cref{eq:qupdate}, we have
	% \[
	% 	\|Q_t - Q\tt\| = \left\|\sum_{l=t-\tau}^{t-1}\alpha_{l}g_{Q_l}\right\| \le \tau\alpha\tt \left(R + (1+\gamma)\frac{R}{1-\gamma}\right) = \frac{2\alpha\tt \tau R}{1-\gamma}
	% 	,\]
	% where we require the step size to be non-increasing and use the fact that the Q-value function is uniformly bounded by $R /(1-\gamma)$ \placeholder{ see @zhang}.
	Then, by the triangle inequality, we have
	\[
		\left\|\hat{Q}_t - \hat{Q}\tt\right\| = O \left( \frac{R\tau }{(1-\gamma)(t-\tau)} \right)
		.\]

	Plugging this bound back gives the result.
\end{proof}

\begin{lemma}[Mixing]\label{lem:mix}
	Let $\tau$ be the backtracking period. Then, for any $\theta>0$, we have
	\[
		\left|2 \EE\left<\Delta_t, \tilde{\delta}_{t+1} - \mu\tt\right> \right|  \le \theta\EE\|\Delta_t\|^2 + m^2 \rho^{2\tau}/\theta
		.\]
\end{lemma}
\begin{proof}
	We first bound the difference between $\tilde{\delta}_{t+1}$ and $\mu\tt$ conditioned on the filtration $\mathcal{F}\tt$ containing all the randomness before time step $t-\tau$.
	Recall that $\EE [\tilde{\delta}_{t+1}] = \operatorname{Pr}(\tilde{s}_{t+1} = \cdot) \in\Delta(\mathcal{S})$, which gives
	\[\label{eq:delta2p}
		\left\| \EE\left[ \tilde{\delta}_{t+1} - \mu\tt \Given \mathcal{F}\tt \right] \right\|
		= \left\| \EE\left[ \tilde{\delta}_{t+1}  \Given \mathcal{F}\tt \right] - \mu\tt\right\|
		= \left\| \operatorname{Pr}(\tilde{s}_{t+1}=\cdot \given \mathcal{F}\tt ) - \mu\tt\right\|
		.\]
	Note that $\tilde{s}_{t+1}$ is on the virtual stationary trajectory following a fixed policy $\Gamma_{\pi}(\hat{Q}\tt)$. Thus, \cref{asmp:ergodic} gives
	\[
		\left\| \EE\left[ \tilde{\delta}_{t+1} - \mu\tt \Given \mathcal{F}\tt \right] \right\|
		\le \|\operatorname{Pr}(\tilde{s}_{t+1} = \cdot \given \mathcal{F}\tt) - \mu\tt\|_{\mathrm{TV}}
		\le m \rho^{\tau}
		.\]
	Also note that conditioned on $\mathcal{F}\tt$, the virtual trajectory and the actual one are independent. Therefore, we have
	\begin{align}
		\left| \EE\left<\Delta_t, \tilde{\delta}_{t+1} - \mu\tt\right> \right|
		=   & \left| \EE\left[ \EE\left[ \left<\Delta_t, \tilde{\delta}_{t+1} - \mu\tt\right> \Given \mathcal{F}\tt \right] \right] \right|                                        \\
		=   & \left| \EE \left<\EE\left[ \Delta_t \Given \mathcal{F}\tt\right], \EE\left[ \tilde{\delta}_{t+1} - \mu\tt \Given \mathcal{F}\tt\right]\right>  \right|               \\
		\le & \EE \left[\left\| \EE\left[ \Delta_t \Given \mathcal{F}\tt\right] \right\|\left\|\EE\left[ \tilde{\delta}_{t+1} - \mu\tt \Given \mathcal{F}\tt\right]\right\|\right] \\
		\le & \EE \|\Delta_t\| \cdot m\rho^{\tau}
		.\end{align}
	Finally, invoking \cref{lem:ineq} gives the result.
\end{proof}

\begin{lemma}[Backtracking] \label{lem:back}
	Let $\tau$ be the backtracking period. Suppose the step size is non-increasing. Then, for any $\theta>0$, we have
	\[
		\left| 2\EE\left<\Delta_t,\delta_{t+1}-\tilde{\delta}_{t+1} \right> \right|
		\le \theta\EE\|\Delta_t\|^2 +
		4 \beta\tt\tau +
		O\left( \frac{LR\tau^3}{(1-\gamma)(t-\tau)}\left(\beta\tt + \frac{LR\tau}{\theta(1-\gamma)(t-\tau)} \right)\right)
		.\]
\end{lemma}
\begin{proof}
	Similar to \cref{eq:delta2p}, we have
	\[
		\left\| \EE\left[ \delta_{t+1} - \tilde{\delta}_{t+1} \Given \mathcal{F}\tt \right] \right\|  = \left\| \operatorname{Pr}\left( s_{t+1}=\cdot \Given \mathcal{F}\tt \right) - \operatorname{Pr}\left(\tilde{s}_{t+1}=\cdot \Given \mathcal{F}\tt\right) \right\|
		.\]
	By \citet[Equation 46]{zou2019Finitesampleanalysis}, we have
	\[\label{eq:back}
		\left\| \operatorname{Pr}\left( s_{t+1}=\cdot \Given \mathcal{F}\tt \right) - \operatorname{Pr}\left(\tilde{s}_{t+1}=\cdot \Given \mathcal{F}\tt\right) \right\|
		% \le \frac{4\alpha\tt\tau L R}{1-\gamma}
		\le L \sum_{l=t-\tau}^{t}\mathbb{E}\|\hat{Q}_{l} - \hat{Q}_{t-\tau}\|
		= O\left(  \frac{L R\tau^2}{(1-\gamma)(t-\tau)}\right)
		.\]
	Here, since $\delta_{t+1}$ and $M_t$ are correlated conditioned on $\mathcal{F}\tt$, we cannot directly get the result analogous to that in \cref{lem:mix}.
	We apply the following decomposition:
	\[
		\left| 2\EE\left<\Delta_t,\delta_{t+1}-\tilde{\delta}_{t+1} \right> \right|
		\le \underbrace{\left| 2\EE\left<M_t - M\tt,\delta_{t+1}-\tilde{\delta}_{t+1} \right> \right| }_{H_1}
		+ \underbrace{\left| 2\EE\left<M\tt - \mu^{\dagger},\delta_{t+1}-\tilde{\delta}_{t+1} \right> \right| }_{H_2}
		.\]
	For $H_1$, by \cref{lem:ineq}, we have
	\begin{align}
		H_1 \le & \theta' \EE\|M_t-M\tt\|^2 + 1 /\theta' \EE\|\delta_{t+1}-\tilde{\delta}_{t+1}\|^2       \\
		\le     & \theta'\EE\left\| \sum_{l=t-\tau}^{t-1} \beta_l g_{M_l} \right\|^2 + 1 /\theta' \cdot 4 \\
		\le     & 4\theta' \beta\tt^2 \tau^2 + 4 /\theta' \label{eq:back-2}
		,\end{align}
	where we use the fact that $\|M_1-M_2\|_{2}\le 2$ for any $M_1,M_2\in\Delta(\mathcal{S})$.
	Let $\theta' = 1 /(\beta\tt \tau)$. We get
	\[
		H_1 \le 4 \beta\tt \tau
		.\]
	For $H_2$, we can apply \cref{eq:back}, which gives
	\begin{align}
		H_2 = & 2\left| \EE\left[ \EE \left[ \left<\Delta\tt, \delta_{t+1}-\tilde{\delta}_{t+1} \right>\Given \mathcal{F}\tt \right] \right] \right|                                      \\
		=     & 2\left| \EE\left<\EE \left[ \Delta\tt \Given \mathcal{F}\tt\right],\EE\left[ \delta_{t+1}-\tilde{\delta}_{t+1} \Given \mathcal{F}\tt \right]\right> \right|               \\
		\le   & 2\EE\left[\left\|\EE \left[ \Delta\tt \Given \mathcal{F}\tt\right]\right\|\left\|\EE\left[ \delta_{t+1}-\tilde{\delta}_{t+1} \Given \mathcal{F}\tt \right]\right\|\right] \\
		% \le   & \EE\|\Delta\tt\| \cdot \frac{4\alpha\tt \tau LR}{1-\gamma}
		=     & \EE\|\Delta\tt\| \cdot O\left( \frac{LR\tau^2}{(1-\gamma)(t-\tau)} \right)
		.\end{align}
	Similar to \cref{eq:back-2}, we have
	\[
		\EE\|\Delta\tt\| \le \EE\|\Delta_t\| + \EE\|M_t - M\tt\| \le \EE\|\Delta_t\| + 2\beta\tt \tau
		.\]
	Plugging the above bounds on $H_1$ and $H_2$ back gives
	\[
		\left| 2\EE\left<\Delta_t,\delta_{t+1}-\tilde{\delta}_{t+1} \right> \right|
		\le 4\tau \beta\tt + \left( \EE\|\Delta_t\| + 2\beta\tt \tau \right) \cdot O\left( \frac{LR\tau^2}{(1-\gamma)(t-\tau)}
		\right)
		.\]
	Applying \cref{lem:ineq} on $\EE\|\Delta_t\| \cdot O\big( \frac{LR\tau^2}{(1-\gamma)(t-\tau)}\big)$ gives the result.
\end{proof}

Given the above five lemmas, we are ready to prove \cref{lem:sarsa}.

\begin{proof}[Proof of \cref{lem:sarsa}]
	We first plug \cref{lem:control,lem:progress,lem:mix,lem:back} into \cref{lem:decomp}:
	\begin{align}
		\EE\left\| \Delta_{t+1} \right\|^2
		\le & (1-2\beta_t)\EE\left\| \Delta_t \right\|^2 + 4\beta_{t}^2                                                                                                                                                         \\
		    & + \beta _{t}\theta\EE\|\Delta_t\|^2 + \beta_t /\theta \EE\|\mu_{t} - \mu^{\dagger}\|^2                                                                                                                            \\
		    & + \beta _{t}\theta\EE\|\Delta_t\|^2 + \beta_t /\theta \EE\|\mu_{t-\tau} - \mu_t\|^2                                                                                                                               \\
		    & + 2\beta _{t}\left|\EE\langle\Delta_t, \tilde{\delta}_{t+1} - \mu_{t-\tau}\rangle \right|                                                                                                                         \\
		    & + 2\beta _{t}\left|\EE\langle\Delta_t,  \delta_{t+1} - \tilde{\delta}_{t+1}\rangle\right|                                                                                                                         \\
		\le & (1-2\beta_t)\EE\left\| \Delta_t \right\|^2 + 4\beta_{t}^2                                                                                                                                                         \\
		    & + \beta _{t}\theta\EE\|\Delta_t\|^2 + \frac{\beta_t}{\theta}\cdot O \left( \frac{SAR^2L^2\sigma^2 \log t}{\lambda_{\min}^2(1-\gamma)^4t} \right) &                                & \textup{(\cref{lem:control})} \\
		    & + \beta _{t}\theta\EE\|\Delta_t\|^2 + \frac{\beta_t}{\theta} \cdot O\left( \frac{R^2L^2\sigma^2 \tau^2}{(1-\gamma)^2(t-\tau)^2} \right)
		    &                                                                                                                                                  & \textup{(\cref{lem:progress})}                                 \\
		    & + \beta_t\theta\EE\|\Delta_t\|^2 + \frac{\beta_t}{\theta} \cdot  m^2 \rho^{2\tau}                                                                &                                & \textup{(\cref{lem:mix})}     \\
		    & + \beta_t\theta\EE\|\Delta_t\|^2 +
		4 \beta\tt\beta_t\tau +
		\beta_{t} \cdot O\left( \frac{LR\tau^3}{(1-\gamma)(t-\tau)}\left(\beta\tt + \frac{LR\tau}{\theta(1-\gamma)(t-\tau)} \right)\right)
		.   &                                                                                                                                                  & \textup{(\cref{lem:back})}
	\end{align}
	After some arrangement, we get
	\begin{align}
		\EE\|\Delta_{t+1}\|^2
		\le & (1-\beta_t(2-4\theta)) \EE\|\Delta_t\|^2 + 4\beta^2_t + \frac{\beta_t}{\theta} \\
		    & \cdot O\left(  \frac{SAR^2L^2\sigma^2 \log t}{\lambda_{\min}^2(1-\gamma)^4t}
		+ \frac{R^2L^2\tau^4}{(1-\gamma)^2(t-\tau)^2}
		+ \frac{\theta LR\beta\tt\tau^3}{(1-\gamma)(t-\tau)}
		+ \frac{R^2L^2\sigma^2\tau^2}{(1-\gamma)^2(t-\tau)^2}
		+ \beta\tt\tau
		+ m^2 \rho^{2\tau} \right)
		\label{eq:sarsa-1}.
	\end{align}
	% where we use the fact that $\sigma \ge 1$.
	Note that we have not yet set $\theta$ and the backtracking period $\tau$.
	Let $\theta = 1 /4$ and $\tau = \lceil\log \frac{\beta_t}{m} / \log \rho\rceil$.
	Recall that $\beta_t \asymp 1 /t$ and $\alpha_t \asymp 1 /(\lambda_{\min}(1-\gamma)t)$; then $\tau \asymp \log t$, $\beta_t \asymp \beta\tt$, and $\alpha_t \asymp \alpha\tt$.
	Therefore, \cref{eq:sarsa-1} gives
	\begin{align}
		    & \EE\|\Delta_{t+1}\|^2                      \\
		% \le & (1-\beta_t) \EE\|\Delta_t\|^2 + 4\beta^2_t
		% + \beta_t \cdot O\left( \frac{L^2SAR^2\sigma^2 \log t}{\lambda_{\min}^2(1-\gamma)^4t}  + \frac{\alpha_t^2 L^2R^2\sigma^2 \log^2 t}{(1-\gamma)^2} + \beta^2_t + \beta_t \log t + \beta_t^2\log^2t \right)                                \\
		\le & (1-\beta_t) \EE\|\Delta_t\|^2 + 4\beta^2_t
		+ \beta_t \cdot O\left( \frac{SAR^2L^2\sigma^2 \log t}{\lambda_{\min}^2(1-\gamma)^4t}
		+ \frac{R^2L^2 \log^4 t}{(1-\gamma)^2t^2}
		+ \frac{RL \log^3 t}{(1-\gamma)t^2}
		+ \frac{R^2L^2\sigma^2 \log^2 t}{(1-\gamma)^2t^2}
		+ \frac{\log t}{t}
		+ \frac{1}{t^2}\right)                           \\
		% =   & (1-\beta_t) \EE\|\Delta_t\|^2
		% + \beta_t \cdot O\left(\frac{1}{t} + \frac{L^2SAR^2\sigma^2 \log t}{\lambda_{\min}^2(1-\gamma)^4t}  + \frac{L^2R^2\sigma^2 \log^2 t}{\lambda_{\min}^2(1-\gamma)^4t^2} + \frac{1}{t^2} + \frac{\log t}{t} + \frac{\log^2t}{t^2}  \right) \\
		=   & (1-\beta_t) \EE\|\Delta_t\|^2
		+ \beta_t \cdot  O\left( \frac{SAR^2L^2\sigma^2 \log t}{\lambda_{\min}^2(1-\gamma)^4t} \right), \label{eq:sarsa-2}
	\end{align}
	where the asymptotic notation in the last equality holds when $t \gg 1$ and $1 = O(SAR^2L^2\sigma^2 /(\lambda_{\min}(1-\gamma)^{4}))$.

	Recall our step size configuration: $\beta _t = 1/(t+1)$. Dividing both sides of \cref{eq:sarsa-2} by $\beta_t$ gives
	\[
		(t+1)\EE\|\Delta_{t+1}\|^2 \le t\EE\|\Delta_{t}\|^2 + C \cdot O\left( \frac{\log t}{t} \right)
		,\]
	where $C \coloneqq SAR^2L^2\sigma^2 / (\lambda^2_{\min}(1-\gamma)^{4})$.%
	\footnote{Here, with a slight abuse of notation, $C$ is smaller than the one in \cref{thm} by a factor of $1 /(1-\gamma)$. The larger $C$ in \cref{thm} accounts for the incorporation of \cref{lem:ql}. Actually, the dependencies in \cref{lem:ql} can be improved to match the ones in \cref{lem:sarsa}; see \citet{li2023QLearningMinimax}.}
	Then, we get
	\begin{align}
		T\EE\|\Delta_{T}\|^2
		\le & (T-1)\EE\|\Delta_{T-1}\|^2 + C \cdot O\left( \frac{\log (T-1)}{T-1} \right) \\
		\le & C \sum_{t=1}^{T}O\left( \frac{\log t}{t} \right)                            \\
		=   & C \cdot O(\log T)
		,\end{align}
	where the last equality is by Merten's theorem. And the result follows.
\end{proof}

\subsection{Proof of Theorem~\ref{thm} for On-Policy QMI}

\begin{proof}
	Similar to \cref{eq:off-3}, by \cref{lem:ineq}, we get
	\[
		\EE\left\| \hat{\Gamma}_{\mathrm{on}}^{K}M_0 -\mu^{*} \right\|^2 \le (1+\kappa)\EE\left\| \Gamma M_{K-1} - \Gamma \mu^{*} \right\|^2 + (1+1 /\kappa)\EE\left\| \hat{\Gamma}_{\mathrm{on}}M_{K-1} - \Gamma M_{K-1} \right\| ^2
		.\]
	By \cref{asmp:contract,lem:sarsa}, we get
	\begin{align}
		\EE\left\| \hat{\Gamma}_{\mathrm{on}}^{K}M_0 -\mu^{*} \right\|^2
		\le & (1+\kappa)(1-\kappa)^2\EE\left\| M_{K-1} - \mu^{*} \right\|^2 + (1+1 /\kappa) C \cdot O\left( \frac{\log T}{T} \right) \\
		\le & (1-\kappa)\EE\left\| M_{K-1} - \mu^{*} \right\|^2 + C \cdot O\left( \frac{\log T}{\kappa T} \right)                    \\
		\le & (1-\kappa)^{K} \|M_0 - \mu^{*}\|^2 + \sum_{k=1}^{K}(1-\kappa)^{k} C \cdot O\left( \frac{\log T}{T} \right)             \\
		=   & O\left( \exp (-\kappa K) + \frac{SAR^2L^2\sigma^2 \log T}{\kappa^2\lambda_{\min}(1-\gamma)^{4}T} \right)
		.\end{align}
\end{proof}

\ifSubfilesClassLoaded{\bibliography{mfg}}{}

\end{document}
%%%%%%%%%%%%%%%%%%%%%%%%%%%%%%%%%%%%%%%%%%%%%%%%%%%%%%%%%%%%%%%%%%%%%%